\documentclass[11pt]{article}

\usepackage{setspace}
\usepackage{cprotect}
\usepackage{amsmath,amssymb,amsthm}
\usepackage[noend]{algorithmic}
\usepackage[ruled,vlined]{algorithm2e}
\usepackage{hyperref}
\usepackage{fullpage}

\usepackage{makeidx}
\usepackage{enumerate}
\usepackage{graphicx,float,psfrag,epsfig}
\usepackage{epstopdf}
\usepackage{color}
\usepackage{enumitem}
\usepackage{subfig}
\usepackage{caption}
\usepackage{bigints}
\usepackage{mathtools}
\usepackage[mathscr]{euscript}

\def\<{\langle}
\def\>{\rangle}
\usepackage{mathtools}
\def\bx{{\boldsymbol x}}
\def\by{{\boldsymbol y}}
\def\bmu{{\boldsymbol \mu}}
\def\bSigma{{\boldsymbol \Sigma}}
\def\bnu{{\boldsymbol \nu}}
\def\bOmega{{\boldsymbol \Omega}}
\def\btOmega{{\tilde{\boldsymbol \Omega}}}
\def\btSigma{{\tilde{\boldsymbol \Sigma}}}
\def\bn{{\boldsymbol n}}
\def\be{{\boldsymbol e}}
\def\bs{{\boldsymbol s}}
\def\br{{\boldsymbol r}}
\def\bdff{{\boldsymbol f}}
\def\bu{{\boldsymbol u}}
\def\bg{{\boldsymbol g}}
\def\bX{{\boldsymbol X}}
\def\bG{{\boldsymbol G}}
\def\bO{{\boldsymbol O}}
\def\bI{{\boldsymbol I}}
\def\bA{{\boldsymbol A}}
\def\bQ{{\boldsymbol Q}}
\def\bLambda{{\boldsymbol \Lambda}}
\def\blambda{{\boldsymbol \lambda}}
\def\bv{{\boldsymbol v}}
\def\bW{{\boldsymbol W}}
\def\sT{{\sf T}}
\def\sb{{\sf b}}
\def\sx{{\sf x}}
\def\sy{{\sf y}}
\def\sa{{\sf a}}
\def\su{{\sf u}}
\def\tsu{{\tilde{\sf u}}}
\def\sv{{\sf v}}
\def\tsv{{\tilde{\sf v}}}
\def\tsb{{\tilde{\sf b}}}
\def\tsa{{\tilde{\sf a}}}
\def\tbu{{\tilde{\boldsymbol{u}}}}
\def\tbv{{\tilde{\boldsymbol{v}}}}
\def\tbf{{\tilde{\boldsymbol{f}}}}
\def\tbg{{\tilde{\boldsymbol{g}}}}
\def\hbu{{\hat{\boldsymbol{u}}}}
\def\hbv{{\hat{\boldsymbol{v}}}}
\def\hbf{{\hat{\boldsymbol{f}}}}
\def\hbg{{\hat{\boldsymbol{g}}}}
\def\hf{\hat{f}}
\def\hg{\hat{g}}
\def\hu{\hat{u}}
\def\hv{\hat{v}}
\def\tmu{\tilde{\mu}}
\def\tsigma{\tilde{\sigma}}
\def\tnu{\tilde{\nu}}
\def\tomega{\tilde{\omega}}
\def\tu{\tilde{u}}
\def\tv{\tilde{v}}
\def\tf{\tilde{f}}
\def\tg{\tilde{g}}
\def\tU{\tilde{U}}
\def\tY{\tilde{Y}}
\def\tV{\tilde{V}}
\def\tZ{\tilde{Z}}
\def\tF{\tilde{F}}
\def\tG{\tilde{G}}

\def\E{{\mathbb E}} 
\def\PL{{\rm PL}}
\def\PCA{{\rm PCA}}
\newcommand{\normal}{\mathcal{N}}
\newcommand{\reals}{\mathbb{R}}
\newcommand{\bzero}{\boldsymbol{0}}
\newcommand{\beq}{\begin{equation}}
\newcommand{\eeq}{\end{equation}}
\newcommand\norm[1]{\left\lVert{#1}\right\rVert}
\newcommand\abs[1]{\left\lvert{#1}\right\rvert}

\numberwithin{equation}{section}

\newtheoremstyle{myexample} 
    {\topsep}                    
    {\topsep}                    
    {\rm }                   
    {}                           
    {\bf }                   
    {.}                          
    {.5em}                       
    {}  

\newtheoremstyle{myremark} 
    {\topsep}                    
    {\topsep}                    
    {\rm}                        
    {}                           
    {\bf}                        
    {.}                          
    {.5em}                       
    {}  

\newtheorem{claim}{Claim}[section]
\newtheorem{lemma}[claim]{Lemma}

\newtheorem{theorem}{Theorem}
\newtheorem{proposition}[claim]{Proposition}

\theoremstyle{myremark}

\theoremstyle{myremark}

\theoremstyle{myexample}

\title{PCA Initialization for Approximate Message Passing in Rotationally Invariant Models}

\author{Marco Mondelli\thanks{Institute of Science and Technology (IST) Austria. Email: \texttt{marco.mondelli@ist.ac.at}.}\;\;\;and\;\;\;Ramji Venkataramanan\thanks{Department of Engineering, University of Cambridge. Email: \texttt{ramji.v@eng.cam.ac.uk}.}}
\begin{document}

\maketitle

\begin{abstract}
  We study the problem of estimating a rank-$1$ signal in the presence of rotationally invariant noise---a class of perturbations more general than Gaussian noise.  Principal Component Analysis (PCA) provides a natural estimator, and sharp results on its performance have been obtained in the high-dimensional regime. Recently, an Approximate Message Passing (AMP) algorithm has been proposed as an alternative estimator with the potential to improve the accuracy of PCA. However, the existing analysis of AMP requires an initialization that is both correlated with the signal and independent of the noise, which is often unrealistic in practice. In this work, we combine the two methods, and propose to initialize AMP with PCA. Our main result is a rigorous asymptotic characterization of the performance of this estimator. Both the AMP algorithm and its analysis differ from those previously derived in the Gaussian setting: at every iteration, our AMP algorithm requires a specific term to account for PCA initialization, while in the Gaussian case, PCA initialization affects only the first iteration of AMP. The proof is based on a two-phase artificial AMP that first approximates the PCA estimator and then mimics the true AMP. Our numerical simulations show an excellent agreement between AMP results and theoretical predictions, and suggest an interesting open direction on achieving Bayes-optimal performance. 
\end{abstract}




\section{Introduction}

We consider the problem of estimating a rank-$1$ signal from a noisy data matrix. In the square symmetric case, the data matrix is modeled as 
\begin{equation}\label{eq:defsquare}
    \bX = \frac{\alpha}{n}\bu^*{\bu^*}^{\sT} + \bW \in\mathbb R^{n\times n},
\end{equation}
where $\bu^*\in\mathbb R^n$ is the unknown rank-$1$ signal, $\bW\in\mathbb R^{n\times n}$ is a symmetric noise matrix, and $\alpha >0$ captures the signal-to-noise ratio (SNR). In the rectangular case, we observe the data matrix 
\begin{equation}\label{eq:defrect}
\bX = \frac{\alpha}{m}\bu^*{\bv^*}^\sT + \bW\in\mathbb R^{m\times n},
\end{equation}
where $\bu^*\in\mathbb R^m$ and $\bv^*\in\mathbb R^n$ are the unknown signals, and $\bW\in\mathbb R^{m\times n}$ is a rectangular noise matrix.
A natural estimator of the signal in the symmetric case is the principal eigenvector of $\bX$ (singular vectors, in the rectangular case). The performance of this principal component analysis (PCA) estimator and, more generally, the behavior of eigenvalues and eigenvectors of models like \eqref{eq:defsquare}-\eqref{eq:defrect} has been widely studied in statistics \cite{johnstone2001distribution, paul2007asymptotics} and random matrix theory \cite{baik2005phase, baik2006eigenvalues, benaych2011eigenvalues, benaych2012singular, capitaine2009largest, feral2007largest, knowles2013isotropic}.  

If $\bu^*, \bv^*$ are unstructured (e.g., they are uniformly distributed on a sphere), then it is not generally possible to improve on the PCA estimator. 
 However,  in a broad range of applications, the unknown signals have some underlying structure, e.g., they may be sparse, their entries may belong to a certain set, or they may be modelled using a prior distribution. Examples of structured matrix estimation problems include sparse PCA \cite{deshpande2014information, johnstone2009consistency, zou2006sparse}, non-negative PCA \cite{lee1999learning, montanari2016non}, community detection under the stochastic block model \cite{abbe2017community, deshpande2016asymptotic, moore2017computer}, and group synchronization \cite{perry2018message}.  Since PCA is ill-equipped to capture the structure of the signal, we aim to improve on it using a family of iterative algorithms known as approximate message passing (AMP). AMP algorithms have two particularly attractive features: \emph{(i)} they can be tailored to take advantage of prior information on the structure of the signal; and \emph{(ii)} under suitable model assumptions, their performance in the high-dimensional limit is precisely characterized by a succinct deterministic recursion called \emph{state evolution} \cite{BM-MPCS-2011, bolthausen2014iterative, javanmard2013state}.  
AMP algorithms have been applied to a wide range of inference problems: estimation in linear models \cite{BayatiMontanariLASSO,BM-MPCS-2011,DMM09,  krzakala2012,maleki2013asymptotic}, generalized linear models \cite{barbier2019optimal,ma2019optimization, maillard2020phase,mondelli2021approximate,RanganGAMP,schniter2014compressive,sur2019modern}, and
low-rank matrix estimation with Gaussian noise \cite{BarbierMR20,deshpande2014information,fletcher2018iterative,kabashima2016phase, lesieur2017constrained,montanari2017estimation}. The survey \cite{feng2021unifying} provides a unified description of AMP for these applications. Using the state evolution analysis, it has been proved that AMP achieves Bayes-optimal performance in some Gaussian models \cite{deshpande2014information,DonSpatialC13, montanari2017estimation}, and a bold conjecture from statistical physics posits that AMP is optimal among polynomial-time algorithms.

We study  rank-1 matrix estimation in the setting where the noise matrix $\bW$ is rotationally invariant. This is a much milder assumption than $\bW$ being Gaussian: it only imposes that the orthogonal matrices in the spectral decomposition of $\bW$ are uniformly random, and allows for arbitrary eigenvalues/singular values. Hence, $\bW$ can capture a more complex correlation structure, which is typical in applications. For the models \eqref{eq:defsquare}-\eqref{eq:defrect} with rotationally invariant noise, AMP algorithms were derived in \cite{ccakmak2019memory,opper2016theory} and generalized in \cite{fan2020approximate}. In particular, the AMP algorithm of \cite{fan2020approximate} for the problem \eqref{eq:defsquare} produces estimates $\bu^t \in \reals^n$  as follows:
\begin{equation}\label{eq:AMP0}
    \bu^t = \su_t(\bdff^{t-1}), \quad \bdff^t= \bX \bu^t - \sum_{i=1}^t \sb_{t,i}\bu^i, \qquad t \ge 2.
\end{equation}
The iteration is initialized with a pilot estimate $\bu^1$. We can interpret \eqref{eq:AMP0} as a generalized power method.  Recall that the power method approximates the principal eigenvector of $\bX$ using the iterative updates $\bar{\bu}^t = \bX \bar{\bu}^{t-1}/ \| \bX \bar{\bu}^{t-1} \|$. For each $t$, the function $\su_t$ can be chosen to exploit any structural information known about  the signal (e.g., sparsity).   The ``memory'' coefficients $\{ \sb_{t,1}, \ldots, \sb_{t,t} \}$ have a specific form  to  ensure that the iterates $(\bdff^{t}, \bu^{t+1})$ have desirable statistical properties captured by state evolution. 
A rigorous state evolution result for the iteration \eqref{eq:AMP0} is established in \cite{fan2020approximate}, but the algorithm  and its analysis require an initialization $\bu^1$ that is correlated with the unknown signal and independent of the noise $\bW$. In practice, one typically does not have access to such an initialization. 

\paragraph{Main contribution.} In this paper, we propose an AMP algorithm initialized via the PCA estimator, namely, the principal eigenvector of $\bX$ for the square case \eqref{eq:defsquare} and the left singular vector of $\bX$ for the rectangular case \eqref{eq:defrect}. Our main technical contribution is a state evolution result for this AMP algorithm, which gives a rigorous characterization of its performance in the high-dimensional limit. The challenge is that, as the PCA initialization depends on the noise matrix $\bW$, one cannot apply the state evolution machinery of \cite{fan2020approximate}. To circumvent this issue, our key idea is to construct and analyze a \emph{two-phase artificial AMP} algorithm. In the first phase, the artificial AMP performs a power method approaching the PCA estimator; and in the second phase, it mimics the behavior of the true AMP. We remark that the artificial AMP only serves as a proof technique. Thus, we can initialize it with a vector correlated with the signal $\bu^*$ and independent of the noise matrix $\bW$, which allows us to analyze it using the existing state evolution result. 

Our analysis is tight in the sense that our AMP algorithm can be initialized with PCA whenever the PCA estimate has strictly positive correlation with the signal. This requires showing that, when PCA is effective, the state evolution of the first phase of the artificial AMP has a unique fixed point. To obtain such a result, we exploit free probability tools developed in \cite{benaych2011eigenvalues, benaych2012singular}. The agreement between the practical performance of AMP and the theoretical predictions of state evolution is demonstrated via numerical results for different spectral distributions of $\bW$. Our simulations also show that the performance of AMP---as well as its ability to improve upon the PCA initialization---crucially depends on the choice of the denoising functions $\su_t$ in the algorithm. Thus, the design of a Bayes-optimal AMP remains an exciting avenue for future research. 

\paragraph{Related work.}   The asymptotic Bayes-optimal error for low-rank matrix estimation has been  precisely characterized for Gaussian noise \cite{barbier2016mutual, lelarge2019fundamental}, but remains an open problem for rotationally invariant noise. An AMP algorithm with PCA initialization was proposed in \cite{montanari2017estimation} for the Gaussian setting, and it was shown to be Bayes-optimal for some signal priors.  A recent paper by Zhong et al. \cite{zhong2021empirical} shows how AMP with PCA initialization can be used for estimating the top-$k$ principal components  in applications such as high-dimensional genomics datasets. The authors use an empirical Bayes method to determine a  joint prior distribution for the $k$ principal components, and assuming a Gaussian noise model, employ an AMP algorithm tailored to the prior  to improve the PCA estimates of the principal components.

Both our AMP algorithm and proof technique differ significantly  from those for Gaussian noise.  When $\bW$ is Gaussian, the PCA initialization affects only the first iteration of AMP. In contrast, for more general noise distributions, the AMP algorithm and its associated state evolution require a correction term at every iteration to account for the PCA initialization. This is due to the fact that, while AMP has a single memory term in the Gaussian case, more general noise distributions lead to a more involved memory structure, as  in \eqref{eq:AMP0}. As regards the proof technique, the argument of \cite{montanari2017estimation} consists of decoupling the PCA estimate from the bulk of the spectrum of $\bX$. In contrast,  our approach is based on a two-phase artificial AMP algorithm. This technique has proved successful in the context of generalized linear models \cite{mondelli2020optimal,mondelli2021approximate}, albeit for Gaussian measurements.  Other extensions of AMP beyond the Gaussian setting include Orthogonal AMP \cite{ma2017orthogonal, takeuchi2020rigorous}, Vector AMP \cite{gerbelot2020asymptotic1, gerbelot2020asymptotic2,rangan2019vector,schniter2016vector}, convolutional AMP \cite{takeuchi2021bayes} and Memory AMP \cite{LiuMemoryAMP20}. These algorithms have been derived specifically for linear or generalized linear models, and extending them (with a practical initialization method) to low-rank matrix estimation is an interesting research direction.

Finally, we mention the recent independent work of Zhong et al.  \cite{zhong2021approximate}, which appeared after the original submission of our paper. This work generalizes AMP with PCA initialization to the problem of estimating rank-$k$ matrices in rotationally invariant noise, for $k \ge 1$. We remark that, in order
to prove a state evolution result for AMP initialized with PCA, in \cite{zhong2021approximate} it is assumed that the signal strength is sufficiently large. In contrast, our result holds for any signal strength such that the PCA method is effective, but we require the free cumulants of the noise matrix to be non-negative. We also note that, when the signal strength is large, the assumption on the free cumulants can be automatically satisfied (see the footnote on p.\pageref{foot}). 


\section{Preliminaries}\label{sec:prel}

\paragraph{Notation and definitions.} Given $a\in\mathbb R$, we define $(a)_+=\max(a, 0)$. Given two integers $i \le j$, we define $[i, j]=\{i, \ldots, j\}$. If $i>j$, then $[i, j]$ denotes the empty set; products over the empty set are taken to be 1. Given a vector $\bx\in\mathbb R^{n}$, we denote by $\|\bx\|$ its Euclidean norm and by $\langle \bx\rangle$ its empirical mean, i.e., $\langle \bx\rangle=\frac{1}{n}\sum_{i=1}^n x_i$. The empirical distribution of $\bx=(x_1, \ldots, x_n)^\sT$ is given by $\frac{1}{n}\sum_{i=1}^n \delta_{x_i}$, where $\delta_{x_i}$ denotes a Dirac delta mass on $x_i$. The notation $\bx\stackrel{\mathclap{W}}{\longrightarrow} X$ denotes convergence of the empirical distribution of $\bx$ to the random variable $X$ in Wasserstein distance at all orders.  Given a symmetric square matrix $\bA\in\mathbb R^{n\times n}$, we denote by $\lambda_1(\bA)\ge \lambda_2(\bA)\ge \ldots \ge\lambda_n(\bA)$ its eigenvalues sorted in decreasing order. Given a rectangular matrix $\bA\in\mathbb R^{m\times n}$, with $m<n$, we denote by $\sigma_1(\bA)\ge \sigma_2(\bA)\ge \ldots \sigma_m(\bA)$ its singular values sorted in decreasing order. 

\paragraph{Rank-$1$ estimation -- Symmetric square matrices.} Consider the problem of estimating the signal $\bu^*\in\mathbb R^n$ from the data matrix $\bX$  in \eqref{eq:defsquare}. We assume that $\bW$ is rotationally invariant in law, i.e., $\bW = \bO^\sT \bLambda \bO$, where $\bLambda={\rm diag}(\blambda)$ is a diagonal matrix containing the eigenvalues of $\bW$ and $\bO$ is a Haar orthogonal matrix independent of $\bLambda$. As $n\to\infty$, we assume that the empirical distributions of $\blambda$ and $\bu^*$ satisfy 
\beq
\blambda\stackrel{\mathclap{W}}{\longrightarrow} \Lambda \quad \text{ and }  \quad  \bu^*\stackrel{\mathclap{W}}{\longrightarrow} U_*,
\label{eq:limiting_laws}
\eeq
where $\Lambda$ and $U_*$ represent the limiting spectral distribution of the noise and the prior on the signal, respectively.
We take $\|\bu\|=\sqrt{n}$ so that $\mathbb E\{U_*^2\}=\lim_{n\to\infty}\frac{1}{n}\|\bu^*\|^2=1$. We assume that the moment $\E\{ U_*^{2+\varepsilon} \} < \infty$ for some $\varepsilon >0$.  We also assume that $\Lambda$ has compact support, and denote by $b$ the supremum of this support. We denote by $\{\kappa_k\}_{k\ge 1}$ the free cumulants corresponding to the moments $\{m_k\}_{k\ge 1}$ of the empirical eigenvalue distribution of $\bX$ excluding its largest eigenvalue, i.e., $m_k = \frac{1}{n}\sum_{i=2}^n\lambda_i(\bX)^k$ (for details, see \eqref{eq:mcrel1}-\eqref{eq:mcrel2} in Appendix \ref{sec:fptools}). The assumption \eqref{eq:limiting_laws} implies that, as $n\to\infty$, $m_k \to m^\infty_k=\mathbb E\{\Lambda^k\}$ and $\kappa_k \to \kappa^\infty_k$, where $\{m_k^\infty\}_{k\ge 1}$ and $\{\kappa_k^\infty\}_{k\ge 1}$ are respectively moments and free cumulants of $\Lambda$.

\paragraph{PCA -- Symmetric square matrices.} Let $\bu_{\rm PCA}$ be the principal eigenvector of $\bX$, and define $\alpha_{\rm s}=1/G(b^+)$, where  $G(z)=\mathbb E\{(z-\Lambda)^{-1}\}$ is the Cauchy transform of $\Lambda$, and $G(b^+)=\lim_{z\to b^+}G(z)$. Then, for $\alpha>\alpha_{\rm s}$,  $\lambda_1(\bX)\stackrel{\mathclap{\mbox{\footnotesize a.s.}}}{\longrightarrow} G^{-1}(1/\alpha)$ and $\lambda_2(\bX)\stackrel{\mathclap{\mbox{\footnotesize a.s.}}}{\longrightarrow} b$, where $G^{-1}$ is the inverse of $G$; see Theorem 2.1  in \cite{benaych2011eigenvalues}. Furthermore, Theorem 2.2 in \cite{benaych2011eigenvalues} gives that, for $\alpha>\alpha_{\rm s}$, 
\beq 
\frac{\langle\bu_{\rm PCA}, \bu^*\rangle^2}{n} \stackrel{\mathclap{\mbox{\footnotesize a.s.}}}{\longrightarrow}\rho_\alpha^2=\frac{-1}{\alpha^2G'(G^{-1}(1/\alpha))} >0 . 
\label{eq:rho_alpha_def}
\eeq
 In words, above the spectral threshold $\alpha_{\rm s}$, the principal eigenvalue of $\bX$ escapes the bulk of the spectrum and its associated eigenvector becomes strictly correlated with the signal $\bu^*$. 

\paragraph{Rank-$1$ estimation -- Rectangular matrices.} Consider now the problem of estimating the signals $\bu^*\in\mathbb R^m$ and $\bv^*\in\mathbb R^n$ given the rectangular data matrix $\bX$ in \eqref{eq:defrect}. Without loss of generality, we assume that $m\le n$ (if $m> n$, one can just exchange the role of $\bu^*$ and $\bv^*$ and consider $\bX^\sT$ in place of $\bX$). We assume that $W$ is bi-rotationally invariant in law, i.e., $\bW = \bO^\sT \bLambda \bQ$, where $\bLambda={\rm diag}(\blambda)$ is a $m\times n$ diagonal matrix containing the singular values of $\bW$, and $\bO$, $\bQ$ are Haar orthogonal matrices independent of one another and also of $\bLambda$. As $n\to\infty$, we assume that $\blambda\stackrel{\mathclap{W}}{\longrightarrow} \Lambda$, $\bu^*\stackrel{\mathclap{W}}{\longrightarrow} U_*$, $\bv^*\stackrel{\mathclap{W}}{\longrightarrow} V_*$ and $m/n\to \gamma$, for some constant $\gamma\in (0, 1]$. We take $\|\bu\|=\sqrt{m}$ and $\|\bv\|=\sqrt{n}$ so that $\mathbb E\{U_*^2\}=\mathbb E\{V_*^2\}=1$. As before, $b<\infty$ is the supremum of the compact support of $\Lambda$, and $U_*, V_*$ are assumed to have finite $(2+\varepsilon)$-th moment for some $\varepsilon >0$ . To analyze PCA using the framework in \cite{benaych2012singular}, we also assume that the entries of $\bu^*$ and $\bv^*$ are i.i.d., and their law has zero mean and satisfies a log-Sobolev inequality. We denote by $\{\kappa_{2k}\}_{k\ge 1}$ the rectangular free cumulants associated to the even moments $\{m_{2k}\}_{k\ge 1}$, with $m_{2k}=\frac{1}{m}\sum_{i=2}^m\sigma_i(\bX)^{2k}$ (for details, see \eqref{eq:mcrel1rect}-\eqref{eq:mcrel2rect} in Appendix \ref{sec:fptools}). Furthermore, as $n,m\to\infty$, $m_{2k} \to m^\infty_{2k}=\mathbb E\{\Lambda^{2k}\}$ and $\kappa_{2k}\to \kappa_{2k}^\infty$, where $\{m_{2k}^\infty\}_{k\ge 1}$ and $\{\kappa_{2k}^\infty\}_{k\ge 1}$ are respectively even moments and rectangular free cumulants of $\Lambda$.

\paragraph{PCA -- Rectangular matrices.} Denote by $\bu_{\rm PCA}$ and $\bv_{\rm PCA}$ the left and right principal singular vectors  of $\bX$, and define $\tilde{\alpha}_{\rm s}=1/\sqrt{D(b^+)}$, where $D(z)=\phi(z) \bar{\phi}(z)$, $\phi(z)=\mathbb E\{z/(z^2-\Lambda^2)\}$, $\bar{\phi}(z)=\gamma\phi(z)+(1-\gamma)/z$, and $D(b^+)=\lim_{z\to b^+}D(z)$. Note that the singular value of the rank-one signal $\frac{\alpha}{m}\bu^*{\bv^*}^\sT$ is $\tilde{\alpha}\triangleq \alpha/\sqrt{\gamma}$. Then, for $\tilde{\alpha}>\tilde{\alpha}_{\rm s}$, 
$\sigma_1(\bX)\stackrel{\mathclap{\mbox{\footnotesize a.s.}}}{\longrightarrow} D^{-1}(1/\tilde{\alpha}^2)$ and $\sigma_2(\bX)\stackrel{\mathclap{\mbox{\footnotesize a.s.}}}{\longrightarrow} b$; see Theorem 2.8 in \cite{benaych2012singular}. Furthermore, Theorem 2.9 in \cite{benaych2012singular} gives that, for $\tilde{\alpha}>\tilde{\alpha}_{\rm s}$,
\begin{align}
    \label{eq:DeltaPCA_def}
    \frac{\langle\bu_{\rm PCA}, \bu^*\rangle^2}{m}\stackrel{\mathclap{\mbox{\footnotesize a.s.}}}{\longrightarrow}\Delta_{\PCA}=\frac{-2\phi(D^{-1}(1/\tilde{\alpha}^2))}{\tilde{\alpha}^2D'(D^{-1}(1/\tilde{\alpha}^2))}>0, \\
    \frac{\langle\bv_{\rm PCA}, \bv^*\rangle^2}{n} \stackrel{\mathclap{\mbox{\footnotesize a.s.}}}{\longrightarrow}\Gamma_{\PCA}=\frac{-2\bar{\phi}(D^{-1}(1/\tilde{\alpha}^2))}{\tilde{\alpha}^2D'(D^{-1}(1/\tilde{\alpha}^2))} >0.
    \label{eq:GammaPCA_def}
\end{align} 
In words, above the spectral threshold $\tilde{\alpha}_{\rm s}$, the principal singular value escapes from the bulk of the spectrum and the left/right principal singular vectors become correlated with the signal $\bu^*$/$\bv^*$.

\section{PCA Initialization for Approximate Message Passing}

\subsection{Symmetric Square Matrices}\label{subsec:square}

We consider a family of Approximate Message Passing (AMP) algorithms to estimate $\bu^*$ from $\bX = \frac{\alpha}{n}\bu^*{\bu^*}^{\sT} + \bW$. We initialize using the PCA estimate $\bu_{\rm PCA}$:
\begin{equation}\label{eq:AMPinit}
    \bu^1 = \sqrt{n}\bu_{\rm PCA}, \quad \bdff^1 = \bX \bu^1 - \sb_{1,1}\bu^1,
\end{equation}
with $\sb_{1,1} =\sum_{i=0}^\infty \kappa_{i+1}\alpha^{-i}$. Then, for $t\ge 2$, the algorithm computes 
\begin{equation}\label{eq:AMP}
    \bu^t = \su_t(\bdff^{t-1}), \quad \bdff^t= \bX \bu^t - \sum_{i=1}^t \sb_{t,i}\bu^i,
\end{equation}
where the memory coefficients $\{\sb_{t,i}\}_{i \in [1, t]}$ are given by $\sb_{t,t}=\kappa_1$, and 
\begin{equation}\label{eq:sbtdef}
\begin{split}
  &\sb_{t,1}=\prod_{\ell=2}^t\langle \su_\ell'(\bdff^{\ell-1})\rangle\sum_{i=0}^\infty  \kappa_{i+t}\alpha^{-i}, \quad \sb_{t,t-j}=\kappa_{j+1} \hspace{-5pt}\prod_{i=t-j+1}^t \hspace{-5pt} \langle \su_i'(\bdff^{i-1})\rangle, \  \mbox{ for }\,\, (t-j) \in [2, t-1]. 
\end{split}
\end{equation}
Here, the function $\su_t:\mathbb R\to\mathbb R$ is continuously differentiable and Lipschitz, it is applied component-wise to vectors, i.e., 
$\su_t(\bdff^{t-1})=(\su_t(f^{t-1}_1), \ldots, \su_t(f^{t-1}_n))$, and $\su_t'$ denotes its derivative.  The AMP algorithm in \eqref{eq:AMPinit}-\eqref{eq:sbtdef} is similar to the one in \cite[Sec. 3.1]{fan2020approximate} (and the ones in \cite{ccakmak2019memory,opper2016theory}), with the main differences being the initialization $\bu^1$ and the formula for the memory term $\sb_{t,1}$. We highlight that the algorithm  does not require the knowledge of $\alpha$ or of the noise distribution. In fact, $\alpha$ can be consistently estimated from the principal eigenvalue of $\bX$ via $\hat{\alpha} = (G(\lambda_1(\bX)))^{-1}$. Furthermore, one can compute the moments $\{m_k\}_{k\ge 1}$ of the empirical eigenvalue distribution of $\bX$ (excluding its largest one) and, from these, deduce  the free cumulants $\{\kappa_k\}_{k\ge 1}$.

 The asymptotic empirical distribution of the  iterates $\bu^t, \bdff^t$, for $t\ge 1$, can be succinctly characterized via a deterministic recursion, called \emph{state evolution}, and expressed via a sequence of mean vectors $\bmu_K=(\mu_t)_{t\in [1, K]}$ and covariance matrices $\bSigma_K=(\sigma_{s, t})_{s, t\in [1, K]}$. For $K=1$, set 
$\mu_1 = \alpha \rho_\alpha$ and $\sigma_{11}=\alpha^2 (1-\rho_\alpha^2)$, with $\rho_\alpha$ given in \eqref{eq:rho_alpha_def}.
 Then define $\bmu_{K+1}, \bSigma_{K+1}$ from $\bmu_{K}, \bSigma_{K}$ as follows. Let
 \begin{align}
  &  (F_1, \ldots, F_K) = \bmu_K U_* +  (Z_1, \ldots, Z_K), \text{ where }    (Z_1, \ldots, Z_K) \sim \normal(\bzero, \bSigma_K), \quad   
  \label{eq:F1FK_def}\\
   & U_t = \su_t(F_{t-1}) \   \text{ for } 2 \leq t \leq K+1, \quad  \text{ and } \quad  U_t = \frac{F_1}{\alpha} \  \text{ for }  -\infty < t \leq 1. \label{eq:Ut_def}
 \end{align}
Then, the entries of $\bmu_{K+1}$ are given by $\mu_t = \alpha \E\{ U_t U_* \}$ for $t \in [1, K+1]$. Furthermore, the entries of $\bSigma_{K+1}$ can be expressed via the following formula, for $s,t \in[1, K+1]$:

\begin{align}
    \sigma_{s,t} & = \sum_{j=0}^{\infty} \sum_{k=0}^{\infty} \kappa_{j+k+2}^\infty  \left(\frac{1}{\alpha}\right)^{(k-t+1)_{+}  +  (j-s+1)_{+}}\hspace{-7em} \cdot \hspace{1em}  \E\{U_{s-j} U_{t-k} \}  \cdot  \Big(\hspace{-2em} \prod_{i=\max(2, s+1-j)}^s \hspace{-2em}\E\{ \su_i'(F_{i-1})\}\Big) \cdot\Big( \hspace{-2em}\prod_{i=\max(2, t+1-k)}^t \hspace{-2em}\E\{ \su_i'(F_{i-1})\}\Big).
    \label{eq:sq_sigma_st}
\end{align}

Our main result, Theorem \ref{thm:square}, shows that for $t \geq 1$, the empirical joint distribution of the entries of $(\bu^*, \bdff^1, \ldots, \bdff^t)$ converges in Wasserstein distance $W_2$ to the law of the random vector $(U_*, F_1, \ldots, F_t)$. We provide a proof sketch in Section \ref{sec:proof_sketch}, and the complete proof is deferred to Appendix \ref{app:pfsq}. 
 This result is stated in terms of \emph{pseudo-Lipschitz} test functions.
A function $\psi: \reals^m \to \reals$ is pseudo-Lipschitz of order $2$, i.e., $\psi \in \PL(2)$, if there is a constant $C > 0$  such that 
\beq
\norm{\psi(\bx)-\psi(\by)} \le C (1 + \| \bx \| + \| \by \| )\norm{\bx - \by}.
\label{eq:PL2_prop}
\eeq
The equivalence between convergence in terms of $\PL(2)$ functions and convergence in $W_2$ distance follows from 
 \cite[Definition 6.7 and Theorem 6.8]{villani2008optimal}.
 \begin{theorem}\label{thm:square}
In the square symmetric model \eqref{eq:defsquare}, assume that $\alpha > \alpha_{\rm s }$, and that the free cumulants of order 2 and higher are non-negative, i.e., $\kappa_k^\infty\ge 0$ for $k\ge 2$. 
Consider the AMP algorithm with PCA initialization in \eqref{eq:AMPinit}-\eqref{eq:AMP}, with continuously differentiable and Lipschitz functions $\su_t:\mathbb R\to\mathbb R$. (Without loss of generality, assume that $\< \bu^*, \bu^{\PCA} \> \ge 0$.)

Then, for $t \geq 1$ and any  \PL($2$) function $\psi:\reals^{2t+2} \to\reals$, we almost surely have:
\begin{align}
\lim_{n \to \infty} \frac{1}{n} \sum_{i=1}^n \psi (u^*_{i}, u^1_i, \ldots, u^{t+1}_i, f^1_i, \ldots f^{t}_i)
 = \E \left\{ \psi(U_*, U_1, \ldots, U_{t+1}, F_1, \ldots, F_t) \right\},
 \label{eq:thm_sq_statement}
\end{align}
where $U_1, \ldots, U_{t+1}$ and $F_1, \ldots, F_t$ are defined  in \eqref{eq:F1FK_def}.
 \end{theorem}

\paragraph{Assumptions of the theorem.} The basic assumption that the noise matrix is rotationally invariant is rather mild as it allows for arbitrary eigenvalue distributions.
The assumption $\alpha > \alpha_{\rm s }$ ensures that the PCA initialization is correlated with the signal. This condition is necessary and sufficient for PCA to be effective: under the additional requirement that $G'(b)=-\infty$, we have that, if $\alpha < \alpha_{\rm s }$, then the normalized correlation between $\bu_{\rm PCA}$ and $\bu^*$ vanishes almost surely; see Theorem 2.3 of \cite{benaych2011eigenvalues}.
Conversely for $ \alpha > \alpha_{\rm s} $, the asymptotic correlation is strictly non-zero and given by \eqref{eq:rho_alpha_def}.

\emph{Non-negativity of free cumulants}: The assumption that   $\kappa^\infty_k\ge 0$ for $k\ge 2$ appears to be an artifact of the proof technique. As detailed in the proof sketch in Section \ref{sec:proof_sketch}, this assumption is needed to show that the state evolution of the artificial AMP in the first phase has a unique fixed point. We expect our approach to generalize to any limiting noise  distribution $\Lambda$ with compact support, and defer such a generalization to future work. In support of this view, the simulations of Section \ref{sec:simu} verify the claim of  Theorem \ref{thm:square}  in a setting where the free cumulants of $\Lambda$ have alternating signs (corresponding to an eigenvalue distribution $\Lambda \sim \text{Uniform}[-1/2, 1/2]$; see  Figs. \ref{fig:AMPSEsq2}-\ref{fig:AMPSErr} and \ref{fig:AMPspectsq2}--\ref{fig:AMPspectrr}). Finally, we remark that, if $\bW$ follows a Marcenko-Pastur distribution ($\bW = \bA\bA^\sT$, where $\bA$ has i.i.d. Gaussian entries), then the free cumulants of $\Lambda$ are all equal and strictly positive; see \cite[Chap. 2, Exercise 11]{mingo2017free}. Thus, the assumption of Theorem \ref{thm:square} holds for noise distributions that are sufficiently close to the Marcenko-Pastur one, or for sufficiently large values of the signal-to-noise ratio $\alpha$.\footnote{One can add an independent artificial noise matrix with Marcenko-Pastur distribution to the data in order to make the required free cumulants non-negative, and the result would hold for $\alpha$ greater than the new spectral threshold.}\label{foot}

 \emph{Continuous differentiability and other technical assumptions}: The assumption that $\su_t$ is continuously differentiable can be weakened to: \emph{(i)} $\su_t$ being differentiable almost everywhere, and \emph{(ii)} satisfying a mild non-degeneracy condition (Assumption 4.2(e) in \cite{fan2020approximate}). In this way, we can cover most practically relevant choices of $\su_t$ such as soft thresholding and ReLU. Theorem \ref{thm:square} also requires the technical assumptions in \eqref{eq:limiting_laws} and the text below it: convergence of the empirical distributions of the signal and of the eigenvalues of the noise matrix; boundedness of the $(2+\varepsilon)$-moment of the signal; and compact support of the spectrum of the noise matrix.  We regard these technical assumptions as minor, and remark that they are quite standard in the literature. For the rectangular case, we also need the additional assumption that the law of the signal is zero mean and satisfies a log-Sobolev inequality, which is necessary to apply the framework in \cite{benaych2011eigenvalues}. 

\paragraph{How PCA initialization influences AMP.} 
The form of the memory coefficient $\sb_{t,1}$ in \eqref{eq:sbtdef} reflects the PCA initialization of the AMP iteration. PCA  initialization  can be interpreted as the result of a first AMP phase with linear denoisers (see the proof sketch in Sec. \ref{sec:proof_sketch}). The  coefficient $\sb_{t,1}$ multiplying the initialization $\bu_1$ represents the cumulative effect of this first AMP phase leading to the PCA estimate. The main differences from the AMP algorithm in \cite{fan2020approximate} (where the initialization is independent of $\bW$) are the expressions for the  coefficient $\sb_{t,1}$ and the state evolution parameters $\sigma_{s, t}$ (compare \eqref{eq:sbtdef} and \eqref{eq:sq_sigma_st} in this paper with (1.15) and (1.17) in \cite{fan2020approximate}). One can interpret the new form of $\sb_{t,1}$ and $\sigma_{s, t}$ as a memory of the PCA initialization.   For the special case of Gaussian noise, the spectral initialization  only affects  the first iteration of AMP \cite{montanari2017estimation}. This is due to the fact that, while in a rotationally invariant model the AMP iterate at step $t$ depends on \emph{all} previous iterates,  in the Gaussian case it depends only on the iterate at step $t-1$.

\paragraph{Choice of $\su_t(\cdot)$.} Theorem \ref{thm:square} holds for any choice of denoisers $\{ \su_t \}$ that are Lipschitz and continuously differentiable. Indeed, our analysis shows that by picking $\su_t(f)  =f /\alpha$, AMP just returns the PCA estimate; see the proof sketch in Section \ref{sec:proof_sketch}. If some structural information about the signal is available (e.g., sparsity), denoisers that take advantage of this structure can give substantial improvements over PCA. Thus, a key question is how to optimally select the $\su_t$'s. Theorem \ref{thm:square} tells us that the empirical distribution of  $\bdff^t$ converges to the law of $\mu_t U_* + \sqrt{\sigma_{t,t}} Z$, for $Z \sim \normal(0,1)$ and independent of $U_*$. Hence,  the quality of the estimate at each iteration $t$ is governed by the SNR $\rho_t := \mu_t^2/\sigma_{t,t}$. Consider running the algorithm for $\bar{t}$ iterations, and let $\bu^{\bar{t}+1} = \su_{\bar{t}+1}(\bdff^{\bar{t}})$ be the final estimate.  Then, for each $t \in [2, \bar{t} \, ]$, the Bayes-optimal choice for $\su_t$ is the one that maximizes $\rho_{t}$, i.e., the SNR for the next iteration. In the case of Gaussian noise \cite{montanari2017estimation}, the maximum is achieved by the posterior mean $\su_t(f) = \E\{ U_* \, | \, \mu_t U_* + \sqrt{\sigma_{t,t}} Z =f  \}$. For rotationally invariant noise, this choice  minimizes the mean-squared error $\frac{1}{n} \| \bu^t - \bu^* \|^2$ (for fixed $\su_1, \ldots, \su_{t-1}$), but it \emph{does not} necessarily maximize the SNR $\rho_t$.  We provide an example of this behavior in the simulations reported in Section \ref{sec:simu}. Therefore, the optimal strategy would be to choose functions $u_2, \ldots, u_{\bar{t}}$ to maximize the SNRs $\rho_2, \ldots, \rho_{\bar{t}}$, and then in the final iteration, to pick $u_{\bar{t}+1}$ to minimize the desired loss. Note that $\su_t$ depends on the previously chosen functions $u_1, \ldots, \su_{t-1}$ in a complicated way, due to the definition of $\sigma_{t,t}$ in \eqref{eq:sq_sigma_st}. Thus, finding $\su_t$ that maximizes the SNR $\rho_t$ remains an outstanding challenge. Finally, we remark that though we only consider one-step denoisers in this paper, Theorem \ref{thm:square} can be readily extended to cover denoisers with memory, i.e., those of the form $\su_t(\bdff^1, \ldots, \bdff^{t-1})$. \label{par:choice_ut}

\subsection{Rectangular Matrices}

We now present an AMP algorithm to estimate $\bu^*$ and $\bv^*$ from the $m \times n$ data matrix $\bX = \frac{\alpha}{m} \bu^* {\bv^*}^\sT + \bW$. We initialize the algorithm using the PCA estimate $\bu_{\rm PCA}$:
\begin{align}
\label{eq:AMP_rect_init}
    \bu^1 = \sqrt{m} \, \bu_{\rm PCA}, \quad 
    \bg^1 = \Bigg( 1 + \gamma \sum_{i=1}^\infty \kappa_{2i} \Big( \frac{\gamma}{\alpha^2}\Big)^i\Bigg)^{-1}\hspace{-5pt}\bX^{\sT}\bu^1, \,  \,  \quad \bv^1 = \sv_1(\bg^1)=\frac{\gamma}{\alpha} \bg^1.
\end{align}
Then, for $t \geq 1$, we iteratively compute:
\begin{align}
    \bdff^t \hspace{-.15em}=\hspace{-.15em}\bX \bv^t \hspace{-.15em} - \hspace{-.15em} \sum_{i=1}^t  \sa_{t,i} \bu^i, \hspace{.95em}\bu^{t+1}\hspace{-.15em}=\hspace{-.15em}\su_{t+1}(\bdff^t), \hspace{.95em}      \bg^{t+1}\hspace{-.15em} =\hspace{-.15em}\bX \bu^{t+1} \hspace{-.15em} - \hspace{-.15em} \sum_{i=1}^{t}  \sb_{t+1,i} \bv^i, \hspace{.95em} \bv^{t+1}\hspace{-.15em}=\hspace{-.15em} \sv_{t+1}(\bg^{t+1}).
     \label{eq:AMP_rectg}
\end{align}
Here, $\su_{t+1}, \sv_{t+1}: \reals \to \reals$ are continuously differentiable Lipschitz functions that act component-wise on vectors. We define  $\sa_{1,1} =\alpha \sum_{i=1}^\infty \kappa_{2i} \big( \frac{\gamma}{\alpha^2}\big)^i$, and for $t \geq 2$:
\begin{align}
    & \sa_{t,1} = \< \sv_t'(\bg^t) \> \prod_{i=2}^t \< \su_i'(\bdff^{i-1}) \>
    \< \sv_{i-1}'(\bg^{i-1}) \> \left( \sum_{i=0}^\infty \kappa_{2(i+t)} \Big( \frac{\gamma}{\alpha^2}\Big)^i \right), \\
    & \sa_{t,t-j} = \< \sv_t'(\bg^t) \> \prod_{i=t-j+1}^t \< \su_i'(\bdff^{i-1}) \>
    \< \sv_{i-1}'(\bg^{i-1}) \> \kappa_{2(j+1)}, \qquad \mbox{ for }\,\, (t-j) \in [2,t].
\end{align}
Furthermore, for $t \geq 1$,
\begin{align}
   &  \sb_{t+1,1} =  \gamma \< \su'_{t+1}(\bdff^t) \> \prod_{i=2}^{t} 
   \< \sv_i'(\bg^i)\> \< \su_i'(\bdff^{i-1}) \> 
   \left( \kappa_{2t} \, +  \, \sum_{i=1}^\infty \kappa_{2(i+t)} \Big( \frac{\gamma}{\alpha^2}\Big)^i\right), \\
   & \sb_{t+1, t+1-j} = \gamma \< \su'_{t+1}(\bdff^t) \> \prod_{i=t+2-j}^{t} 
   \< \sv_i'(\bg^i)\> \< \su_i'(\bdff^{i-1}) \>  \, \kappa_{2j}, \qquad \mbox{ for }\,\,(t+1-j) \in [2, t].
\end{align}
Similarly to the square case, $\alpha$ can be consistently estimated from the largest singular value of $\bX$ via $\alpha=\sqrt{\gamma(D(\sigma_1(\bX)))^{-1}}$, and the rectangular free cumulants $\{\kappa_{2k}\}_{k\ge 1}$ can be obtained from the even moments of the empirical distribution of the singular values of $\bX$ (excluding its largest one). 

The asymptotic empirical distributions of the iterates $(\bdff^t \, , \, \bg^t)$ can be characterized via a state evolution recursion, which specifies a sequence of mean vectors $\bmu_K=(\mu_t)_{t\in [0,K]}, \,  \bnu_K=(\nu_t)_{t\in [1, K]} $ and covariance matrices $\bSigma_K=(\sigma_{s, t})_{s, t\in [0,K]},  \bOmega_K=(\omega_{s, t})_{s, t\in [1, K]}$. These are iteratively defined, starting with the initialization $\mu_0 = \alpha \sqrt{\Delta_{\PCA}}$ and $\sigma_{0,0}= \alpha^2(1- \Delta_{\PCA})$,
where $\Delta_{\PCA}$ is given by \eqref{eq:DeltaPCA_def}.
Having defined $\bmu_K, \bSigma_K, \bnu_K,  \bOmega_K$, let 
 \begin{align}
  &  (F_0, \ldots, F_K) = \bmu_K U_* +  (Y_0, \ldots, Y_K), \text{ where }    (Y_0, \ldots, Y_K) \sim \normal(\bzero, \bSigma_K), \quad   
  \label{eq:F0FK_rect}\\
   & U_t = \su_t(F_{t-1}) \,  \text{ for } \, 2 \leq t \leq K+1, \quad  \text{ and } \quad  U_t = \frac{F_0}{\alpha} \  \text{ for }  -\infty < t \leq 1, 
    \end{align}
    \begin{align}
    &  (G_1, \ldots, G_K) = \bnu_K V_* +  (Z_1, \ldots, Z_K), \text{ where }    (Z_1, \ldots, Z_K) \sim \normal(\bzero, \bOmega_K),   \label{eq:G0GK_rect}\\
     & V_t = \sv_t(G_t) \,  \text{ for } \,  2 \leq t \leq K+1, \quad  \text{ and } \quad  V_t = \frac{\gamma}{\alpha}G_1 \  \text{ for }  -\infty < t \leq 1. \label{eq:V0VK_rect}
 \end{align}
Given $\bmu_K$
 and $\bSigma_K$, the entries of $\bnu_{K+1}$ are given by $\nu_t = \alpha \E\{ U_{t} U_* \}$ (for $t \in [1, K+1]$), and the entries of $\bOmega_{K+1}$ (for $s+1,t+1 \in[1, K+1]$) are given by
\begin{align}
    \omega_{s+1,t+1} & =  \sum_{j=0}^\infty \sum_{k=0}^\infty
    \gamma  \left( \frac{\gamma}{\alpha^2} \right)^{(j-s)_+ + (k-t)_+}
    \Big(\hspace{-0.25em} \prod_{i=\max(2, s+2-j)}^{s+1} \hspace{-0.25em}\sx_{i}\cdot \sy_{i-1}\Big)  \cdot  \Big(\hspace{-0.25em} \prod_{i=\max(2, t+2-k)}^{t+1}\hspace{-0.25em} \sx_{i}\cdot \sy_{i-1}\Big) \nonumber \\
    &   \cdot \Big[ \kappa_{2(j+ k+1)}^\infty \E\{ U_{s+1-j} U_{t+1-k} \} \, + \, \kappa_{2(j+k+2)}^\infty \E\{ V_{s-j} V_{t-k} \} \sx_{s+1-j}\cdot \sx_{t+1-k} \Big]. 
    \label{eq:omega_rect}
\end{align}
    %
    %
Here, we define $\sx_{i} = \E\{ u_{i}'(F_{i-1}) \}$ if $i \geq 2$, and $\sx_{i} = 1/\alpha$ otherwise; $\sy_{i} = \E\{ v_{i}'(G_{i}) \}$ if $i \geq 2$, and $\sy_{i} = \gamma/\alpha$ otherwise.    We note that 
$\omega_{11}$ is computed by solving the linear equation obtained by setting $s=t=0$ in \eqref{eq:omega_rect} (see \eqref{eq:omega_11}).
Next, given $\bnu_{K+1}$ and $\bOmega_{K+1}$ for some $K \ge 1$, the entries of $\bmu_{K+1}$ are  $\mu_t = \frac{\alpha}{\gamma} \E\{ V_{t} V_* \}$ (for $t \in [0, K+1]$), and the entries of $\bSigma_{K+1}$ (for $s,t \in [0, K+1]$) are 
\begin{align}
    \sigma_{s,t} & = \sum_{j=0}^\infty \sum_{k=0}^\infty
    \left( \frac{\gamma}{\alpha^2} \right)^{(j-s+1)_+ \, + \,  (k-t+1)_+}
    \Big( \prod_{i=\max(2, s+1-j)}^{s} \sx_{i}\cdot \sy_{i}\Big)  \cdot  \Big( \prod_{i=\max(2, t+1-k)}^{t} \sx_{i}\cdot \sy_{i}\Big) \nonumber  \\
    & \qquad \  \cdot \Big[ \kappa_{2(j+ k+1)}^\infty \E\{ V_{s-j} V_{t-k} \} \, + \, \kappa_{2(j+k+2)}^\infty \E\{ U_{s-j} U_{t-k} \} \sy_{s-j} \cdot\sy_{t-k} \Big]. 
    \label{eq:sigma_rect}
\end{align}
       %
    %

Our main result for the rectangular case, Theorem \ref{thm:rect}, shows that for $t \geq 1$, the empirical joint distribution of the entries of $(\bu^*, \bdff^1, \ldots, \bdff^t)$ converges in Wasserstein distance $W_2$ to the law of the random vector 
$(U_*, F_1, \ldots, F_t)$. Similarly, the empirical joint distribution of the entries of $(\bv^*, \bg^1, \ldots, \bg^t)$ converges to the law of $(V_*, G_1, \ldots, G_t)$. The proof is given in Appendix \ref{app:pfrect}. As in the square case, we state this result in terms of pseudo-Lipschitz test functions. 
\begin{theorem}\label{thm:rect}
In the rectangular model \eqref{eq:defrect}, assume that $\tilde{\alpha} > \tilde{\alpha}_{\rm s }$ and that $\kappa_{2k}^\infty\ge 0$ for $k\ge 1$. 
Consider the AMP algorithm with PCA initialization in \eqref{eq:AMP_rect_init}-\eqref{eq:AMP_rectg}, with continuously differentiable and Lipschitz functions $\su_t, \sv_t:\mathbb R\to\mathbb R$. (Assume without loss of generality that $\< \bu^*, \bu^{\PCA} \> \ge 0$.) 

Then, for $t \geq 1$ and any  \PL($2$) functions $\psi:\reals^{2t+2} \to\reals$ and $\varphi: \reals^{2t+1} \to \reals$, we almost surely have:
\begin{align}
\lim_{m \to \infty} \frac{1}{m} \sum_{i=1}^m 
 \psi (u^*_{i}, u^1_i, \ldots, u^{t+1}_i, f^1_i, \ldots f^{t}_i)
 = \E \left\{ \psi(U_*, U_1, \ldots, U_{t+1}, F_1, \ldots, F_t) \right\}, \label{eq:ustat_rect}  \\
 \lim_{n \to \infty} \frac{1}{n} \sum_{i=1}^n \varphi (v^*_{i}, v^1_i, \ldots, v^t_i, g^1_i, \ldots g^{t}_i)
 = \E \left\{ \varphi(V_*,V_1, \ldots, V_t, G_1, \ldots, G_t) \right\}, \label{eq:vstat_rect}
\end{align}
where $(U_1, \ldots, U_{t+1})$, $(F_1, \ldots, F_t)$, $(V_1, \ldots, V_{t})$ and $(G_1, \ldots, G_t)$ are defined as in \eqref{eq:F0FK_rect}-\eqref{eq:V0VK_rect}.
\end{theorem}

The condition $\tilde{\alpha} > \tilde{\alpha}_{\rm s }$ is necessary and sufficient for PCA to be effective: under the additional requirement that $\phi'(b^+)=-\infty$, if $\tilde{\alpha} < \tilde{\alpha}_{\rm s }$, then the normalized correlation between $\bu_{\rm PCA}$ and $\bu^*$ vanishes almost surely, see  \cite[Theorem 2.10]{benaych2012singular}. Comments similar to those at the end of Section \ref{subsec:square} can be made about \emph{(i)} the requirement that the rectangular free cumulants are non-negative, \emph{(ii)} the effect of the PCA initialization on AMP, and \emph{(iii)} the choice of the denoisers  $\su_t, \sv_t$.

\section{Numerical Simulations}\label{sec:simu}

We consider the following settings: \emph{(i)} square model \eqref{eq:defsquare} with Marcenko-Pastur noise, i.e., $\bW = \frac{1}{n}\bA \bA^\sT\in\mathbb R^{n\times n}$, where the entries of $\bA \in\mathbb R^{n\times p}$ are i.i.d. standard Gaussian, see (a) in the figures; \emph{(ii)} square model \eqref{eq:defsquare} with uniform noise, i.e., $\bW = \bO^\sT\bLambda\bO \in\mathbb R^{n\times n}$, where $\bO$ is a Haar orthogonal matrix and the entries of $\bLambda$ are i.i.d. and uniformly distributed in the interval $[-1/2, 1/2]$, see (b) in the figures; \emph{(iii)}   rectangular model \eqref{eq:defrect} with uniform noise, i.e., $\bW = \bO^\sT\bLambda\bQ \in\mathbb R^{m\times n}$, where $\bO$, $\bQ$ are Haar orthogonal matrices and the entries of $\bLambda^2$ are i.i.d. and uniformly distributed in the interval $[0, 1]$, see (c)-(d) in the figures.

In the simulations, $\alpha$ is estimated from the largest eigenvalue/singular value of $\bX$. Furthermore, the  free cumulants $\kappa_k$ ($\kappa_{2k}$ in the rectangular case) are replaced by their limits $\kappa_k^\infty$ ($\kappa_{2k}^\infty$ resp.), which are obtained as follows. For (a), all the free cumulants of $\Lambda$ are equal to $c\triangleq p/n$, i.e., $\kappa_k^\infty=c$ for $k\ge 1$, see \cite[Chap. 2, Exercise 11]{mingo2017free}. For (b), the odd free cumulants of $\Lambda$ are $0$ and the even ones are given by $\kappa_{2n}^\infty=B_{2n}/(2n!)$, where $B_{2n}$ denotes the $2n$-th Bernoulli number. For details, see the derivation of \eqref{eq:cumunif} in Appendix \ref{sec:fptools}. For (c)-(d), the even moments of $\Lambda$ are given by $m_{2k}^\infty=1/(k+1)$ and, from these, we numerically compute the rectangular free cumulants via \eqref{eq:mcrel2rect} in Appendix \ref{sec:fptools}. Furthermore,  the spectral threshold for the setting in (a) is $\alpha_{\rm s}=1+\sqrt{c}$ ; for (b), $\alpha_{\rm s}=0$ ; and for (c)-(d), $\tilde{\alpha}_{\rm s}=0$. In (a), we set $n=8000$ and $c=2$; in (b), we set $n=4000$; and in (c)-(d), we set $n=8000$ and $\gamma=1/2$. The signal $\bu^*$ has a Rademacher prior, i.e., its entries are i.i.d. and uniform in $\{-1, 1\}$. In the rectangular case, the signal $\bv^*$ has a Gaussian prior, i.e., it is uniformly distributed on the sphere of radius $\sqrt{n}$. Given these priors, $\su_t$ is chosen to be the single-iterate posterior mean denoiser given by 
$\su_t(x) = \tanh({\mu_t\, x}/\sigma_{t, t})$, where $\mu_t$ and $\sigma_{t, t}$ are the state evolution parameters; these are replaced by consistent estimates in the simulations. 
For the rectangular case, we choose $\sv_t(x)=x$.
 Each experiment is repeated for $n_{\rm trials}=100$ independent runs. We report the average and error bars at $1$ standard deviation.

 \begin{figure}[t]
    \centering
    \subfloat[Square, MP.\label{fig:AMPSEsq}]{\includegraphics[width=.25\columnwidth]{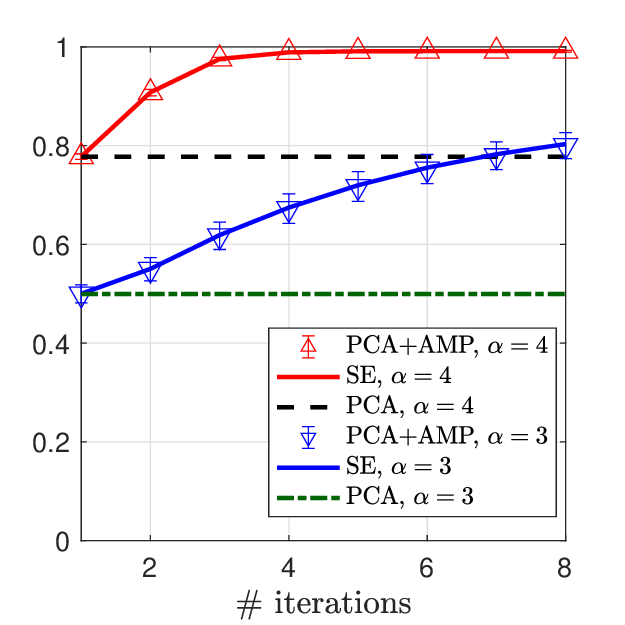}}
        \subfloat[Square, uniform.\label{fig:AMPSEsq2}]{\includegraphics[width=.25\columnwidth]{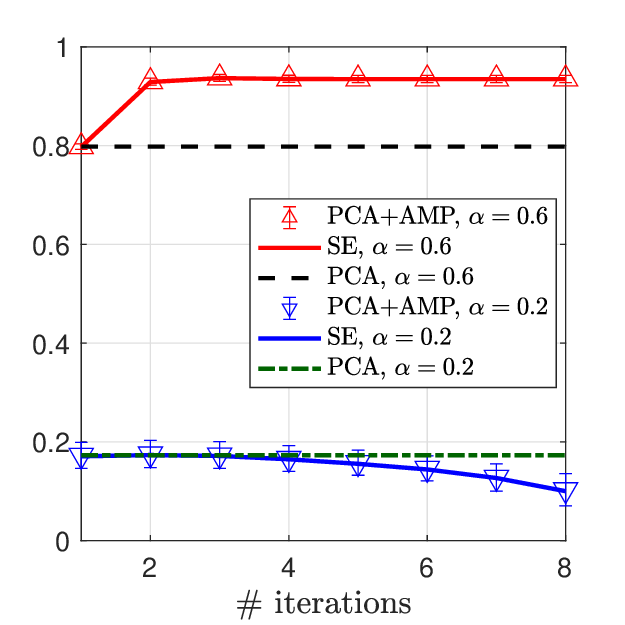}}
    \subfloat[Rectangular, left.\label{fig:AMPSErl}]{\includegraphics[width=.25\columnwidth]{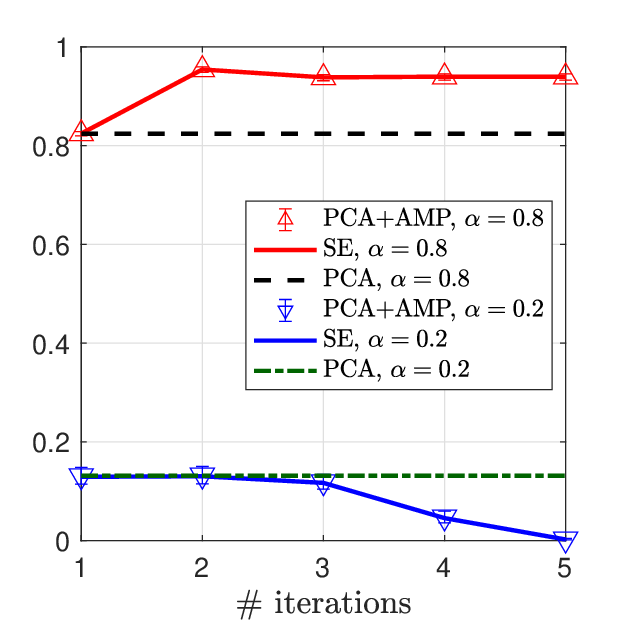}}
    \subfloat[Rectangular, right.\label{fig:AMPSErr}]{\includegraphics[width=.25\columnwidth]{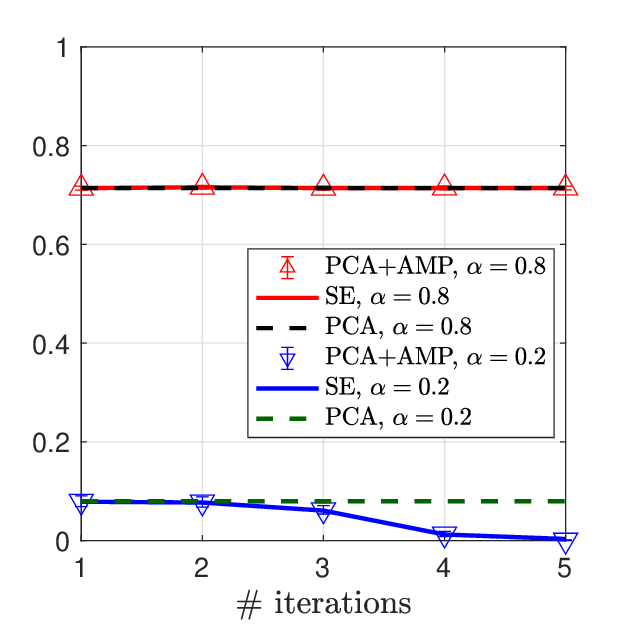}}
\caption{\small Comparison between AMP with PCA initialization and the related state evolution (SE). The plots show the normalized squared correlation between iterate and signal, as a function of the number of iterations.}
\label{fig:AMPSE}
\end{figure}
\begin{figure}[t]
    \centering
    \subfloat[Square, MP.\label{fig:AMPspectsq}]{\includegraphics[width=.25\columnwidth]{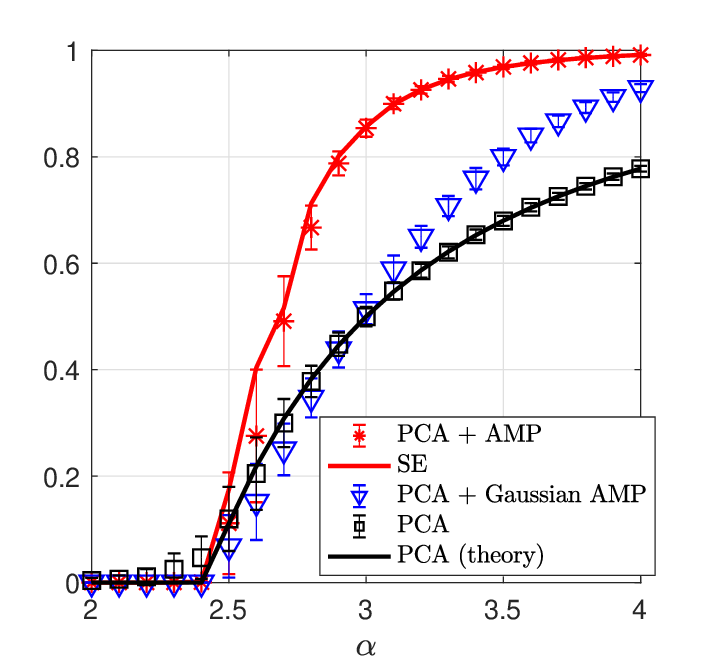}}
    \subfloat[Square, uniform.\label{fig:AMPspectsq2}]{\includegraphics[width=.25\columnwidth]{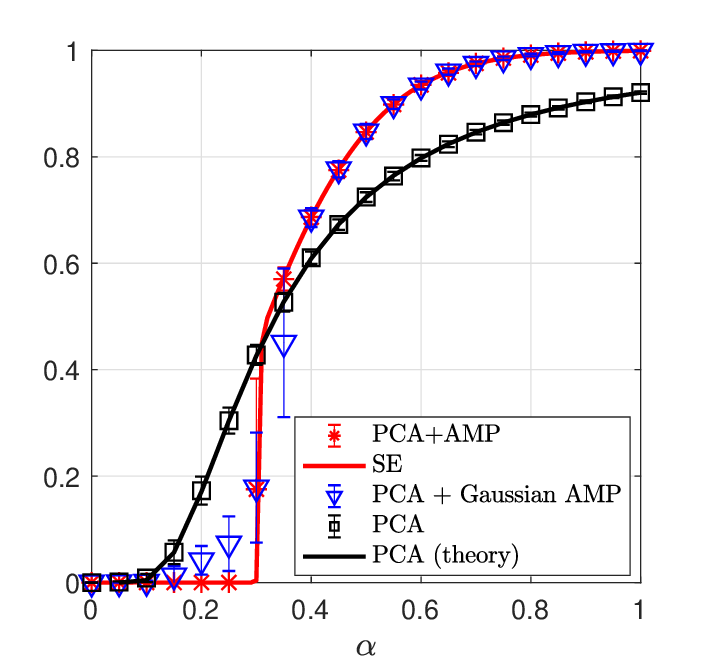}}
    \subfloat[Rectangular, left.\label{fig:AMPspectrl}]{\includegraphics[width=.25\columnwidth]{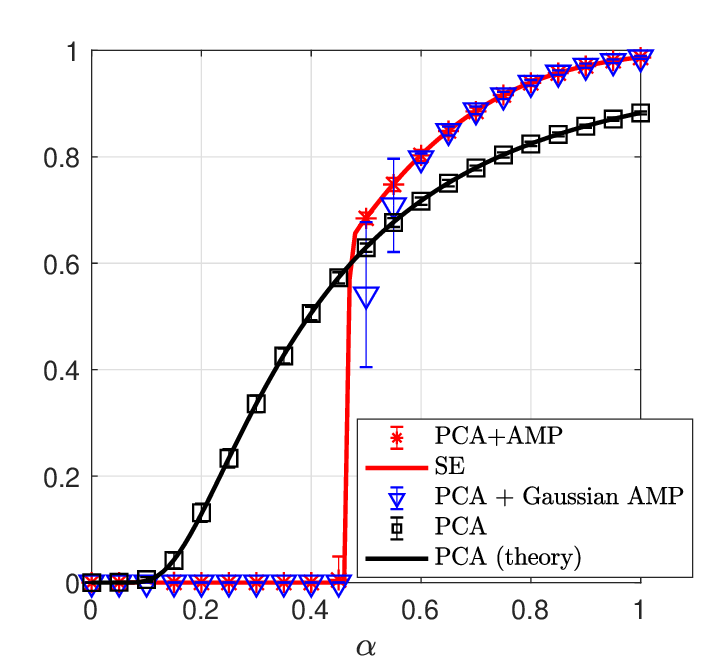}}
    \subfloat[Rectangular, right.\label{fig:AMPspectrr}]{\includegraphics[width=.25\columnwidth]{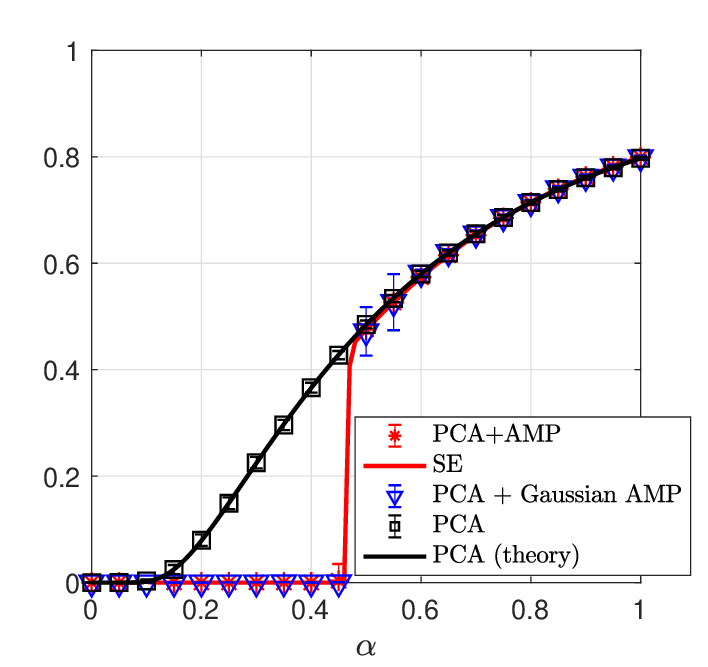}}
\caption{\small Comparison between AMP with PCA initialization and the PCA method alone. The plots show the normalized squared correlation between the signal and the estimate (PCA, or AMP+PCA), as a function of $\alpha$. }
\label{fig:AMPspect}
\end{figure}

Figure \ref{fig:AMPSE} compares the performance between the proposed AMP algorithm with PCA initialization (PCA+AMP) and the theoretical predictions of state evolution (SE), for two different values of $\alpha$. On the $x$-axis, we have the number of iterations of AMP, and on the $y$-axis the normalized squared correlation between the iterate and the signal. As a reference, we also plot the performance of PCA as a horizontal line.  We observe an excellent agreement of AMP with state evolution, even in the settings (b)-(c)-(d) where the free cumulants (resp. rectangular free cumulants) are alternating in sign. This supports our conjecture that Theorems \ref{thm:square}-\ref{thm:rect} hold for more general noise distributions.

In Figure \ref{fig:AMPspect}, we run PCA+AMP until the algorithm converges, and we compare the results with \emph{(i)} the AMP with PCA initialization developed in \cite{montanari2017estimation} which assumes that the noise matrix is Gaussian (with the correct variance), and \emph{(ii)} the PCA method alone, as a function of the SNR $\alpha$.
For Marchenko-Pastur noise (setting (a)), PCA+AMP always improves upon the PCA initialization. However, this is not the case when the eigenvalues/singular values of the noise are uniformly distributed (settings (b), (c) and (d)). In fact, we observe a phase transition phenomenon: below a certain critical $\alpha$, AMP converges to a trivial fixed point at $0$, while PCA shows positive correlation with the signal; above the critical $\alpha$, PCA+AMP is no worse than PCA. This is due to the sub-optimal choice of $\su_t$; recall the discussion on p.\pageref{par:choice_ut}.  We observe no improvement for the estimation of the right singular vector (setting (d)), as the prior of $\bv^*$ is Gaussian, in which case we expect the PCA estimate to be optimal. The interesting behavior demonstrated in Figure \ref{fig:AMPspect} motivates the study of the optimal choice for $\su_t, \sv_t$ in future work. We also note that, in settings (c)-(d), $\tilde{\alpha}_{\rm s}=0$, which means that the PCA estimator has non-zero correlation with the signal for all $\alpha>0$. However, for $\alpha<0.1$, this correlation remains rather small. Finally, we highlight that our proposed rotationally invariant PCA+AMP always improves upon the Gaussian PCA+AMP. In general, this performance gap will be significant unless the sequence of free cumulants $\kappa_k^\infty$ ($\kappa_{2k}^\infty$ in the rectangular case) decays quickly. For Marchenko-Pastur noise, the free cumulants are all equal, and thus the performance gap is significant. If the eigenvalues/singular values of the noise are uniform, then the sequence of free cumulants decays rapidly and the performance gap is small. 


\section{Proof Sketch: Symmetric Square Matrices} \label{sec:proof_sketch}


We consider the following artificial AMP algorithm, whose iterates are denoted by $\tbu^{t}, \tbf^{t}$ for $t \ge 1$. We initialize with $\tilde{\bu}^1 = \rho_\alpha \bu^* + \sqrt{1-\rho_\alpha^2} \, \bn$ and $\tbf^1 = \bX \tbu^1 - \kappa_1 \tbu^1$. Here, $\bn$ is standard Gaussian  and $\rho_\alpha$ is the normalized (limit) correlation of the PCA estimate given in \eqref{eq:rho_alpha_def}. We note that this initialization is impractical, as it requires the knowledge of the unknown signal $\bu^*$. However, this is not an issue since the artificial AMP serves only as a proof technique. (The \emph{true AMP} \eqref{eq:AMP} used for estimation uses the PCA initialization in \eqref{eq:AMPinit}.)  The subsequent iterates of the artificial AMP are defined in two phases. In the first phase, which lasts up to iteration $(T+1)$, the functions defining the artificial AMP are chosen so that  $\tilde{\bu}^{T+1}$ is closely aligned with the eigenvector $\bu_{\PCA}$ as $T \to \infty$. In the second phase, the functions are chosen to match those in the true AMP.  

The artificial AMP initialization $\tilde{\bu}^1$ is chosen such that it has non-zero asymptotic correlation with the signal $\bu^*$.  Indeed, when the signal prior has zero mean, a random initialization (independent of  $\bu^*$) would be asymptotically uncorrelated with the signal; consequently, the first phase of the artificial AMP  would get stuck at a trivial fixed point and the iterates would not be guaranteed to converge to the principal eigenvector. We ensure that this does not happen by defining the initialization $\tilde{\bu}^1$ to be a linear combination of the signal and Gaussian noise.

\paragraph{First phase.} 
For $2\le t\le (T+1)$, the artificial AMP iterates are
\begin{equation}\label{eq:AMPfake1}
\begin{split}
   \tbu^{t} = \tilde{\bdff}^{t-1}/\alpha, 
   \qquad \tbf^t = \bX \tilde{\bu}^t - \sum_{i=1}^t \tilde{\sb}_{t,i}\tilde{\bu}^i,
   \end{split}
\end{equation}
where $\tilde{\sb}_{t,t-j}=\kappa_{j+1}\alpha^{-j}$, for $(t-j)\in [1, t]$. 
We claim that, for sufficiently large $T$, $\tilde{\bu}^{T+1}$ approaches the PCA estimate $\bu_{\rm PCA}$, that is,
$ \lim_{T\to\infty}\lim_{n\to\infty} \frac{1}{\sqrt{n}} \|\tilde{\bu}^{T+1}-\sqrt{n}\bu_{\rm PCA}\| \, = \, 0.  $
This result is proved in Lemma \ref{lem:Phase1_PCA_conv} in Appendix \ref{app:confirstest}. We give a heuristic sanity check here. Assume that the iterate $\tilde{\bu}^{T+1}$ converges to a limit $\tilde{\bu}^\infty$ in the sense that 
$ \lim_{T \to \infty} \lim_{n\to \infty}\frac{1}{\sqrt{n}}\|\tilde{\bu}^{T+1} - \tilde{\bu}^\infty \| =0$. Then, from \eqref{eq:AMPfake1}, the limit  $\tilde{\bu}^\infty$ satisfies
\begin{equation}
    \tilde{\bu}^\infty = \frac{1}{\alpha} \bX \tilde{\bu}^\infty - \sum_{i=1}^\infty\kappa_i \left(\frac{1}{\alpha}\right)^i\tilde{\bu}^\infty \Longleftrightarrow \Bigg(\alpha +\sum_{i=1}^\infty\kappa_i \left(\frac{1}{\alpha}\right)^{i-1}\Bigg)\tilde{\bu}^\infty = \bX \tilde{\bu}^\infty,
\end{equation}
which means that $\tilde{\bu}^\infty$ is an eigenvector of $\bX$. Furthermore, by using known identities in free probability (see \eqref{eq:R} and \eqref{eq:RG}), the eigenvalue $\alpha +\sum_{i=1}^\infty\kappa_i \left(\frac{1}{\alpha}\right)^{i-1}$ can be re-written as $G^{-1}(1/\alpha)$. Recall that, for $\alpha>\alpha_{\rm s}$, $\bX$ exhibits a spectral gap and its largest eigenvalue converges to $G^{-1}(1/\alpha)$. Thus, $\bu^\infty$ must be aligned with the principal eigenvector of $\bX$, as desired. 

A key step in our analysis is to show that, as $T\to\infty$, the state evolution of the artificial AMP in the first phase has the unique fixed point $(\tilde{\mu}=\alpha\rho_\alpha, \tilde{\sigma}=\alpha^2(1-\rho_\alpha^2))$. This is established in Lemma \ref{lem:SE_FP_phase1} proved in Appendix \ref{app:fpsefirst}. The proof follows the approach developed in Section 7 of \cite{fan2020approximate}. However, the analysis of \cite{fan2020approximate} requires that $\alpha$ is sufficiently large, while our result holds for all $\alpha>\alpha_{\rm s}$. Our idea is to exploit the expression of the limit correlation between the PCA estimate and the signal. In particular, we prove that, when the PCA estimate is correlated with the signal, state evolution is close to a limit map which is a contraction. The price to pay for this approach is the requirement that the free cumulants are non-negative. 

\paragraph{Second phase.}

The second phase of the artificial AMP is designed so that its iterates $(\tbu^{T+k}, \tbf^{T+k})$ are close to $(\bu^{k}, \bdff^{k})$, for $k \ge 2$.
For $ t \ge (T+2)$, the  artificial AMP computes:
\begin{equation}\label{eq:AMPfake2}
\tbu^{t} = \su_{t-T}(\tbf^{t-1}), \qquad \tbf^t = \bX \tbu^t - \sum_{i=1}^t\tsb_{t,i} \tbu^i.
\end{equation}
Here, the functions $\{ u_k \}_{k \ge 2}$, are the ones used in the true AMP \eqref{eq:AMP}. The coefficients $\{\tsb_{t,i}\}$ for $t \ge (T+2)$ are given by:
\begin{align}
    \tsb_{tt} = \kappa_1, \quad \tsb_{t,t-j} = \kappa_{j+1}   \left(\frac{1}{\alpha} \right)^{(T+1-(t-j))_+} \hspace{-25pt}\prod_{i= \max\{t-j+1, T+2 \}}^t \hspace{-5pt} \< \su'_{i-T}(\tbf^{i-1}) \>, \quad   (t-j)\in [1, t-1].
    \label{eq:Phase2_onsager}
\end{align}

Since the artificial AMP is initialized with $\tbu^1$ that is correlated with $\bu^*$ and independent of the noise matrix $\bW$, a state evolution result for it can be obtained directly from \cite[Theorem 1.1]{fan2020approximate}. We then show in Lemma \ref{lem:sec_phase_sq} in Appendix \ref{sec:app_sec_phase_analysis} that  the second phase iterates in \eqref{eq:AMPfake2} are close to the true AMP iterates in \eqref{eq:AMP}, and that their state evolution parameters are also close. This result yields Theorem \ref{thm:square}, as shown in Appendix \ref{subsec:app_thm_prof_sq}.  The complete proof of Theorem \ref{thm:rect} (rectangular case) is given in Appendix \ref{app:pfrect}. We describe the artificial AMP for this case along with a proof sketch in Appendix \ref{subsec:proof_sketch_rect}.

\section*{Acknowledgements}

M. Mondelli would like to thank L\'aszl\'o Erd\"os for helpful discussions. M. Mondelli was partially supported by the 2019 Lopez-Loreta Prize. R. Venkataramanan was partially supported by the Alan Turing Institute under the EPSRC grant EP/N510129/1.

{\small{
\bibliographystyle{amsalpha}
\bibliography{all-bibliography}
\addcontentsline{toc}{section}{References}
}}

\newpage

\appendix

\section{Free Probability Background}\label{sec:fptools}


\subsection{Symmetric Square Matrices}

Let $X$ be a random variable of finite moments of all orders, and denote its moments by $m_k =\mathbb E\{X^k\}$. In this paper, $X$ represents either the empirical eigenvalue distribution of the noise matrix $\bW\in\mathbb R^{n\times n}$, or its limit law $\Lambda$ (in the latter case, the moments and free cumulants are denoted by $\{m_k^\infty\}_{k\ge 1}$ and $\{\kappa_k^\infty\}_{k\ge 1}$, respectively). For the model \eqref{eq:defsquare}, note that the empirical eigenvalue distribution of $\bW$ coincides with the empirical eigenvalue distribution of $\bX$ after excluding the largest eigenvalue of $\bX$, since we consider the case $\alpha>\alpha_{\rm s}$. The free cumulants $\{\kappa_k\}_{k\ge 1}$ of $X$ are defined recursively by the moment-cumulant relations
\begin{equation}\label{eq:mcrel1}
    m_k = \sum_{\pi\in {\rm NC}(k)}\prod_{S\in \pi}\kappa_{|S|},
\end{equation}
where ${\rm NC}(k)$ is the set of all non-crossing partitions of $\{1, \ldots, k\}$, and $|S|$ denotes the cardinality of $S$. Furthermore, by exploiting the connection between the formal power series with coefficients $\{m_k\}_{k\ge 1}$ and $\{\kappa_k\}_{k\ge 1}$, each free cumulant $\kappa_k$ can be computed from $m_1, \ldots, m_k$ and $\kappa_1, \ldots, \kappa_{k-1}$ as \cite[Section 2.5]{novak2014three}
\begin{equation}\label{eq:mcrel2}
    \kappa_k = m_k - [z^k]\sum_{j=1}^{k-1}\kappa_j\left(z+m_1z^2 + m_2z^3 + \cdots + m_{k-1}z^k\right)^j,
\end{equation}
where $[z^k](q(z))$ denotes the coefficient of $z^k$ in the polynomial $q(z)$.

Consider now the random variable $\Lambda$ representing the limiting spectral distribution of $\bW$, and recall that $b<\infty$ denotes the supremum of the support of $\Lambda$. Then, for $z>b$, the Cauchy transform $G(z)$ of $\Lambda$ is given by
\begin{equation}\label{eq:Gtdef}
    G(z) = \mathbb E\left\{\frac{1}{z-\Lambda}\right\}.
\end{equation}
Another transform that will be useful in our analysis is the $R$-transform $R(z)$ of $\Lambda$, which can be defined by the convergent series:
\begin{equation}\label{eq:R}
    R(z) = \sum_{i=0}^\infty \kappa_{i+1}^\infty z^{i}, 
\end{equation}
where $\{\kappa^\infty_k\}_{k\ge 1}$ are the free cumulants of $\Lambda$. 
The derivative of the $R$-transform is denoted by $R'(z)$ and given by
\begin{equation}\label{eq:R1}
    R'(z) = \sum_{i=0}^\infty (i+1)\kappa_{i+2}^\infty z^{i}= \sum_{j=0}^\infty\sum_{k=0}^\infty\kappa_{j+k+2}z^{j+k},
\end{equation}
where the second equality follows from a double-counting argument. The series in \eqref{eq:R} and \eqref{eq:R1} are well-defined and converge to a finite value for $z<1/\alpha_{\rm s}$, where $\alpha_{\rm s}=1/G(b^+)$ is the spectral threshold \cite{benaych2011eigenvalues}. The $R$-transform can also be expressed in terms of the Cauchy transform, see e.g. Theorem 12.7 of \cite{nica2006lectures}:
\begin{equation}\label{eq:RG}
    R(z) = G^{-1}(z)-\frac{1}{z}.
\end{equation}
By taking the derivative on both sides of \eqref{eq:R1}, we have
\begin{equation}\label{eq:R1G}
    R'(z) = \frac{1}{G'(G^{-1}(z))} + \frac{1}{z^2}.
\end{equation}

If $\bW$ follows a Marcenko-Pastur distribution (i.e., $\bW = \frac{1}{n}\bG_n \bG_n^\sT\in\mathbb R^{n\times n}$, where the entries of $\bG_n\in\mathbb R^{n\times p}$ are i.i.d. standard Gaussian), then it is well known that $\kappa_k^\infty=c\triangleq p/n$ for $k\ge 1$, see e.g. \cite[Chap. 2, Exercise 11]{mingo2017free}. This corresponds to the setting (a) in the numerical results of Section \ref{sec:simu}. If the eigenvalues of $\bW$ are i.i.d. and uniformly distributed in the interval $[-1/2, 1/2]$, the free cumulants $\kappa_k^\infty$ have also a simple form. In fact, by explicitly computing the expectation in \eqref{eq:Gtdef}, we have that
\begin{equation}
    G(z) = \log\frac{2z+1}{2z-1}.
\end{equation}
 Thus, by applying \eqref{eq:RG}, we deduce that
 \begin{equation}
     R(z) = \frac{1}{2}\coth\left(\frac{z}{2}\right)-\frac{1}{z}.
 \end{equation}
 By comparing the series expansion \eqref{eq:R} with that of the hyperbolic cotangent, we conclude that 
 \begin{equation}\label{eq:cumunif}
     \kappa_k=\begin{cases}
     0, \quad\quad\mbox{ if $k$ is odd,}\vspace{.5em}\\
     \displaystyle\frac{B_{k}}{k!},\,\, \quad\mbox{if $k$ is even,}\\
     \end{cases}
 \end{equation}
 where $B_k$ denotes the $k$-th Bernoulli number.
  This corresponds to the setting (b) in the numerical results of Section \ref{sec:simu}.

 \subsection{Rectangular Matrices}

Let $X$ be a random variable of finite moments of all orders, and denote its even moments by $m_{2k} =\mathbb E\{X^{2k}\}$. In this paper, $X^2$ represents either the empirical eigenvalue distribution of $\bW\bW^\sT\in\mathbb R^{m\times m}$, or its limit law $\Lambda^2$ (in the latter case, the moments and rectangular free cumulants are denoted by $\{m_{2k}^\infty\}_{k\ge 1}$ and $\{\kappa_{2k}^\infty\}_{k\ge 1}$, respectively). 
For the model \eqref{eq:defrect}, note that the empirical eigenvalue distribution of $\bW\bW^\sT$ coincides with the empirical eigenvalue distribution of $\bX\bX^\sT$ after excluding the largest eigenvalue of $\bX\bX^\sT$, since we consider the case $\tilde{\alpha}>\tilde{\alpha}_{\rm s}$.
The rectangular free cumulants $\{\kappa_{2k}\}_{k\ge 1}$ of $X$ are defined recursively by the moment-cumulant relations \cite[Section 3]{benaych2009rectangular}
\begin{equation}\label{eq:mcrel1rect}
    m_{2k} =\gamma \sum_{\pi\in {\rm NC}'(2k)}\prod_{\substack{S\in \pi\\\min S \,\,\,\mbox{\scriptsize is odd}}}\kappa_{|S|}\prod_{\substack{S\in \pi\\\min S \,\,\,\mbox{\scriptsize is even}}}\kappa_{|S|},
\end{equation}
where ${\rm NC}'(2k)$ is the set of non-crossing partitions $\pi$ of $\{1, \ldots, 2k\}$ such that each set $S\in \pi$ has even cardinality. Furthermore, by exploiting the connection between the formal power series with coefficients $\{m_{2k}\}_{k\ge 1}$ and $\{\kappa_{2k}\}_{k\ge 1}$, each rectangular free cumulant $\kappa_{2k}$ can be computed from $m_2, \ldots, m_{2k}$ and $\kappa_2, \ldots, \kappa_{2(k-1)}$ as \cite[Lemma 3.4]{benaych2009rectangular}
\begin{equation}\label{eq:mcrel2rect}
    \kappa_{2k} = m_{2k} - [z^k]\sum_{j=1}^{k-1}\kappa_{2j}\left(z(\gamma M(z)+1)(M(z)+1)\right)^j,
\end{equation}
where $M(z)=\sum_{k=1}^\infty m_{2k}z^k$ and $[z^k](q(z))$ denotes again the coefficient of $z^k$ in the polynomial $q(z)$.

Consider now the random variable $\Lambda$ representing the limiting distribution of the singular values of $\bW$, and recall that $b<\infty$ denotes the supremum of the support of $\Lambda$. Then, for $z>b$, the $D$-transform $D(z)$ of $\Lambda$ is given by 
\begin{equation}
 D(z)=\phi(z)\cdot \bar{\phi}(z) ,  
\end{equation}
where
\begin{equation}
    \phi(z)=\mathbb E\left\{\frac{z}{z^2-\Lambda^2}\right\},\quad \bar{\phi}(z)=\gamma\phi(z)+\frac{1-\gamma}{z}.
\end{equation}
Another transform that will be useful in our analysis is the rectangular $R$-transform $R(z)$ of $\Lambda$, which can be defined by the convergent series:
\begin{equation}\label{eq:Rrect}
    R(z) = \sum_{i=1}^\infty \kappa_{2i}^\infty z^{i}, 
\end{equation}
where $\{\kappa^\infty_{2k}\}_{k\ge 1}$ are the rectangular free cumulants of $\Lambda$. The derivative of the rectangular $R$-transform is denoted by $R'(z)$ and given by  
\begin{equation}\label{eq:R1rect}
    R'(z) = \sum_{i=0}^\infty (i+1)\kappa_{2(i+1)}^\infty z^{i}= \sum_{j=0}^\infty\sum_{k=0}^\infty\kappa_{2(j+k+1)}z^{j+k},
\end{equation}
where the second equality follows from a double-counting argument. By combining \eqref{eq:Rrect} and \eqref{eq:R1rect}, we also obtain the useful identities
\begin{align}
    & \sum_{j=0}^\infty \sum_{k=0}^\infty \kappa_{2(j+k+2)} z^{j+k+2} = z R'(z) - R(z), \label{eq:R20rect} \\
    & \sum_{i=0}^\infty (i+1)\kappa_{2(i+2)}^\infty z^{i}=z^{-1}R'(z)-z^{-2}R(z). \label{eq:R2rect}
\end{align}
The series in \eqref{eq:Rrect}-\eqref{eq:R2rect} are well-defined and converge to a finite value for $z<1/(\tilde{\alpha}_{\rm s})^2$, where $\tilde{\alpha}_{\rm s}=1/\sqrt{D(b^+)}$ is the spectral threshold \cite{benaych2012singular}. The rectangular $R$-transform can also be expressed in terms of the $D$-transform, see e.g. \cite[Section 2.5]{benaych2012singular}:
\begin{equation}\label{eq:RD}
    \gamma R^2(z)+(\gamma+1)R(z)+1 = z(D^{-1}(z))^2.
\end{equation}

\section{Proof of Theorem \ref{thm:square}}\label{app:pfsq}

This appendix is organized as follows. In Appendix \ref{app:sesq}, we present the state evolution recursion associated to the artificial AMP iteration defined in \eqref{eq:AMPfake1} and \eqref{eq:AMPfake2}. In Appendix \ref{app:fpsefirst}, we prove that the first phase of this state evolution admits a unique fixed point. Using this fact, in Appendix \ref{app:confirstest}, we prove that the artificial AMP iterate at the end of the first phase approaches the PCA estimator. Then, in Appendix  \ref{sec:app_sec_phase_analysis}, we show that \emph{(i)} the iterates in the second phase of the artificial AMP are close to the true AMP iterates, and \emph{(ii)} the related state evolution parameters also remain close. Finally, in Appendix \ref{subsec:app_thm_prof_sq}, we give the proof of Theorem \ref{thm:square}.

\subsection{State Evolution for the Artificial AMP}\label{app:sesq}

Consider the artificial AMP iteration defined in \eqref{eq:AMPfake1} and \eqref{eq:AMPfake2}, with initialization 
\beq 
\tilde{\bu}^1 = \rho_\alpha \bu^* + \sqrt{1-\rho_\alpha^2}\bn, \qquad \tbf^1 = \bX \tbu^1 - \kappa_1 \tbu^1.
\label{eq:art_sq_init}
\eeq
Then, its associated state evolution recursion is expressed in terms of a sequence of mean vectors $\tilde{\bmu}_{K}=(\tilde{\mu}_t)_{t\in [0, K]}$ and covariance matrices $\tilde{\bSigma}_{K}=(\tilde{\sigma}_{s, t})_{s, t\in [0, K]}$ defined recursively as follows. We initialize with
\begin{equation}\label{eq:SEfakeinit}
   \tilde{\mu}_0= \alpha \rho_{\alpha}, \qquad  \tilde{\sigma}_{0, 0}=\alpha^2(1-\rho_\alpha^2), \quad \tilde{\sigma}_{0, t}= \tilde{\sigma}_{t, 0}=0, \quad \mbox{ for } t \ge 1. 
\end{equation}

 Given $\tilde{\bmu}_K$ and $\tilde{\bSigma}_K$,  let 
\begin{align}
   &  (\tilde{F_0}, \ldots, \tilde{F}_K) = \tilde{\bmu}_K U_* + (\tilde{Z}_0, \ldots, \tilde{Z}_K), \quad \text{ where } \  (\tilde{Z}_0, \ldots, \tilde{Z}_K) \sim \normal(\bzero, \tilde{\bSigma}_K),  \quad \text{ and  } \nonumber \\ 
  &  \tilde{U}_t = \tsu_t(\tF_{t-1}), \quad 
 \text{ where } \ 
  \tsu_t(x) =
  \begin{cases} 
  x/\alpha, &  1 \le t \le T+1, \\
  \su_{t-T}(x),&   t \ge T+2. 
  \end{cases}
  \label{eq:tilU_tilF_def}
\end{align}
Then,  the entries of $\tilde{\bmu}_{K+1}$ are given by $\tmu_t = \alpha \E\{ \tU_t U_* \}$ (for $t \in [1, K+1]$), and the entries of $\tilde{\bSigma}_{K+1}$ (for $s,t \in[1, K+1]$) are given by
\begin{align}
    \tsigma_{s,t} & = \sum_{j=0}^{s-1} \sum_{k=0}^{t-1}
    \kappa_{j+k+2}^{\infty}  
    \left( \prod_{i=s-j+1}^s 
    \E\{ \tsu'_i (\tF_{i-1}) \} \right)  
    \left( \prod_{i=t-k+1}^t \E\{ \tsu'_i (\tF_{i-1}) \} \right) \E\{ \tU_{s-j} \tU_{t-k} \}.
    \label{eq:SE_artificial_sigma}
\end{align}

\begin{proposition}[State evolution for artificial AMP -- symmetric square matrices]
\label{prop:tilSE}
Consider the setting of Theorem \ref{thm:square}, the artificial AMP iteration described in \eqref{eq:AMPfake1} and \eqref{eq:AMPfake2} with the initialization given in \eqref{eq:art_sq_init}, and the corresponding state evolution parameters defined in \eqref{eq:SEfakeinit}-\eqref{eq:SE_artificial_sigma}. Then,
for $t \geq 1$ and any  \PL($2$) function $\psi:\reals^{2t+2} \to\reals$, the following holds almost surely:
\begin{align}
\lim_{n \to \infty} \frac{1}{n} \sum_{i=1}^n 
\psi (u^*_{i}, \tu^1_i, \ldots, \tu^{t+1}_i, \tilde{f}^1_i, \ldots \tilde{f}^{t}_i)
 = \E \left\{ \psi(U_*, \tU_1, \ldots, \tU_{t+1}, \tilde{F}_1, \ldots, \tilde{F}_t) \right\}.
\end{align}
\end{proposition}

The proposition follows directly from Theorem 1.1 in \cite{fan2020approximate} since the initialization $\tilde{\bu}^1$ of the artificial AMP is independent of $\bW$.

\subsection{Fixed Point of State Evolution for the First Phase}\label{app:fpsefirst}

From \eqref{eq:SEfakeinit}-\eqref{eq:SE_artificial_sigma}, we note that the state evolution recursion for the first phase $(t \in [1, T+1])$ has the following form:
\begin{equation}\label{eq:SEfake1}
\begin{split}
    \tilde{\mu}_t &= \alpha \rho_\alpha, \quad \mbox{ for }t\in [1, T+1],\\
    \tilde{\sigma}_{s, t} &= \sum_{j=0}^{s-1} \sum_{k=0}^{t-1} \kappa_{j+k+2}^\infty \left(\frac{1}{\alpha}\right)^{j+k+2}\left((\alpha\rho_\alpha)^2+\tilde{\sigma}_{s-j-1, t-k-1}\right), \quad \mbox{ for }s, t\in [1, T+1].
\end{split}
\end{equation}
In this section, we prove the following result concerning the fixed point of the recursion \eqref{eq:SEfake1}.

\begin{lemma}[Fixed point of state evolution for first phase -- Square matrices] Consider 
the state evolution recursion for the first phase given by \eqref{eq:SEfake1}, initialized according to \eqref{eq:SEfakeinit}. Assume that $\kappa_i^\infty\ge 0$ for all $i\ge 2$, and that $\alpha>\alpha_{\rm s}$. Pick any $\xi<1$ such that $\alpha\xi>\alpha_{\rm s}$. Then, \begin{equation}
    \lim_{T\to\infty}\max_{s, t\in [0, T]}\xi^{\max(s, t)}|\tilde{\sigma}_{T+1-s, T+1-t}-\alpha^2(1-\rho_\alpha^2)|=0.
    \label{eq:SE_FP_phase1}
\end{equation}
\label{lem:SE_FP_phase1}
\end{lemma}

To prove the claim, we consider the space of infinite matrices $\bx=(x_{s, t} : s, t\le 0)$ indexed by the non-positive integers and equipped with the weighted $\ell_\infty$-norm:
\begin{equation}\label{eq:norm}
    \|\bx\|_\xi = \sup_{s, t\le 0}\xi^{\max(|s|, |t|)}|x_{s, t}|.
\end{equation}
We define $\mathcal X=\{\bx : \|\bx\|_\xi<\infty\}$, and note that $\mathcal X$ is complete under $\|\cdot\|_\xi$. For any compact set $I\subset \mathbb R$, we also define 
\begin{equation}\label{eq:defXI}
    \mathcal X_I = \{\bx : x_{s, t}\in I \mbox{ for all }s, t\le 0\}\subset \mathcal X.
\end{equation}
Then, $\mathcal X_I$ is closed in $\mathcal X$ and therefore it is also complete under $\|\cdot\|_\xi$. We embed the matrix $\btSigma_{\bar{T}}$ as an element $\bx\in\mathcal X$ with the following coordinate identification:
\begin{equation*}
    \begin{split}
        \tilde{\sigma}_{s, t}&=x_{s-\bar{T}, t-\bar{T}},\\
        x_{s, t}&=0, \quad \mbox{ if }s<-\bar{T} \mbox{ or }t<-\bar{T}.
    \end{split}
\end{equation*}
The idea is to approximate the map $\btSigma_{\bar{T}-1}\mapsto \btSigma_{\bar{T}}$ with the \emph{limit} map $h^{\Sigma}$ defined as
\begin{equation}\label{eq:fixedmap}
    h_{s, t}^{\Sigma}(\bx) = \sum_{j=0}^\infty \sum_{k=0}^\infty \kappa_{j+k+2}^\infty\left(\frac{1}{\alpha}\right)^{j+k+2}\left((\alpha\rho_\alpha)^2 + x_{s-j, t-k}\right).
\end{equation}
The map $h^{\Sigma}$ has a similar structure to the embedding of the map $\btSigma_{\bar{T}-1}\mapsto \btSigma_{\bar{T}}$ into $\mathcal X$. However, comparing \eqref{eq:SEfake1} and \eqref{eq:fixedmap}, we highlight two important differences. First, the indices of $x_{s-j, t-k}$ are shifted with respect to the indices of $\tilde{\sigma}_{s-j-1, t-k-1}$. This difference is purely technical and it simplifies the proof of the subsequent Lemma \ref{lemma:app}, which shows that $h^{\Sigma}$ is close to the map $\btSigma_{\bar{T}-1}\mapsto \btSigma_{\bar{T}}$. Second, the map $h^\Sigma$ is \emph{fixed}, in the sense that it does not depend on $s, t$. In fact, note that the sums over $j$ and $k$ run from $0$ to $\infty$ in \eqref{eq:fixedmap}. This is in contrast with \eqref{eq:SEfake1} where the two sums run until $j=s-1$ and $k=t-1$. 

The approach of approximating the state evolution map with a fixed limit map was first developed in \cite{fan2020approximate}. The key difference is that, in \cite{fan2020approximate}, it is assumed that $\alpha$ is sufficiently large, which allows to simplify the analysis. On the contrary, our result holds for all $\alpha>\alpha_{\rm s}$, $\alpha_{\rm s}$ being the spectral threshold for PCA. This is because of two main reasons. First, the expressions for the state evolution recursion are simplified by considering linear denoisers in the first phase of the artificial AMP. Second, we crucially exploit the form (and the strict positivity) of the correlation between the signal and the PCA estimate, in order to prove that the limit map \eqref{eq:fixedmap} is a contraction (cf. \eqref{eq:Lip1} in Lemma \ref{lemma:cont}). 

First, we show that $h^\Sigma(\mathcal X_{I^*})\subseteq \mathcal X_{I^*}$ for a suitably defined compact set $I^*$.

\begin{lemma}[Image of limit map -- Square matrices]\label{lemma:image}
    Consider the map $h^{\Sigma}$ defined in \eqref{eq:fixedmap}. Assume that $\kappa_i^\infty\ge 0$ for all $i\ge 2$, and that $\alpha>\alpha_{\rm s}$. Then, there exists $I^*=[-a^*, a^*]$ such that, if $\bx\in \mathcal X_{I^*}$, then $h^\Sigma(\bx)\in \mathcal X_{I^*}$.
\end{lemma}

\begin{proof}
Let $\bx\in \mathcal X_{I^*}$. Then, the following chain of inequalities holds:
\begin{equation*}
    \begin{split}
        |h_{s, t}^\Sigma(\bx)| &\stackrel{\mathclap{\mbox{\footnotesize (a)}}}{=} \left| \rho_\alpha^2 R'\left(\frac{1}{\alpha}\right)+ \sum_{j=0}^\infty \sum_{k=0}^\infty \kappa_{j+k+2}^\infty\left(\frac{1}{\alpha}\right)^{j+k+2} x_{s-j, t-k}\right|\\
        &\stackrel{\mathclap{\mbox{\footnotesize (b)}}}{\le} \rho_\alpha^2 \left|R'\left(\frac{1}{\alpha}\right)\right|+\sum_{j=0}^\infty \sum_{k=0}^\infty \kappa_{j+k+2}^\infty\left(\frac{1}{\alpha}\right)^{j+k+2} |x_{s-j, t-k}|\\
        &\stackrel{\mathclap{\mbox{\footnotesize (c)}}}{\le} \rho_\alpha^2 \left|R'\left(\frac{1}{\alpha}\right)\right|+a^*R'\left(\frac{1}{\alpha}\right)\left(\frac{1}{\alpha}\right)^2.
    \end{split}
\end{equation*}
Here, (a) follows from \eqref{eq:fixedmap} and \eqref{eq:R1}; (b) follows from the hypothesis that $\kappa_i^\infty\ge 0$ for $i\ge 2$; and (c) uses again \eqref{eq:R1} and the fact that $\bx\in \mathcal X_{I^*}$. 

Now, recall from \eqref{eq:rho_alpha_def} that above the spectral threshold, namely, when $\alpha > \alpha_{\rm s}$, the PCA estimator $\bu_{\rm PCA}$ has strictly positive correlation with the signal $\bu^*$:
\begin{equation*}
    \frac{\langle\bu_{\rm PCA}, \bu^*\rangle^2}{n}\stackrel{\mathclap{\mbox{\footnotesize a.s.}}}{\longrightarrow} \rho_\alpha^2 = \frac{-1}{\alpha^2G'(G^{-1}(1/\alpha))},
\end{equation*}
which immediately implies that 
\begin{equation}\label{eq:ineqG1}
    \frac{1}{\alpha^2G'(G^{-1}\left(\frac{1}{\alpha}\right))}<0.
\end{equation}
Thus, by combining \eqref{eq:ineqG1} with \eqref{eq:R1G}, we deduce that
\begin{equation}\label{eq:Rtin}
    R'\left(\frac{1}{\alpha}\right)\left(\frac{1}{\alpha}\right)^2 < 1. 
\end{equation}
Hence, as $R'\left(\frac{1}{\alpha}\right)<\infty$, there exists an $a^*$ such that 
\begin{equation*}
    \rho_\alpha^2 \left|R'\left(\frac{1}{\alpha}\right)\right|+a^*R'\left(\frac{1}{\alpha}\right)\left(\frac{1}{\alpha}\right)^2\le a^*,
\end{equation*}
which implies the desired claim. 
\end{proof}

Next, we compute a fixed point of $h^{\Sigma}$.

\begin{lemma}[Fixed point of limit map -- Square matrices]\label{lemma:fixed}
    Consider the map $h^{\Sigma}$ defined in \eqref{eq:fixedmap}, and let $\bx^*=(x^*_{s, t} : s, t\le 0)$ with $x^*_{s, t}=\alpha^2(1-\rho_\alpha^2)$. Assume that $\alpha>\alpha_{\rm s}$. Then, $\bx^*$ is a fixed point of $h^{\Sigma}$.
\end{lemma}

\begin{proof}
Note that, for $x=1/\alpha$, the power series expansion \eqref{eq:R1} of $R'$ converges to a finite limit as $\alpha>\alpha_{\rm s}$. Hence, by using the definition \eqref{eq:fixedmap}, we have that 
\begin{equation*}
        h_{s, t}^\Sigma(\bx^*) =  R'\left(\frac{1}{\alpha}\right).
\end{equation*}
Then, the claim  follows from  \eqref{eq:R1G} and the definition $\rho_\alpha = \sqrt{\frac{-1}{\alpha^2G'(G^{-1}(1/\alpha))}}$, which together show that $R'\left(\frac{1}{\alpha}\right) = \alpha^2(1-\rho_\alpha^2)$. 
\end{proof}

Let $I^*$ be such that $h^{\Sigma}: \mathcal X_{I^*}\to \mathcal X_{I^*}$ (the existence of such a set $I^*$ is guaranteed by Lemma \ref{lemma:image}). Then, the next step is to show that $h^{\Sigma}: \mathcal X_{I^*}\to\mathcal X_{I^*}$ is a contraction. We remark that, by the Banach fixed point theorem, this result implies that the fixed point $\bx^*$ defined in Lemma \ref{lemma:fixed} is unique. 

\begin{lemma}[Limit map is a contraction]\label{lemma:cont}
    Consider the map $h^{\Sigma}: \mathcal X_{I^*}\to \mathcal X_{I^*}$ defined in \eqref{eq:fixedmap} and where $I^*$ is given by Lemma \ref{lemma:image}. Assume that $\kappa_i^\infty\ge 0$ for all $i\ge 2$, and let $\xi<1$ be such that $\alpha\xi>\alpha_{\rm s}$. Then, for any $\bx, \by \in \mathcal X_{I^*}$, 
    \begin{equation}\label{eq:Lip}
        \|h^\Sigma(\bx)-h^\Sigma(\by)\|_\xi \le R'\left(\frac{1}{\xi\alpha}\right)\left(\frac{1}{\xi\alpha}\right)^2 \|\bx-\by\|_\xi,
    \end{equation}
where
\begin{equation}\label{eq:Lip1}
    R'\left(\frac{1}{\xi\alpha}\right)\left(\frac{1}{\xi\alpha}\right)^2<1.
\end{equation}
\end{lemma}

\begin{proof}
First of all, for any $s,t \le 0$, we have that
\begin{equation}\label{eq:int1}
    \begin{split}
        |h_{s, t}^\Sigma(\bx)-h_{s, t}^\Sigma(\by)| &\stackrel{\mathclap{\mbox{\footnotesize (a)}}}{=} \left| \sum_{j=0}^\infty \sum_{k=0}^\infty \kappa_{j+k+2}^\infty\left(\frac{1}{\alpha}\right)^{j+k+2} \left(x_{s-j, t-k}-y_{s-j, t-k}\right)\right|\\
        &\stackrel{\mathclap{\mbox{\footnotesize (b)}}}{\le} \sum_{j=0}^\infty\sum_{k=0}^\infty \kappa_{j+k+2}^\infty\left(\frac{1}{\alpha}\right)^{j+k+2} |x_{s-j, t-k}-y_{s-j, t-k}|.
    \end{split}
\end{equation}
Here, (a) follows from \eqref{eq:fixedmap}, and (b) follows from the hypothesis that $\kappa_i^\infty\ge 0$ for $i\ge 2$. Furthermore, we have that
\begin{equation}\label{eq:int2}
    |x_{s-j, t-k}-y_{s-j, t-k}| \le \|\bx-\by\|_\xi \xi^{-\max(|s-j|, |t-k|)}.
\end{equation}
Thus, by using \eqref{eq:int1} and \eqref{eq:int2}, we obtain
\begin{equation}\label{eq:int3}
    \begin{split}
    \|h^\Sigma(\bx)-h^\Sigma(\by)\|_\xi & = \sup_{s, t\le 0}\xi^{\max(|s|, |t|)}|h_{s, t}^\Sigma(\bx)-h_{s, t}^\Sigma(\by)|\\
    &\le \sup_{s, t\le 0}\xi^{\max(|s|, |t|)}\|\bx-\by\|_\xi\sum_{j=0}^\infty\sum_{k=0}^\infty \kappa_{j+k+2}^\infty\left(\frac{1}{\alpha}\right)^{j+k+2}\xi^{-\max(|s-j|, |t-k|)}.
    \end{split}
\end{equation}
Note that, as $\xi<1$,
\begin{equation*}
    \xi^{-\max(|s-j|, |t-k|)}\le \xi^{-\max(|s|, |t|)-j-k-2},
\end{equation*}
which implies that the RHS of \eqref{eq:int3} is  bounded above by
\begin{equation}
    \|\bx-\by\|_\xi\sum_{j=0}^\infty\sum_{k=0}^\infty \kappa_{j+k+2}^\infty\left(\frac{1}{\xi\alpha}\right)^{j+k+2} = R'\left(\frac{1}{\xi\alpha}\right)\left(\frac{1}{\xi\alpha}\right)^2\|\bx-\by\|_\xi,
\end{equation}
where the equality follows from \eqref{eq:R1}. This shows that \eqref{eq:Lip} holds. The proof of \eqref{eq:Lip1} follows the same argument as  \eqref{eq:Rtin}, since  $\xi \alpha>\alpha_{\rm s}$.
\end{proof}

At this point, we show that the state evolution of $\btSigma_{\bar{T}}$ can be approximated via the fixed map $h^\Sigma$.

\begin{lemma}[Limit map approximates state evolution map -- Square matrices]\label{lemma:app}
Consider the map $h^{\Sigma}: \mathcal X_{I^*}\to \mathcal X_{I^*}$ defined in \eqref{eq:fixedmap}, where $I^*$ is given by Lemma \ref{lemma:image}. Assume that $\kappa_i^\infty\ge 0$ for all $i\ge 2$, and let $\xi<1$ be such that $\alpha\xi>\alpha_{\rm s}$. Then, for any $\bx\in \mathcal X_{I^*}$,
\begin{equation}
    \|\btSigma_{\bar{T}} - h^\Sigma(\bx)\|_\xi \le R'\left(\frac{1}{\xi\alpha}\right)\left(\frac{1}{\xi\alpha}\right)^2 \|\btSigma_{\bar{T}-1} - \bx\|_\xi + F(\bar{T}),
\end{equation}
where 
\begin{equation}
    \lim_{\bar{T}\to\infty}F(\bar{T})=0.
\end{equation}
\end{lemma}

\begin{proof} 
Throughout the proof, we consider $\btSigma_{\bar{T}}, \btSigma_{\bar{T}-1}$ as embedded in $\mathcal{X}$. First, we write
\begin{equation}
\begin{split}
    \|\btSigma_{\bar{T}} - h^\Sigma(\bx)\|_\xi &= \sup_{s, t\le 0} \xi^{\max(|s|, |t|)}|(\btSigma_{\bar{T}})_{s, t} - h^\Sigma_{s, t}(\bx)|\\
    &= \max\Bigg(\sup_{\substack{s, t\le 0\\\max(|s|, |t|)<\bar{T}}} \xi^{\max(|s|, |t|)}|(\btSigma_{\bar{T}})_{s, t} - h^\Sigma_{s, t}(\bx)|,\\
    &\hspace{5em}\sup_{\substack{s, t\le 0\\\max(|s|, |t|)\ge \bar{T}}} \xi^{\max(|s|, |t|)}|(\btSigma_{\bar{T}})_{s, t} - h^\Sigma_{s, t}(\bx)|\Bigg),
    \end{split}
\end{equation}
where  $(\btSigma_{\bar{T}})_{s, t}=\tilde{\sigma}_{s+\bar{T}, t+\bar{T}}$ if $s\ge -\bar{T}$ and $t\ge -\bar{T}$, and $(\btSigma_{\bar{T}})_{s, t}=0$ otherwise. 

Let us look at the case $\max(|s|, |t|)<\bar{T}$, and define $I_1 = \{(j, k): j\ge s+\bar{T} \mbox{ or }k\ge t+\bar{T}\}$. Then, 
\begin{equation}\label{eq:case1int1}
\begin{split}
    |(\btSigma_{\bar{T}})_{s, t} - h^\Sigma_{s, t}(\bx)| &= \bigg|\sum_{j=0}^{s+\bar{T}-1}\sum_{k=0}^{t+\bar{T}-1}\kappa_{j+k+2}^\infty\left(\frac{1}{\alpha}\right)^{j+k+2}\left(\alpha^2\rho_\alpha^2+\tilde{\sigma}_{s-j+\bar{T}-1, t-k+\bar{T}-1}\right)\\
    &\hspace{9em} - \sum_{j=0}^{\infty}\sum_{k=0}^{\infty}\kappa_{j+k+2}^\infty\left(\frac{1}{\alpha}\right)^{j+k+2}\left(\alpha^2\rho_\alpha^2+x_{s-j, t-k}\right)\bigg|\\
    &\le\left|\sum_{j=0}^{s+\bar{T}-1}\sum_{k=0}^{t+\bar{T}-1}\kappa_{j+k+2}^\infty\left(\frac{1}{\alpha}\right)^{j+k+2}\left(\tilde{\sigma}_{s-j+\bar{T}-1, t-k+\bar{T}-1}-x_{s-j, t-k}\right)\right|\\
    &\hspace{2em}+\left|\sum_{j, k\in I_1}\kappa_{j+k+2}^\infty\left(\frac{1}{\alpha}\right)^{j+k+2}\left(\alpha^2\rho_\alpha^2+x_{s-j, t-k}\right)\right|:= T_1 + T_2.
\end{split}
\end{equation}
The term $T_1$ can be upper bounded as follows:
\begin{equation}\label{eq:case1int2}
    \begin{split}
        T_1 &\stackrel{\mathclap{\mbox{\footnotesize (a)}}}{\le} \sum_{j=0}^{s+\bar{T}-1}\sum_{k=0}^{t+\bar{T}-1}\kappa_{j+k+2}^\infty\left(\frac{1}{\alpha}\right)^{j+k+2}\left|\tilde{\sigma}_{s-j+\bar{T}-1, t-k+\bar{T}-1}-x_{s-j, t-k}\right|\\
        &\le \sum_{j=0}^{s+\bar{T}-1}\sum_{k=0}^{t+\bar{T}-1}\kappa_{j+k+2}^\infty\left(\frac{1}{\alpha}\right)^{j+k+2} \|\btSigma_{\bar{T}-1} - \bx\|_\xi \xi^{-\max(|s-j|, |t-k|)}\\
        &\stackrel{\mathclap{\mbox{\footnotesize (b)}}}{\le} \sum_{j=0}^{s+\bar{T}-1}\sum_{k=0}^{t+\bar{T}-1}\kappa_{j+k+2}^\infty\left(\frac{1}{\xi\alpha}\right)^{j+k+2} \|\btSigma_{\bar{T}-1} - \bx\|_\xi \xi^{-\max(|s|, |t|)}\\
        &\stackrel{\mathclap{\mbox{\footnotesize (c)}}}{\le} \sum_{j=0}^{\infty}\sum_{k=0}^{\infty}\kappa_{j+k+2}^\infty\left(\frac{1}{\xi\alpha}\right)^{j+k+2} \|\btSigma_{\bar{T}-1} - \bx\|_\xi \xi^{-\max(|s|, |t|)}\\
        &\stackrel{\mathclap{\mbox{\footnotesize (d)}}}{=} R'\left(\frac{1}{\xi\alpha}\right)\left(\frac{1}{\xi\alpha}\right)^2 \|\btSigma_{\bar{T}-1} - \bx\|_\xi \xi^{-\max(|s|, |t|)}.
    \end{split}
\end{equation}
Here, (a) and (c) follows from the hypothesis that $\kappa_i^\infty\ge 0$ for $i\ge 2$; (b) uses that $\xi<1$; and (d) uses \eqref{eq:R1}. The term $T_2$ can be upper bounded as follows:
\begin{equation}\label{eq:case1int3}
    \begin{split}
        T_2&\le \left(\alpha^2\rho_\alpha^2+a^*\right)\sum_{j, k\in I_1}\kappa_{j+k+2}^\infty\left(\frac{1}{\alpha}\right)^{j+k+2}\\
        &\le \frac{\alpha^2\rho_\alpha^2+a^*}{\alpha^2}\sum_{i=-\max(|s|, |t|)+\bar{T}}^\infty \kappa_{i+2}^\infty (i+1)\left(\frac{1}{\alpha}\right)^{i},
    \end{split}
\end{equation}
where the first inequality uses that $\bx\in \mathcal X_{I^*}$ and the second inequality uses that, if $(j, k)\in I_1$, then $j+k\ge-\max(|s|,|t|)+\bar{T}$. By combining \eqref{eq:case1int1}, \eqref{eq:case1int2} and \eqref{eq:case1int3}, we obtain that 
\begin{equation}\label{eq:case1fin}
\begin{split}
    &\sup_{\substack{s, t\le 0\\\max(|s|, |t|)<\bar{T}}} \xi^{\max(|s|, |t|)}|(\btSigma_{\bar{T}})_{s, t} - h^\Sigma_{s, t}(\bx)|\\
    &\hspace{3em}\le R'\left(\frac{1}{\xi\alpha}\right)\left(\frac{1}{\xi\alpha}\right)^2 \|\btSigma_{\bar{T}-1} - \bx\|_\xi + \frac{\alpha^2\rho_\alpha^2+a^*}{\alpha^2}\sup_{0\le t\le \bar{T}}\xi^t\sum_{i=\bar{T}-t}^\infty \kappa_{i+2}^\infty (i+1)\left(\frac{1}{\alpha}\right)^{i}.
\end{split}
\end{equation}

Let us now look at the case $\max(|s|, |t|)\ge \bar{T}$. Recall that $|h^\Sigma_{s, t}(\bx)|\le a^*$, $\tilde{\sigma}_{0, 0}=(1-\rho_\alpha^2)\alpha^2$ and $\tilde{\sigma}_{0, t}=0$ for $t\in [1, \bar{T}]$. Thus,
\begin{equation*}
    |(\btSigma_{\bar{T}})_{s, t} - h^\Sigma_{s, t}(\bx)|\le c_1,
\end{equation*}
where $c_1$ is a constant independent of $s, t, \bar{T}$. This immediately implies that
\begin{equation*}
    \sup_{\substack{s, t\le 0\\\max(|s|, |t|)\ge \bar{T}}} \xi^{\max(|s|, |t|)}|(\btSigma_{\bar{T}})_{s, t} - h^\Sigma_{s, t}(\bx)|\le c_1 \xi^{\bar{T}},
\end{equation*}
which combined with \eqref{eq:case1fin} allows us to conclude that
\begin{equation}
\begin{split}
    \|\btSigma_{\bar{T}} - h^\Sigma(\bx)\|_\xi &\le R'\left(\frac{1}{\xi\alpha}\right)\left(\frac{1}{\xi\alpha}\right)^2 \|\btSigma_{\bar{T}-1} - \bx\|_\xi\\ &\hspace{2em}+\frac{\alpha^2\rho_\alpha^2+a^*}{\alpha^2}\sup_{0\le t\le \bar{T}}\xi^t\sum_{i=\bar{T}-t}^\infty \kappa_{i+2}^\infty (i+1)\left(\frac{1}{\alpha}\right)^{i}+c_1 \xi^{\bar{T}}.
    \end{split}
\end{equation}
As $\alpha>\alpha_{\rm s}$ and the series in \eqref{eq:R1} is convergent for $z<1/\alpha_{\rm s}$, one readily verifies that 
\begin{equation}
\lim_{\bar{T}\to\infty}    \sup_{0\le t\le \bar{T}}\xi^t\sum_{i=\bar{T}-t}^\infty \kappa_{i+2}^\infty (i+1)\left(\frac{1}{\alpha}\right)^{i}=0,
\end{equation}
which concludes the proof.
\end{proof}

Finally, we can put everything together and prove Lemma \ref{lem:SE_FP_phase1}. 

\begin{proof}[Proof of Lemma \ref{lem:SE_FP_phase1}]
Fix $\epsilon>0$ and denote by $\left(h^\Sigma\right)^{T_0}$ the $T_0$-fold composition of $h^\Sigma$. Recall from Lemmas \ref{lemma:fixed} and \ref{lemma:cont} that $\bx^*$ is the unique fixed point of $h^\Sigma : X_{I^*} \to X_{I^*}$. Then, for any $\bx\in \mathcal X_{I^*}$,
\begin{equation}\label{eq:pt1}
    \|\left(h^\Sigma\right)^{T_0}(\bx)-\bx^*\|_\xi = \|\left(h^\Sigma\right)^{T_0}(\bx)-\left(h^\Sigma\right)^{T_0}(\bx^*)\|_\xi \le \left(R'\left(\frac{1}{\xi\alpha}\right)\left(\frac{1}{\xi\alpha}\right)^2\right)^{T_0}\|\bx-\bx^*\|_\xi,
\end{equation}
where the inequality follows from Lemma \ref{lemma:cont}. Note that $R'\left(\frac{1}{\xi\alpha}\right)\left(\frac{1}{\xi\alpha}\right)^2<1$ (see \eqref{eq:Lip1}) and that $\bx, \bx^*\in \mathcal X_{I^*}$. Thus, we can make the RHS of \eqref{eq:pt1} smaller than $\epsilon/2$ by choosing a sufficiently large $T_0$.
Furthermore, an application of Lemma \ref{lemma:app} gives that, for all sufficiently large $\bar{T}$,
\begin{equation}\label{eq:pt2}
    \|\bSigma_{\bar{T}+T_0}-\left(h^\Sigma\right)^{T_0}(\bx)\|_\xi\le \left(R'\left(\frac{1}{\xi\alpha}\right)\left(\frac{1}{\xi\alpha}\right)^2\right)^{T_0}\|\bSigma_{\bar{T}}-\bx\|_\xi+\frac{\epsilon}{4}.
\end{equation}
Note that $\bx\in \mathcal X_{I^*}$ implies that $\|\bx\|_\xi\le a^*$. In addition, by following the same argument as in Lemma \ref{lemma:image}, one can show that $|\tilde{\sigma}_{s, t}|\le a^*$ for all $s, t$, which in turn implies that $\|\bSigma_{\bar{T}}\|_\xi\le a^*$. As a result, we can make the RHS of \eqref{eq:pt2} is smaller than $\epsilon/2$ by choosing sufficiently large $T_0$. As the RHS of both \eqref{eq:pt1} and \eqref{eq:pt2} can be made smaller than $\epsilon/2$, an application of the triangle inequality gives that  
\begin{equation}
    \limsup_{\bar{T}\to\infty}\|\bSigma_{\bar{T}} -\bx^*\|_\xi\le \epsilon,
\end{equation}
which, after setting $\bar{T}=T+1$, implies the desired result.
\end{proof}

\subsection{Convergence to PCA Estimator for the First Phase}\label{app:confirstest}

In this section, we prove that the artificial AMP iterate at the end of the first phase converges to the PCA estimator in normalized $\ell_2$-norm.

\begin{lemma}[Convergence to PCA estimator  -- Square matrices]
Consider the setting of Theorem \ref{thm:square}, and the first phase of the artificial AMP iteration described in \eqref{eq:AMPfake1}, with the initialization  given in \eqref{eq:art_sq_init}. Assume that $\kappa_i^\infty\ge 0$ for all $i\ge 2$, and that $\alpha>\alpha_{\rm s}$. Then, 
\begin{equation}
\lim_{T \to \infty} \lim_{n \to \infty} \frac{1}{\sqrt{n}} \|\tilde{\bu}^{T+1}-\sqrt{n}\bu_{\rm PCA}\| =0 \ \text{ almost surely}.
     \label{eq:xstxT_diff}
\end{equation}
\label{lem:Phase1_PCA_conv}
\end{lemma}

\begin{proof}
Consider the following decomposition of $\tilde{\bu}^{T+1}$:
\begin{equation}\label{eq:dc}
    \tilde{\bu}^{T+1} = \zeta_{T+1} \bu_{\rm PCA}+\br^{T+1},
\end{equation}
where $\zeta_{T+1} = \langle \tilde{\bu}^{T+1}, \bu_{\rm PCA}\rangle$ and $\langle \br^{T+1}, \bu_{\rm PCA}\rangle=0$. Define
\begin{equation}\label{eq:err}
    \be^{T+1} = \left(\bX-G^{-1}\left(\frac{1}{\alpha}\right) \bI_n\right)\tilde{\bu}^{T+1},
\end{equation}
where $G^{-1}$ is the inverse of the Cauchy transform of $\Lambda$. Then,  using \eqref{eq:dc}, \eqref{eq:err} can be rewritten as
\begin{equation}
    \left(\bX-G^{-1}\left(\frac{1}{\alpha}\right) \bI_n\right)\br^{T+1} = \be^{T+1} - \left(\bX-G^{-1}\left(\frac{1}{\alpha}\right) \bI_n\right)\zeta_{T+1}\bu_{\rm PCA}.
\end{equation}
First, we will show that 
\begin{equation}\label{eq:lb1lemma}
\left\|    \left(\bX-G^{-1}\left(\frac{1}{\alpha}\right) \bI_n\right)\br^{T+1}\right\|\ge c\|\br^{T+1}\|,
\end{equation}
where $c >0$ is a constant (independent of $n, T$). We start by observing that the matrix $\bX-G^{-1}\left(\frac{1}{\alpha}\right) \bI_n$ is symmetric, hence it can be written in the form $\bQ\tilde{\bLambda}\bQ^\sT$, with $\bQ$ orthogonal and $\tilde{\bLambda}$ diagonal. Furthermore, the columns of $\bQ$ are the eigenvectors of $\bX-G^{-1}\left(\frac{1}{\alpha}\right) \bI_n$ and the diagonal entries  of $\tilde{\bLambda}$ are the corresponding eigenvalues. As $\br^{T+1}$ is orthogonal to $\bu_{\rm PCA}$,  we can write
\begin{equation}\label{eq:comprt}
    \left(\bX-G^{-1}\left(\frac{1}{\alpha}\right) \bI_n \right) \br^{T+1}=\bQ\tilde{\bLambda}'\bQ^\sT\br^{T+1},
\end{equation}
where $\tilde{\bLambda}'$ is obtained from $\tilde{\bLambda}$ by changing the entry corresponding to $\lambda_1(\bX)-G^{-1}\left(\frac{1}{\alpha}\right)$ to any other value. For our purposes, it suffices to substitute $\lambda_1(\bX)-G^{-1}\left(\frac{1}{\alpha}\right)$ with $\lambda_2(\bX)-G^{-1}\left(\frac{1}{\alpha}\right)$. Note that
\begin{equation}\label{eq:comprt2}
\begin{split}
        \|\bQ\tilde{\bLambda}'\bQ^\sT\br^{T+1}\|^2 &\ge \|\br^{T+1}\|^2 \min_{\bs:\|\bs\|=1}\|\bQ\tilde{\bLambda}'\bQ^\sT\bs\|^2\\
        &=\|\br^{T+1}\|^2\min_{\bs:\|\bs\|=1}\<\bs, \bQ\left(\tilde{\bLambda}'\right)^2\bQ^\sT\bs\>\\
        &=\|\br^{T+1}\|^2\,\lambda_{\rm min}(\bQ\left(\tilde{\bLambda}'\right)^2\bQ^\sT),
\end{split}
\end{equation}
where $\lambda_{\rm min}(\bQ\left(\tilde{\bLambda}'\right)^2\bQ^\sT)$ denotes the smallest eigenvalue of $\bQ\left(\tilde{\bLambda}'\right)^2\bQ^\sT$ and the last equality follows from the variational characterization of the smallest eigenvalue of a symmetric matrix. Note that
\begin{equation}\label{eq:Lambda}
    \lambda_{\rm min}(\bQ\left(\tilde{\bLambda}'\right)^2\bQ^\sT)=\lambda_{\rm min}\left((\tilde{\bLambda}')^2\right)=\min_{i\in\{2, \ldots, n\}}\left(\left(G^{-1}\left(\frac{1}{\alpha}\right)-\lambda_i(\bX)\right)^2\right).
\end{equation}
Recall that, for $\alpha>\alpha_{\rm s}$,  $\lambda_1(\bX)\stackrel{\mathclap{\mbox{\footnotesize a.s.}}}{\longrightarrow} G^{-1}(1/\alpha)$ and $\lambda_2(\bX)\stackrel{\mathclap{\mbox{\footnotesize a.s.}}}{\longrightarrow} b<G^{-1}(1/\alpha)$, see \cite[Theorem 2.1]{benaych2011eigenvalues}. Thus, the RHS of \eqref{eq:Lambda} is lower bounded by a constant independent of $n, T$. By combining this result with \eqref{eq:comprt} and \eqref{eq:comprt2}, we deduce that \eqref{eq:lb1lemma}
holds.

Next, we prove that a.s.
\begin{equation}\label{eq:lastlim}
    \lim_{T\to\infty}\lim_{n\to\infty}\frac{1}{\sqrt{n}}\left\|\be^{T+1} - \left(\bX-G^{-1}\left(\frac{1}{\alpha}\right) \bI_n\right)\zeta_{T+1}\bu_{\rm PCA}\right\|=0.
\end{equation}
An application of the triangle inequality gives that 
\begin{equation}\label{eq:pt2re}
\left\|\be^{T+1} - \left(\bX-G^{-1}\left(\frac{1}{\alpha}\right) \bI_n\right)\zeta_{T+1}\bu_{\rm PCA}\right\|\le \left\|\be^{T+1}\right\|+\left\| \left(\bX-G^{-1}\left(\frac{1}{\alpha}\right) \bI_n\right)\zeta_{T+1}\bu_{\rm PCA}\right\|.
\end{equation}
The second term on the RHS of \eqref{eq:pt2re} is equal to 
\begin{equation}
    |\zeta_{T+1}| \left|\lambda_1(\bX)-G^{-1}\left(\frac{1}{\alpha}\right)\right|.
\end{equation}
By using Theorem 2.1 of \cite{benaych2011eigenvalues}, we have that, for $\alpha>\alpha_{\rm s}$, almost surely,
\begin{equation}\label{eq:0lim}
   \lim_{n\to\infty} \left|\lambda_1(\bX)-G^{-1}\left(\frac{1}{\alpha}\right)\right|=0.
\end{equation}
Furthermore,
\begin{equation*}
    \frac{1}{\sqrt{n}}|\zeta_{T+1}|\le \frac{1}{\sqrt{n}}\|\tilde{\bu}^{T+1}\|=\frac{1}{\alpha\sqrt{n}}\|\tilde{\bdff}^{T}\|.
\end{equation*}
By Proposition \ref{prop:tilSE}, we have that
\begin{equation*}
    \lim_{n\to\infty}\frac{1}{\alpha\sqrt{n}}\|\tilde{\bdff}^{T}\| = \frac{1}{\alpha}\sqrt{\tilde{\mu}_{T}^2+\tilde{\sigma}_{T, T}},
\end{equation*}
which, for sufficiently large $T$, is upper bounded by a constant independent of $n, T$, as $\tilde{\mu}_{T}=\alpha\rho_\alpha$ and $\tilde{\sigma}_{T, T}$ converges to $\alpha^2(1-\rho_\alpha^2)$ as $T\to\infty$ by Lemma \ref{lem:SE_FP_phase1}. By combining this result with \eqref{eq:0lim}, we deduce that 
\begin{equation}\label{eq:scd}
   \lim_{T\to\infty}\lim_{n\to\infty}\frac{1}{\sqrt{n}}  \left\| \left(\bX-G^{-1}\left(\frac{1}{\alpha}\right) \bI_n\right)\zeta_{T+1}\bu_{\rm PCA}\right\|=0.
\end{equation}
To bound the first term on the RHS of \eqref{eq:pt2re}, we proceed as follows:
\begin{equation}\label{eq:etbd}
    \begin{split}
\lim_{n\to\infty}        \frac{1}{n}\|\be^{T+1}\|^2 &=\lim_{n\to\infty} \frac{1}{n}\left\|\left(\bX-G^{-1}\left(\frac{1}{\alpha}\right) \bI_n\right)\tilde{\bu}^{T+1}\right\|^2\\
        &\stackrel{\mathclap{\mbox{\footnotesize (a)}}}{=} \lim_{n\to\infty}\frac{1}{n}\left\| \tilde{\bdff}^{T+1}+\sum_{i=1}^{T+1} \kappa_{T-i+2}\left(\frac{1}{\alpha}\right)^{T-i+1}\tilde{\bu}^i -G^{-1}\left(\frac{1}{\alpha}\right)\tilde{\bu}^{T+1}\right\|^2\\
        &\stackrel{\mathclap{\mbox{\footnotesize (b)}}}{=} \lim_{n\to\infty}\frac{1}{n}\left\|\tilde{\bdff}^{T+1}+\sum_{i=1}^{T+1} \kappa_{T-i+2}^\infty\left(\frac{1}{\alpha}\right)^{T-i+1}\tilde{\bu}^i -G^{-1}\left(\frac{1}{\alpha}\right)\tilde{\bu}^{T+1}\right\|^2\\
        & \stackrel{\mathclap{\mbox{\footnotesize (c)}}}{=} \E \left\{\left(\tilde{F}_{T+1}+\sum_{i=1}^{T+1} \kappa_{T-i+2}^\infty\left(\frac{1}{\alpha}\right)^{T-i+1} \tilde{U}_i-G^{-1}\left(\frac{1}{\alpha}\right)\tilde{U}_{T+1}\right)^2\right\}.
    \end{split}
\end{equation}
Here, (a) uses the iteration \eqref{eq:AMPfake1} of the first phase of the artificial AMP, and (c) follows from Proposition \ref{prop:tilSE}, where $\tilde{U}_t$ for $t\in [1, T+1]$ and $\tilde{F}_{T+1}$ are defined in \eqref{eq:tilU_tilF_def}. To obtain (b), we write
\begin{equation}
    \begin{split}
        & \lim_{n \to \infty} \frac{1}{n} \Big\|  \sum_{i=1}^{T+1} (\kappa_{T-i+2} -\kappa_{T-i+2}^\infty) \left(\frac{1}{\alpha}\right)^{T-i+1} \tbu^{i}  \Big\|^2   \\
        & = 
        \lim_{n \to \infty} \sum_{i,j=1}^{T+1} (\kappa_{T-i+2} -\kappa_{T-i+2}^\infty) (\kappa_{T-j+2} -\kappa_{T-j+2}^\infty) \left(\frac{1}{\alpha}\right)^{2T-i-j+2} 
        \frac{\< \tbu^{i}, \tbu^{j} \>}{n} \ .
    \end{split}
    \label{eq:kappj_kappj_inf}
\end{equation}
Using the state evolution result of  Proposition \ref{prop:tilSE} and \eqref{eq:tilU_tilF_def}, we almost surely have
\begin{align}
   \lim_{n \to \infty} \frac{\< \tbu^{i}, \, \tbu^{j} \>}{n} = \frac{1}{\alpha^2}(\alpha^2 \rho_\alpha^2 + \tsigma_{i, j}) < 1,
   \label{eq:uiuj_inner_prod}
\end{align} 
where the last inequality uses  $\tsigma_{i, j} < \tsigma_{0,0}= \alpha^2 (1- \rho_\alpha^2)$.  (This can be deduced from the recursion \eqref{eq:SEfake1} using the formula \eqref{eq:rho_alpha_def} for $\rho_\alpha^2$, and the relations \eqref{eq:R1} and \eqref{eq:R1G}.) 
Therefore, since  $\kappa_{i} \stackrel{n \to \infty}{\longrightarrow} \kappa^\infty_{i}$ for $i \in [1,T+1]$ (by the model assumptions), we almost surely have that (b) holds.

Next, by the triangle inequality, \eqref{eq:etbd} is upper bounded by
\begin{equation}\label{eq:etbd2}
\begin{split}
    3 \cdot \,  &\E\left\{\left(\alpha-G^{-1}\left(\frac{1}{\alpha}\right)+\sum_{i=1}^{T+1} \kappa_{T-i+2}^\infty\left(\frac{1}{\alpha}\right)^{T-i+1}\right)^2\tilde{U}_{T+1}^2\right\} \\
    &+ \,  3\cdot\E\left\{\left(\sum_{i=1}^{T+1} \kappa_{T-i+2}^\infty\left(\frac{1}{\alpha}\right)^{T-i+1}(\tilde{U}_{i}-\tilde{U}_{T+1})\right)^2\right\}\\
    &+ \, 3\cdot\E\left\{(\tilde{F}_{T+1}-\alpha\tilde{U}_{T+1})^2\right\}:= S_1 + S_2 + S_3.
\end{split}
\end{equation}
The term $S_3$ can be expressed as
\begin{equation*}
        S_3 = 3\cdot\E\left\{(\tilde{F}_{T+1}-\tilde{F}_{T})^2\right\}=3(\tilde{\sigma}_{T+1, T+1}-2\tilde{\sigma}_{T+1, T}+\tilde{\sigma}_{T, T}).
\end{equation*}
Thus, by Lemma \ref{lem:SE_FP_phase1}, we have that 
\begin{equation}\label{eq:S3lim}
    \lim_{T\to\infty} S_3 = 0.
\end{equation}
The term $S_1$ can be expressed as 
\begin{equation*}
        S_1 =3\cdot \left(\alpha-G^{-1}\left(\frac{1}{\alpha}\right)+\sum_{i=1}^{T+1} \kappa_{T-i+2}^\infty\left(\frac{1}{\alpha}\right)^{T-i+1}\right)^2\frac{\alpha^2\rho_\alpha^2+\tilde{\sigma}_{T, T}}{\alpha^2}.
\end{equation*}
By Lemma \ref{lem:SE_FP_phase1}, we have $\lim_{T \to \infty} \tsigma_{T,T} = \alpha^2(1-\rho_{\alpha^2})$, and hence
\begin{equation}\label{eq:S1lim}
    \lim_{T\to\infty} S_1 =3\cdot \left(\alpha-G^{-1}\left(\frac{1}{\alpha}\right)+\sum_{i=0}^\infty \kappa_{i+1}^\infty\left(\frac{1}{\alpha}\right)^{i}\right)^2=0,
\end{equation}
where the last equality follows from  \eqref{eq:R} and \eqref{eq:RG}. Finally, consider the term $S_2$, which after expanding the square and some manipulations,  can be expressed as 
\begin{equation}
    \label{eq:S2_sumexp}
    S_2 =\frac{3}{\alpha^2}\sum_{i, j=0}^{T}\kappa_{i+1}^\infty\kappa_{j+1}^\infty \left(\frac{1}{\alpha}\right)^{i+j}\left(\tilde{\sigma}_{T-j, T-i}+\tilde{\sigma}_{T, T}-\tilde{\sigma}_{T, T-i}-\tilde{\sigma}_{T, T-j}\right).
\end{equation}
The expression above can be bounded above as 
\begin{equation}\label{eq:S1re}
\begin{split}
    S_2 & \le \frac{3}{\alpha^2}\sum_{i, j=0}^{T}\kappa_{i+1}^\infty\kappa_{j+1}^\infty \left(\frac{1}{\alpha}\right)^{i+j} \big(|\tilde{\sigma}_{T-j, T-i}-\alpha^2(1-\rho_\alpha^2)|+|\tilde{\sigma}_{T, T}-\alpha^2(1-\rho_\alpha^2)|\\
    &\hspace{10em}+|\tilde{\sigma}_{T, T-i}-\alpha^2(1-\rho_\alpha^2)|+|\tilde{\sigma}_{T, T-j}-\alpha^2(1-\rho_\alpha^2)|\big).
\end{split}
\end{equation}
We now apply Lemma \ref{lem:SE_FP_phase1} to bound each of the four absolute values on the RHS of \eqref{eq:S1re}. Fix any $\xi \in (\frac{\alpha_{\rm s}}{\alpha}, 1)$. Then, by Lemma \ref{lem:SE_FP_phase1}, for any $\epsilon>0$ there exists $T^*(\epsilon)$ such that  for $T>T^*(\epsilon)$, we have
\begin{equation}
\label{eq:S2_small_bound}
\begin{split}
S_2 & \leq \epsilon \cdot  \frac{3}{\alpha^2}\sum_{i, j=0}^{T}\kappa_{i+1}^\infty\kappa_{j+1}^\infty \left(\frac{1}{\alpha}\right)^{i+j}\cdot  \big(\xi^{-\max(i, j)}+1+\xi^{-i}+\xi^{-j}\big)\\
&\stackrel{\mathclap{\mbox{\footnotesize (a)}}}{\le}\epsilon \cdot\frac{12}{\alpha^2}\sum_{i, j=0}^{T}\kappa_{i+1}^\infty\kappa_{j+1}^\infty \left(\frac{1}{\xi\alpha}\right)^{i+j}\\
&\stackrel{\mathclap{\mbox{\footnotesize (b)}}}{\le}\epsilon \cdot\frac{12}{\alpha^2}\sum_{i, j=0}^{\infty}\kappa_{i+1}^\infty\kappa_{j+1}^\infty \left(\frac{1}{\xi\alpha}\right)^{i+j}\\
&\stackrel{\mathclap{\mbox{\footnotesize (c)}}}{\le}\epsilon \cdot\frac{12}{\alpha^2}\left(R\left(\frac{1}{\xi\alpha}\right)\right)^2.
\end{split}
\end{equation}
Here, (a) uses that $\xi<1$, (b) uses that $\kappa_i^\infty\ge 0$ for $i\ge 2$, and (c) uses the power series expansion \eqref{eq:R} of $R(\cdot)$, which converges to a finite limit as $\xi\alpha>\alpha_{\rm s}$. Since $\epsilon$ can be arbitrarily small, we have
\begin{equation}\label{eq:S2lim}
    \lim_{T\to\infty} S_2 = 0.
\end{equation}
By combining \eqref{eq:etbd}, \eqref{eq:etbd2}, \eqref{eq:S3lim}, \eqref{eq:S1lim} and \eqref{eq:S2lim}, we have that 
\begin{equation}
   \lim_{T\to\infty}\lim_{n\to\infty}\frac{1}{\sqrt{n}}  \left\|\be^{T+1}\right\|=0,
\end{equation}
which, combined with \eqref{eq:scd}, gives \eqref{eq:lastlim}. Finally, by using \eqref{eq:lb1lemma} and \eqref{eq:lastlim}, we have that
\begin{equation}
    \lim_{T\to\infty}\lim_{n\to\infty}\frac{1}{\sqrt{n}}  \left\|\br^{T+1}\right\|=0. 
\end{equation}
Thus, from the decomposition \eqref{eq:dc}, we conclude that, as $n\to\infty$ and $T\to\infty$, $\tilde{\bu}^{T+1}$ is aligned with $\bu_{\rm PCA}$. Furthermore, from another application of Proposition \ref{prop:tilSE}, we obtain
\begin{equation}
 \lim_{T\to\infty}\lim_{n\to\infty}\frac{1}{\sqrt{n}}\|\tilde{\bu}^{T+1}\| = \lim_{T\to\infty}\frac{1}{\alpha}\sqrt{\tilde{\mu}_{T}^2+\tilde{\sigma}_{T, T}} = 1, 
\end{equation}
which implies that $\lim_{T\to\infty}\lim_{n\to\infty}\zeta_{T+1} = 1$ and concludes the proof.  
\end{proof}

\subsection{Analysis for the Second Phase} \label{sec:app_sec_phase_analysis}

We first define a  modified version of the true AMP algorithm, in which the memory coefficients $\{\sb_{t,i}\}_{i \in [1, t]}$ in \eqref{eq:AMPinit}-\eqref{eq:AMP} are replaced by deterministic values obtained from state evolution. The iterates of the modified AMP, denoted by $\hbu^t$, are given by:
\begin{align}
    \hbu^1 = \sqrt{n}\bu_{\rm PCA}, \quad \hbf^1 = \bX \hbu^1 - \bar{\sb}_{1,1}\hbu^1, \label{eq:modAMP_init} \\
    \hbu^t = \su_t(\hbf^{t-1}), \quad \hbf^t= \bX \hbu^t - \sum_{i=1}^t \bar{\sb}_{t,i}\hbu^i, \qquad t \ge 2, \label{eq:modAMP}
\end{align}
where
\beq
\begin{split}
& \bar{\sb}_{1,1} =\sum_{i=0}^\infty \kappa_{i+1}^\infty \alpha^{-i}, \\
& \bar{\sb}_{t,t}=\kappa_1^\infty, \,\,\,\,\,\, \bar{\sb}_{t,1}=\sum_{i=0}^\infty \kappa_{i+t}^\infty \alpha^{-i} \prod_{\ell=2}^t \E\{  \su_\ell'(F_{\ell-1}) \},  \\
&\bar{\sb}_{t,t-j}=\kappa_{j+1}^\infty \prod_{i=t-j+1}^t \E\{  \su_i'(F_{i-1}) \}, \mbox{ for }\,\, (t-j) \in [2, t-1]. 
\end{split}
\label{eq:modAMP_onsager}
\eeq
We recall that $\{ \kappa_i^\infty \}$ are the free cumulants of the limiting spectral distribution $\Lambda$, and the random variables $\{F_i\}$ are given by \eqref{eq:F1FK_def}.

The following lemma shows that, as $T$ grows, the iterates of the second phase of the artificial AMP approach those of the modified AMP algorithm above, as do the corresponding state evolution parameters.

\begin{lemma}
\label{lem:sec_phase_sq}
Consider the setting of Theorem \ref{thm:square}. Assume that $\kappa_i^\infty\ge 0$ for all $i\ge 2$, and that $\alpha>\alpha_{\rm s}$. Consider the modified version of the true AMP in \eqref{eq:modAMP_init}-\eqref{eq:modAMP}, and the artificial AMP in \eqref{eq:AMPfake1}-\eqref{eq:AMPfake2} along with its state evolution recursion given by \eqref{eq:SEfakeinit}-\eqref{eq:SE_artificial_sigma}. Then, the following results hold for $s,t \ge 1$:
\begin{enumerate}
\item 
\begin{align}
    \lim_{T \to \infty} \tilde{\mu}_{T+t} = \mu_t, \qquad 
    \lim_{T \to \infty} \tilde{\sigma}_{T+s, T+t} = \sigma_{s,t}.
    \label{eq:fake_true_SEconv}
\end{align}
\item For any $\PL(2)$ function $\psi:\reals^{2t+2} \to \reals$, we have
\end{enumerate}
\beq
\begin{split}
& \lim_{T \to \infty} \lim_{n \to \infty} 
 \bigg\vert  \frac{1}{n}  \sum_{i=1}^n \psi (u^*_{i}, \tu^{T+1}_i, \ldots, \tu^{T+t+1}_i, \tf^{T+1}_i, \ldots \tf^{T+t}_i)\\
&\hspace{9em} - \, 
\frac{1}{n}  \sum_{i=1}^n \psi (u^*_{i}, \hu^1_i, \ldots, \hu^{t+1}_i, \hf^1_i, \ldots \hf^{t}_i)  \bigg\vert  =0 \quad \text{ almost surely}.
\end{split}
\label{eq:fake_modified_conv}
\eeq
\end{lemma}

\begin{proof}
\textbf{ Proof of \eqref{eq:fake_true_SEconv}}.
We prove by induction. Consider the base case $t=1$. The formula in \eqref{eq:SEfake1} for $\tmu_t$ shows that $\tmu_{t} = \alpha \rho_\alpha = \mu_1$ for $t \in [1, T+1]$. Furthermore, Lemma \ref{lem:SE_FP_phase1} shows that 
$\lim_{T \to \infty} \tsigma_{T+1,T+1} = \alpha^2(1- \rho_\alpha^2)$, which equals $\sigma_{11}$ (defined right before \eqref{eq:F1FK_def}). 

For $t \ge 2$, assume towards induction that   $\lim_{T \to \infty} \tmu_{T+\ell} = \mu_\ell$ and $\sigma_{T+k, T+\ell} = \sigma_{k,\ell}$, for $k, \ell \in [1, t-1]$. From \eqref{eq:tilU_tilF_def}-\eqref{eq:SE_artificial_sigma}, we have
\beq
\tmu_{T+t} = \alpha \E\{  \su_t(\tmu_{T+t-1}U_* +  \tZ_{T+t-1}) \, U_*  \}.
\eeq
Recalling that $\tZ_{T+t-1} \sim \normal(0, \tsigma_{T+t-1, T+t-1})$ and $Z_{t-1} \sim \normal(0, \sigma_{t-1, t-1})$, by the induction hypothesis and the continuous mapping theorem, the sequence of random variables 
$\{\su_t( \tmu_{T+t-1}U_* +  \tZ_{T+t-1}) U_* \}$ converges in distribution  as $T \to \infty$ to $\su_t(\mu_{t-1}U_* +  Z_{t-1}) \, U_*$. We now claim that the sequence 
$\{  \su_t(\tmu_{T+t-1}U_* +  \tZ_{T+t-1}) \, U_* \}$ is uniformly integrable, from which it follows that \cite{billingsley2008probability}
\beq
\lim_{T \to \infty} \tmu_{T+t} =  \alpha \E\{ \su_t(\mu_{t-1}U_* +  Z_{t-1}) \, U_* \} = \mu_{t}.
\eeq
We show uniform integrability by showing that 
$\sup_{T} \, \E\{  | \su_t(\tmu_{T+t-1}U_* +  \tZ_{T+t-1}) U_* |^{1+\varepsilon/2} \}$ is bounded, where we recall that $\varepsilon >0$ is any constant such that $\E\{ U_*^{2+\varepsilon}\}$ exists. Using $L_t \ge 1$ to denote a Lipschitz constant of $\su_t$, we have 
\begin{align}
    & \E\{  | \su_t(\tmu_{T+t-1}U_* +  \tZ_{T+t-1})  U_* |^{1+\varepsilon/2} \}  \nonumber \\
     & \leq L_t^{1 + \varepsilon/2} \E\left\{ \abs{\abs{\tmu_{T+t-1}} U_*^2 +  |\tZ_{T+t-1} U_*| + |\su_t(0) U_* |}^{1+\varepsilon/2} \right\} \nonumber \\
     & \stackrel{\mathclap{\mbox{\footnotesize (a)}}}{\leq} (3L_t )^{1 + \varepsilon/2} \bigg( \abs{\tmu_{T+t-1}}^{1+ \varepsilon/2} \E\{ |U_*|^{2 + \varepsilon}  \} \, + \,   \big( \E\{|\tZ_{T+t-1}|^{1 + \varepsilon/2} \} + |\su_t(0)| ^{1 + \varepsilon/2} \big)
     \E\{ |U_*|^{1 + \varepsilon/2} \} \bigg)  \nonumber  \\
     & \stackrel{\mathclap{\mbox{\footnotesize (b)}}}{<} \infty,
     \label{eq:UI_utU}
\end{align}
where (a) is obtained using H{\"o}lder's inequality, and (b) holds because, by the induction hypothesis, $\tmu_{T+t-1} \to \mu_{t-1}$ and $\tsigma_{T+t-1, T+t-1} \to \sigma_{t-1, t-1}$. 

Next, consider $\tsigma_{T+s, T+t}$ for $s \in [1, t]$. From \eqref{eq:SE_artificial_sigma},
\begin{align}
    \tsigma_{T+s, T+t} &= \sum_{j=0}^{T+s-1} \sum_{k=0}^{T+t-1}
    \kappa_{j+k+2}^{\infty}  
    \left( \prod_{i=T+s-j+1}^{T+s} 
    \E\{ \tsu'_i (\tF_{i-1}) \} \right)  \nonumber  \\
    &\hspace{1.3in} \cdot \left( \prod_{i=T+t-k+1}^{T+t} \E\{ \tsu'_i (\tF_{i-1}) \} \right) \E\{ \tU_{T+s-j} \tU_{T+t-k} \}\nonumber \\
    & := A_1 + A_2 +A_3 + A_4,
\end{align}
where the four terms correspond to the sum over different subsets  of the indices $(j,k)$. By using the definition of $\tu_i(\cdot)$ in \eqref{eq:tilU_tilF_def}, those terms can be written as
\begin{align}
    A_{1} & = \sum_{j=0}^{s-2} \sum_{k=0}^{t-2} \kappa_{j+k+2}^{\infty}  
    \left( \prod_{i=s-j+1}^{s} \E\{ \su'_{i} (\tF_{T+i-1}) \} \right)\left( \prod_{i= t+1-k}^{t} \E\{ \su'_{i} (\tF_{T+i-1}) \} \right) \nonumber \\
    & \qquad \qquad \qquad \cdot \E\{ \tU_{T+s-j} \tU_{T+t-k} \},
\end{align}
\begin{align}
    A_{2} & = \sum_{j=0}^{s-2} \sum_{k=t-1}^{T+t-1} \Big(\frac{1}{\alpha}\Big)^{(k-t+1)} \kappa_{j+k+2}^{\infty}  
    \left( \prod_{i=s-j+1}^{s} \E\{ \su'_{i} (\tF_{T+i-1}) \} \right)\left( \prod_{i= 2}^{t} \E\{ \su'_{i} (\tF_{T+i-1}) \} \right) \nonumber \\
    & \hspace{1.5in}  \cdot \E\{ \tU_{T+s-j} \tU_{T+t-k} \},
    \label{eq:A2_def}
\end{align}
\begin{align}
    A_{3} & = \sum_{j=s-1}^{T+s-1} \sum_{k=0}^{t-2} \Big(\frac{1}{\alpha}\Big)^{(j-s+1)} \kappa_{j+k+2}^{\infty}  
    \left( \prod_{i=2}^{s} \E\{ \su'_{i} (\tF_{T+i-1}) \} \right)\left( \prod_{i= k-t+1}^{t} \E\{ \su'_{i} (\tF_{T+i-1}) \} \right) \nonumber \\
    & \hspace{1.5in}  \cdot \E\{ \tU_{T+s-j} \tU_{T+t-k} \},
    \label{eq:A3_def}
\end{align}
\begin{align} 
    A_{4} &= \sum_{j=s-1}^{T+s-1}\sum_{k=t-1}^{T+t-1} \Big(\frac{1}{\alpha}\Big)^{(j+k-s-t+2)} \kappa_{j+k+2}^{\infty}  \left( \prod_{i=2}^{s} \E\{ \su'_{i} (\tF_{T+i-1}) \} \right)\left( \prod_{i=2}^{t} \E\{ \su'_{i} (\tF_{T+i-1}) \} \right)  \nonumber \\
    & \hspace{1.5in}  \cdot \E\{ \tU_{T+s-j} \tU_{T+t-k} \}.
    \label{eq:A4_def}
\end{align}

 For $i \in [2,t]$, the induction hypothesis implies that $\tF_{T+i-1} = \tmu_{T+i-1}U_* + \tZ_{T+i-1} \stackrel{d}{\to} F_{i-1} = \mu_{i-1}U_* + Z_{i-1}$. Since $u_i$ is Lipschitz and continuously differentiable, Lemma \ref{lem:lipderiv} implies that
\beq 
\lim_{T \to \infty} \, \E\{ \su'_{i} (\tF_{T+i-1}) \} =  \E\{ \su'_{i} (F_{i-1}) \}, \qquad i \in [2,t]. 
\label{eq:ui_der_conv}
\eeq

Next, note that
\beq
\begin{split}
& \tU_{T+s-j} = \begin{cases}
\su_{s-j}(\tF_{T+s-j-1}),  & 0 \leq j \leq s-2, \\
\tF_{T+s-j-1}/\alpha, & s-1 \leq j \leq T+s-1,
\end{cases} \\
& \tU_{T+t-k} = \begin{cases}
\su_{t-k}(\tF_{T+t-k-1}),  & 0 \leq k \leq t-2, \\
\tF_{T+t-k-1}/\alpha, & t-1 \leq k \leq T+t-1.
\end{cases}
\end{split}
\label{eq:Ustjk}
\eeq
We separately consider $\E\{ \tU_{T+s-j} \tU_{T+t-k}  \}$ for the  four cases of $(j,k)$, corresponding to $A_1, A_2, A_3, A_4$. 
First, for $j \in [0,t-2]$, $k \in [0,s-2]$, we have
\begin{align}
\E\{ \tU_{T+s-j} \tU_{T+t-k}  \} = 
\E\{  \su_{s-j}(\tF_{T+s-j-1}) \, \su_{t-k}(\tF_{T+t-k-1}) \}.
\end{align}
By the induction hypothesis and the continuous mapping theorem, the sequence  $$\{ \su_{s-j}(\tF_{T+s-j-1}) \su_{t-k}(\tF_{T+t-k-1}) \}$$ converges in distribution to $ \su_{s-j}(F_{s-j-1}) \su_{t-k}(F_{t-k-1})$ as $T \to \infty$. From an argument similar to \eqref{eq:UI_utU}, we also deduce that $\{ \su_{s-j}(\tF_{T+s-j-1}) \su_{t-k}(\tF_{T+t-k-1}) \}$ is uniformly integrable, from which it follows that 
\beq
\begin{split}
    & \lim_{T \to \infty} \E\{  \su_{s-j}(\tF_{T+s-j-1}) \su_{t-k}(\tF_{T+t-k-1}) \} =  \E\{ \su_{s-j}(F_{s-j-1}) \su_{t-k}(F_{t-k-1}) \}, \\
    & \hspace{2in} j \in [0,t-2], \ k \in [0,s-2].
\end{split}
\label{eq:usuk_conv}
\eeq
Eqs. \eqref{eq:ui_der_conv} and \eqref{eq:usuk_conv} imply that 
\begin{align}
    \lim_{T \to \infty} A_1 = \sum_{j=0}^{s-2} \sum_{k=0}^{t-2} \kappa_{j+k+2}^{\infty}  
    \left( \prod_{i=s-j+1}^{s} \E\{ \su'_{i} (F_{i-1}) \} \right)\left( \prod_{i= t+1-k}^{t} \E\{ \su'_{i} (F_{i-1}) \} \right)  \E\{ U_{s-j} U_{t-k} \}.
    \label{eq:A1_limit}
\end{align}

Next consider the case where $j \in [s-1, \,T+s-1]$ and $k \in [t-1, \, T+t-1]$. Here, 
\beq 
\label{eq:UjUk_dev1}
\E\{ \tU_{T+s-j} \tU_{T+t-k}  \} = \frac{1}{\alpha^2}
\E\{  \tF_{T-(j+1-s)}  \tF_{T-(k+1-t)} \}
=\rho_\alpha^2 + \frac{1}{\alpha^2} \tsigma_{T-(j+1-s), T- (k+1-t)}.
\eeq
From Lemma \ref{lem:SE_FP_phase1}, for any $\delta >0$, for sufficiently large $T$, we have
\beq 
\label{eq:tsig_dev1}
|\tilde{\sigma}_{T-(j+1-s), T-(k+1-t)}-\alpha^2(1-\rho_\alpha^2)| < \delta \xi^{-\max(j+1-s, \, k+1-t)},
\eeq
for some $\xi >0$ such that $\xi \alpha > \alpha_s$.
Combining \eqref{eq:UjUk_dev1}-\eqref{eq:tsig_dev1} and noting from \eqref{eq:F1FK_def} that $\E\{U_{s-j} U_{t-k} \} =\frac{1}{\alpha^2} \E\{ F_1^2 \}= 1$ we obtain, for sufficiently large $T$:
\beq
| \E\{ \tU_{T+s-j} \tU_{T+t-k}  \} - \E\{ U_{s-j} U_{t-k}  \}| < \frac{\delta}{\alpha^2} \xi^{-\max(j+1-s, \, k+1-t)}, \quad \text{ for } j \ge (s-1), \ k \geq (t-1). 
\label{eq:tUjtUk_diff4}
\eeq
Now we write  $A_4$ in \eqref{eq:A4_def} as
\beq
\begin{split}
 A_4& =  \left( \prod_{i=2}^{s} \E\{ \su'_{i} (\tF_{T+i-1}) \} \right)\left( \prod_{i= 2}^{t} \E\{ \su'_{i} (\tF_{T+i-1}) \} \right) \\ 
&\qquad \cdot \left[  \sum_{j=s-1}^{T+s-1}\sum_{k=t-1}^{T+t-1}
\Big(\frac{1}{\alpha}\Big)^{(j+k-s-t+2)} \kappa_{j+k+2}^{\infty}  \, \E\{ U_{s-j} U_{t-k} \}   
\,  + \,  \Delta_{4}\right],
 \end{split}
 \label{eq:A4_split}
\eeq
where 
\beq 
\Delta_4 =  \sum_{j=s-1}^{T+s-1}\sum_{k=t-1}^{T+t-1}
\Big(\frac{1}{\alpha}\Big)^{(j+k-s-t+2)} \kappa_{j+k+2}^{\infty} \ [\E\{ \tU_{T+s-j} \tU_{T+t-k} \} - \E\{ U_{s-j} U_{t-k} \}].
\label{eq:Del4_def}
\eeq
Using \eqref{eq:tUjtUk_diff4}, for sufficiently large $T$ we have
\begin{equation}
    \begin{split}
        | \Delta_4 | & <  \frac{\delta}{\alpha^2} \, \sum_{j=s-1}^{T+s-1}\sum_{k=t-1}^{T+t-1}
\left(\frac{1}{\xi\alpha} \right)^{(j+k-s-t+2)}  \kappa_{j+k+2}^{\infty}  \\
& =  \frac{\delta}{\alpha^2} \, (\xi\alpha)^{s+t} \sum_{j=0}^{T}\sum_{k=0}^{T}
\left(\frac{1}{\xi\alpha} \right)^{(j+k +s +t)}  \kappa_{j+k+s+t}^{\infty} \\
& < C_{s,t} \delta ,
    \end{split}
\end{equation}
for a positive constant $C_{s,t}$ since the double sum is bounded for $\xi\alpha > \alpha_s$ (see \eqref{eq:R1}). Since $\delta >0$ is arbitrary, this shows that $\Delta_4 \to 0$ as $T \to \infty$. Using this in \eqref{eq:A4_split} along with \eqref{eq:ui_der_conv}, we obtain
\beq
\begin{split}
 \lim_{T \to \infty} A_4 & =    \left( \prod_{i=2}^{s} \E\{ \su'_{i} (F_{i-1}) \} \right)\left( \prod_{i= 2}^{t} \E\{ \su'_{i} (F_{i-1}) \} \right) \\
 & \qquad 
 \cdot \sum_{j=s-1}^{\infty}\sum_{k=t-1}^{\infty}
\Big(\frac{1}{\alpha}\Big)^{(j+k-s-t+2)} \kappa_{j+k+2}^{\infty}  \, \E\{ U_{s-j} U_{t-k} \}.
\end{split}
\label{eq:A4_limit}
\eeq

Next consider  $j \in [0,s-2]$, $k \in [t-1, \,T+t-1]$. Here
\begin{align}
 &  \E\{ \tU_{T+s-j} \tU_{T+t-k}  \} = \frac{1}{\alpha}
  \E\{  \su_{s-j}(\tF_{T+s-j-1})  \tF_{T-(k+1-t)}\} \nonumber \\
  & = \frac{1}{\alpha} \E\{  \su_{s-j}(\tF_{T+s-j-1}) \tF_{T+1} \}
  + \frac{1}{\alpha}\E\{  \su_{s-j}(\tF_{T+s-j-1}) (\tF_{T+1} - \tF_{T-(k+1-t)})\}.
  \label{eq:tUjtUk_split}
\end{align}
By the induction hypothesis and the uniform integrability of $\{  \su_{s-j}(\tF_{T+s-j-1})\}$, we have
\beq
\lim_{T \to \infty} \frac{1}{\alpha} \E\{  \su_{s-j}(\tF_{T+s-j-1}) \tF_{T+1} \} = \frac{1}{\alpha} \E\{ \su_{s-j}(F_{s-j-1}) F_1 \} = \E\{ U_{s-j} U_{t-k} \}.
\eeq
The second term in \eqref{eq:tUjtUk_split} can be bounded as follows, using the Cauchy-Schwarz inequality:
\begin{align}
    & |\E\{  \su_{s-j}(\tF_{T+s-j-1}) (\tF_{T+1} - \tF_{T-(k+1-t)})\}| \nonumber \\
    & \leq L_t (\tmu_{T+s-j-1}^2 + \tsigma_{T+s-j-1, T+s-j-1} + C)^{1/2} \big(\E\{ (\tF_{T+1} - \tF_{T-(k+1-t)})^2 \} \big)^{1/2}.
    \label{eq:tUUT1_diff}
\end{align}
Using Lemma \ref{lem:SE_FP_phase1}, for any $\delta >0$ and $T$ sufficiently large, we have
\begin{align}
    & \E\{ (\tF_{T+1} - \tF_{T-(k+1-t)})^2 \} \leq 
    |\sigma_{T+1, T+1} - \alpha^2 (1- \rho_\alpha^2)| \nonumber  \\
    & \quad \, + \, 
    |\sigma_{T-(k+1-t), T-(k+1-t)} - \alpha^2 (1- \rho_\alpha^2)|\, + \, 
    2|\sigma_{T-(k+1-t), T+1} - \alpha^2 (1- \rho_\alpha^2)| \nonumber  \\
    & <  \delta \xi^{-(k+1-t)}.
\end{align}
Combining \eqref{eq:tUjtUk_split}-\eqref{eq:tUUT1_diff}, we deduce that for any $\delta>0$, the following holds for  sufficiently large $T$:
\begin{align}
| \E\{ \tU_{T+s-j} \tU_{T+t-k}  \}  - \E\{ U_{s-j} U_{t-k} \}|
< \delta \xi^{-(k+1-t)}, \ \text{ for }  j \in [0, s-2], \ k \in [t-1, T+t-1].
\label{eq:Delta2_bnd0}
\end{align}
We write $A_2$ in \eqref{eq:A2_def} as
\begin{align}
  A_2 & =  \left(  \prod_{i= 2}^{t} \E\{ \su'_{i} (\tF_{T+i-1}) \} \right)\sum_{j=0}^{s-2} 
    \left( \prod_{i=s-j+1}^{s} \E\{ \su'_{i} (\tF_{T+i-1}) \} \right) \nonumber \\
    & \qquad   \cdot\left[ \,  \sum_{k=t-1}^{T+t-1}  \Big(\frac{1}{\alpha}\Big)^{(k-t+1)} \kappa_{j+k+2}^{\infty} \,   \E\{ U_{s-j} U_{t-k} \} + \Delta_{2,j} \right],  
\label{eq:A2_split}
\end{align}
where
\beq
\Delta_{2,j} = \sum_{k=t-1}^{T+t-1} \Big(\frac{1}{\alpha}\Big)^{(k-t+1)} \kappa_{j+k+2}^{\infty} \,   
(\E\{ \tU_{T+s-j} \tU_{T+t-k} \} - \E\{ U_{s-j} U_{t-k}) \}).
\label{eq:Delta2_def}
\eeq
From \eqref{eq:Delta2_bnd0}, for any $\delta >0$ and sufficiently large $T$ we have
\begin{align}
    |\Delta_{2,j}| < \delta (\xi \alpha)^{j+t+1} \sum_{k=1}^{T+1} 
    \left( \frac{1}{\xi \alpha} \right)^{j+k+ t} \kappa^\infty_{j+k+t} < C_{s,j} \delta,
    \label{eq:bnd_Csj}
\end{align}
for a positive constant $C_{s,j}$ since the sum over $k$ is bounded (see \eqref{eq:R}). Using this in \eqref{eq:A2_split}  along with \eqref{eq:ui_der_conv}, we obtain
\begin{align}
    \lim_{T \to \infty} A_2 &  =   \sum_{j=0}^{s-2} \sum_{k=t-1}^{\infty}
      \Big(\frac{1}{\alpha}\Big)^{(k-t+1)} \kappa_{j+k+2}^{\infty} \,   \E\{ U_{s-j} U_{t-k} \}      \left( \prod_{i=s-j+1}^{s} \E\{ \su'_{i} (F_{i-1}) \} \right) \nonumber  \\
     & \hspace{1.9in} \cdot  \left(  \prod_{i= 2}^{t} \E\{ \su'_{i} (F_{i-1}) \} \right).
     \label{eq:A2lim}
\end{align}
Using a similar argument, we also have 
\begin{align}
 \lim_{T \to \infty} A_3 &  =   \sum_{j=s-1}^{\infty} \sum_{k=0}^{t-2}
      \Big(\frac{1}{\alpha}\Big)^{(j-s+1)} \kappa_{j+k+2}^{\infty} \,   \E\{ U_{s-j} U_{t-k} \}      \left( \prod_{i=2}^{s} \E\{ \su'_{i} (F_{i-1}) \} \right) \nonumber  \\
     & \hspace{1.9in} \cdot  \left(  \prod_{i= t-k+1}^{t} \E\{ \su'_{i} (F_{i-1}) \} \right).
     \label{eq:A3lim}
\end{align}
Noting that the sum of the limits in \eqref{eq:A1_limit}, \eqref{eq:A4_limit},  \eqref{eq:A2lim} and \eqref{eq:A3lim} equals $\sigma_{s,t}$ (defined in  \eqref{eq:sq_sigma_st}), we have shown that $\lim_{T \to \infty} \tsigma_{T+s, T+t} = \sigma_{s,t} $.

\textbf{Proof of \eqref{eq:fake_modified_conv}}.
 Since $\psi \in \PL(2)$, for some universal constant $C >0$ we have
\begin{align}
& 
\left\vert \frac{1}{n}  \sum_{i=1}^n \psi (u^*_{i},  \tu^{T+1}_i, \ldots, \tu^{T+t+1}_i, \tf^{T+1}_i, \ldots \tf^{T+t}_i)
\, - \, 
\frac{1}{n}  \sum_{i=1}^n \psi (u^*_{i}, \hu^1_i, \ldots, \hu^{t+1}_i, \hf^1_i, \ldots \hf^{t}_i)  \right\vert  \nonumber \\
 & \le \frac{C}{n} \sum_{i=1}^n \left( 1 + |u^*_i| + \sum_{\ell=1}^{t+1} \big( |\tu^{T+\ell}_i| + |\hu^{\ell}_i| \big)
 + \sum_{\ell=1}^{t} \big( |\tf^{T+\ell}_i| + |\hf^{\ell}_i| \big) \right)
 \nonumber \\
 & \qquad  \quad \cdot \left( (\tu^{T+1}_i- \hat{u}^1_i)^2+ \ldots + (\tu^{T+t+1}_i - \hat{u}^{t+1}_i)^2 + (\tf^{T+1}_i- \hat{f}^1_i)^2+ \ldots + (\tf^{T+t}_i - \hat{f}^t_i)^2\right)^{\frac{1}{2}} \nonumber  \\
  & \le 2C(t+2) \left[ 1 + \frac{\| \bu^* \|^2}{n}  +  
 \sum_{\ell=1}^{t+1} \Big( \frac{\| \tbu^{T+\ell} \|^2}{n} +
 \frac{\| \hbu^{\ell} \|^2}{n} \Big) + 
 \sum_{\ell=1}^{t} \Big( \frac{\| \tbf^{T+\ell} \|^2}{n}
 + \frac{\| \hbf^{\ell} \|^2}{n} \Big) \right]^{\frac{1}{2}} \nonumber  \\ 
 &  \cdot \left( \frac{\| \tbu^{T+1} - \hbu^1 \|^2}{n}+ \ldots + \frac{\| \tbu^{T+t+1} - \hbu^{t+1} \|^2}{n}  +\frac{\| \tbf^{T+1} - \hbf^1  \|^2}{n} + \ldots  + \frac{\| \tbf^{T+t} - \hbf^t  \|^2}{n}\right)^{\frac{1}{2}}, \label{eq:PL2_psi_bnd}
\end{align}
where the last inequality is obtained by using Cauchy-Schwarz inequality (twice).

 We will inductively show that in the limit $T,n \to \infty$ (with the limit in $n$ taken first): i) the terms $\frac{\| \tbu^{T+1} - \hbu^1 \|^2}{n}, \ldots, \frac{\| \tbf^{T+1} - \hbf^1  \|^2}{n}, \ldots$, $\frac{\| \tbf^{T+t} - \hbf^t  \|^2}{n}$ all converge to $0$ almost surely, and ii) each of the terms within the square brackets in \eqref{eq:PL2_psi_bnd} converges to a finite deterministic value.

\underline{Base case: $t=1$}. From Lemma  \ref{lem:Phase1_PCA_conv}, we have 
\beq \label{eq:ut1_hu1_conv}
\lim_{T \to \infty} \lim_{n \to \infty} 
\frac{\| \tbu^{T+1} - \hbu^1 \|^2}{n} =0.
\eeq 
From the definitions of $\tbf^{T+1}$ and $\hbf^1$ in \eqref{eq:AMPfake1} and \eqref{eq:modAMP_init}, we have
\beq
\label{eq:fT1f1_split}
\begin{split}
\| \tbf^{T+1} - \hbf^1 \|^2&  = \Big \| \bX( \tbu^{T+1} - \hbu_1 )
- \big( \sum_{i=1}^{T+1}\tsb_{T+1,i}  \tbu^{i}\, - \, \bar{\sb}_{1,1} \hbu^1  \big) \Big \|^2 \\
& \leq 2 \| \bX \|^2_{\rm{op}} \| \tbu^{T+1} - \hbu^1 \|^2 \,  + \, 
\Big\| \sum_{i=1}^{T+1}\tsb_{T+1,i}  \tbu^{i}\, - \, \bar{\sb}_{1,1} \hbu^1  \Big\|^2 .
\end{split}
\eeq
From \cite[Theorem 2.1]{benaych2011eigenvalues}, we know that the $\| \bX \|_{\rm op} = |\lambda_1(\bX)| \stackrel{n \to \infty}{\to} |G^{-1}(1/\alpha)|$ almost surely.  Therefore, from \eqref{eq:ut1_hu1_conv}, we almost surely have
\beq
\lim_{T \to \infty} \lim_{n \to \infty}   \| \bX \|^2_{\rm{op}} \frac{\| \tbu^{T+1} - \hbu^1 \|^2}{n}=0.
\label{eq:fT1f1_term1}
\eeq
For the second term in \eqref{eq:fT1f1_split}, recalling that  $\tsb_{T+1, T+1-j}= \kappa_{j+1} \alpha^{-j}$ for $j \in [0,T]$ (see \eqref{eq:AMPfake1}),  and $\bar{\sb}_{1,1}= \sum_{j=0}^\infty \kappa_{j+1}^\infty \alpha^{-j}$ (see \eqref{eq:modAMP_init}), we write
\begin{align}
   \sum_{i=1}^{T+1}\tsb_{T+1,i}  \tbu^{i}\, - \, \bar{\sb}_{1,1} \hbu^1 
   & =\sum_{j=0}^T (\kappa_{j+1} -\kappa_{j+1}^\infty) \alpha^{-j}  \tbu^{T+1-j} 
    +   \sum_{j=0}^T \kappa_{j+1}^\infty \alpha^{-j} (\tbu^{T+1-j} - \tbu^{T+1}) \nonumber \\
    & \qquad +
   \sum_{j=0}^T \kappa_{j+1}^\infty \alpha^{-j} (\tbu^{T+1} - \hbu^{1}) 
    -  \sum_{j=T+1}^\infty \kappa_{j+1}^{\infty}\alpha^{-j} \hbu^1. 
   %
\end{align}
Hence,
\begin{align}
   &  \frac{1}{n}  \Big\| \sum_{i=1}^{T+1}\tsb_{T+1,i}  \tbu^{i}\, - \, \bar{\sb}_{1,1} \hbu^1  \Big\|^2 \leq 
      \frac{4}{n} \Big\|  \sum_{j=0}^T (\kappa_{j+1} -\kappa_{j+1}^\infty) \alpha^{-j}  \tbu^{T+1-j}  \Big\|^2   \nonumber \\
    & \qquad +\frac{4}{n}  \Big \| \sum_{j=0}^T \kappa_{j+1}^\infty \alpha^{-j} (\tbu^{T+1-j} - \tbu^{T+1}) \Big \|^2  +   \frac{4}{n} \Big\| \sum_{j=0}^T \kappa_{j+1}^\infty \alpha^{-j} (\tbu^{T+1} - \hbu^{1}) \Big\|^2 \nonumber  \\
    & \qquad +   \frac{4}{n} \Big\| \sum_{j=T+1}^\infty \kappa_{j+1}^{\infty}\alpha^{-j} \hbu^1 \Big\|^2 := 
    R_1 +R_2 +R_3 +R_4.  \label{eq:three_term_split}
\end{align}
First, by using passages analogous to \eqref{eq:kappj_kappj_inf}-\eqref{eq:uiuj_inner_prod}, we almost surely have $\lim_{T \to \infty} \lim_{n \to \infty} R_1=0$. Considering $R_2$ next, Proposition \ref{prop:tilSE} implies that almost surely
\begin{align}
    & \lim_{n \to \infty} \frac{1}{n} \Big \| \sum_{j=0}^T \kappa_{j+1}^\infty \alpha^{-j} (\tbu^{T+1-j} - \tbu^{T+1}) \Big \|^2 =   \Big( \sum_{j=0}^T \kappa_{j+1}^\infty \alpha^{-j}     \E\{ \tU_{T+1-j} - \tU_{T+1}  \} \Big)^2 \nonumber \\
    & =
    \sum_{i=0}^T \sum_{j=0}^T \kappa_{i+1}^\infty \kappa_{j+1}^\infty \alpha^{-(i+j)} \frac{1}{n} 
    \E\{ (\tU_{T+1-i} - \tU_{T+1}) (\tU_{T+1-j} - \tU_{T+1}) \} 
    \nonumber \\
    & \stackrel{\mathclap{\mbox{\footnotesize (a)}}}{=} \sum_{i=0}^T \sum_{j=0}^T \kappa_{i+1}^\infty \kappa_{j+1}^\infty  \alpha^{-(i+j)} 
    ( \tsigma_{T-j, T-i} + \tsigma_{T,T} - \tsigma_{T-i, T} - \tsigma_{T-j,T} ). \label{eq:four_sigma_split1}
\end{align}
Here, (a) is obtained from the definition  $\tilde{U}_{\ell} = \tF_{\ell-1}/\alpha$ from \eqref{eq:tilU_tilF_def}, for $\ell \in [1, T+1]$. 
As $T \to \infty$, it was shown in \eqref{eq:S2_sumexp}-\eqref{eq:S2lim} that the sum on the RHS of \eqref{eq:four_sigma_split1} converges to $0$. Therefore 
\beq
\lim_{T \to \infty} \lim_{n \to \infty}   \frac{1}{n} \Big \| \sum_{j=0}^T \kappa_{j+1}^\infty \alpha^{-j} (\tbu^{T+1-j} - \tbu^{T+1}) \Big \|^2 =0 \ \text{ almost surely.}
\label{eq:tuTj_tuT_conv}
\eeq
For the third term in \eqref{eq:three_term_split}, recalling that $\hbu^1= \sqrt{n} \bu_{\PCA}$,  we almost surely have
\begin{align}\label{eq:extraeq}
    \lim_{T \to \infty} \Big( \sum_{j=0}^T \kappa_{j+1}^\infty \alpha^{-j} \Big)^2 
\lim_{T \to \infty} \lim_{n \to \infty} \frac{\| \tbu^{T+1} - \hbu^{1}\|^2}{n} =0,
\end{align}
where we use Lemma \ref{lem:Phase1_PCA_conv} and the fact that  $\sum_{j=0}^\infty \kappa_{j+1}^\infty \alpha^{-j} = R(1/\alpha)$ is convergent (see \eqref{eq:R}).  The convergence of this series also implies that 
$ \lim_{T \to \infty} \sum_{j=T+1}^\infty \kappa_{j+1}\alpha^{-j} = 0$, and hence the fourth term in  \eqref{eq:three_term_split} goes to $0$. We have therefore shown that
\beq
\lim_{T \to \infty} \lim_{n \to \infty} \frac{1}{n}  \Big\| \sum_{i=1}^{T+1}\tsb_{T+1,i}  \tbu^{i}\, - \, \bar{\sb}_{1,1} \hbu^1  \Big\|^2 = 0,
\label{eq:fT1f1_term2}
\eeq
almost surely. Using \eqref{eq:fT1f1_term1} and \eqref{eq:fT1f1_term2} in \eqref{eq:fT1f1_split} shows that almost surely
\beq
\lim_{T \to \infty} \lim_{n \to \infty} \, \frac{1}{n} \| \tbf^{T+1} - \hbf^1\|^2 =0. 
\label{eq:fT1f1_conv}
\eeq
Recalling that $\tbu^{T+2}= u_2(\tbu^{T+1})$, $\hbu^{2} = u_2(\hbf^1)$ and that $u_2$ is Lipschitz, we have $\| \tbu^{T+2} - \hbu^2\| \leq L_2\| \tbf^{T+1} - \hbf^1 \|$, where $L_2$ is the Lipschitz constant. Eq. \eqref{eq:fT1f1_conv} therefore implies
\beq
\lim_{T \to \infty} \lim_{n \to \infty} \, \frac{1}{n} \| \tbu^{T+2} - \hbu^2\|^2 = 0 \ \text{ almost surely}.
\label{eq:uT2u2_conv}
\eeq
By the triangle inequality, we have for $t \geq 1$:
\begin{equation}
    \begin{split}
    \| \tbu^{T+t} \| - \| \tbu^{T+t} - \hbu^{t} \|  \leq   \| \hbu^{t} \| \leq  \| \tbu^{T+t} \| +  \| \tbu^{T+t} - \hbu^{t} \|.
    \end{split}
    \label{eq:tbu_triangle}
\end{equation}
Therefore, from \eqref{eq:ut1_hu1_conv}, Proposition \ref{prop:tilSE}, \eqref{eq:F1FK_def} and \eqref{eq:Ut_def}, we almost surely have
\begin{equation}
    \label{eq:hu1_lim}
     \lim_{n \to \infty} \,  
    \frac{\| \hbu^1 \|^2}{n} = \lim_{T \to \infty}\lim_{n \to \infty} \, 
    \frac{\| \tbu^{T+1} \|^2}{n} = \lim_{T \to \infty} \frac{1}{\alpha^2}(\tmu_{T+1}^2 + \tsigma_{T+1, T+1}) \stackrel{\mathclap{\mbox{\footnotesize (a)}}}{=} \frac{1}{\alpha^2}(\mu_1^2+ \sigma_{1,1}) =1,
\end{equation}
where (a) is due to \eqref{eq:fake_true_SEconv}. Similarly using \eqref{eq:fT1f1_conv}, \eqref{eq:uT2u2_conv}, Proposition \ref{prop:tilSE}, and \eqref{eq:Ut_def}, we almost surely have
\begin{equation}
\begin{split}
        \label{eq:hu2_lim}
   &\lim_{n \to \infty} \,
    \frac{\| \hbu^2 \|^2}{n} = \lim_{T \to \infty}\lim_{n \to \infty} \,  
    \frac{\| \tbu^{T+2} \|^2}{n} = \E\{ u_2(\mu_2 U_* + Z_{2})^2\},\\
& \lim_{n \to \infty} \, 
    \frac{\| \hbf^1 \|^2}{n} = \lim_{T \to \infty}\lim_{n \to \infty} \, 
    \frac{\| \tbf^{T+1} \|^2}{n}  = \mu_1^2+ \sigma_{1,1} = \alpha^2.
\end{split}
\end{equation}
Using \eqref{eq:ut1_hu1_conv}, \eqref{eq:fT1f1_conv}, \eqref{eq:uT2u2_conv},
\eqref{eq:hu1_lim}, and  \eqref{eq:hu2_lim} in \eqref{eq:PL2_psi_bnd}, we conclude 
\begin{equation}
    \left\vert \frac{1}{n}  \sum_{i=1}^n \psi (u^*_{i},  \tu^{T+1}_i, \tu^{T+2}_i, \tf^{T+1}_i)
\, - \, 
\frac{1}{n}  \sum_{i=1}^n \psi (u^*_{i}, \hu^1_i,  \hu^{2}_i, \hf^1_i)  \right\vert  =0 \quad \text{ almost surely}.
\end{equation}

\underline{Induction step}: For $t \ge 2$, assume towards induction that almost surely 
\begin{align}
    &  \lim_{T \to \infty} \lim_{n \to \infty} \frac{1}{n}\| \tbf^{T+ \ell-1} - \hbf^{\ell-1} \|^2=0, \quad  \quad 
    \lim_{T \to \infty} \lim_{n \to \infty} \frac{1}{n}\| \tbu^{T+ \ell} - \hbu^{\ell} \|^2=0, \quad \text{ for }   2 \leq \ell  \leq t, \
    \nonumber \\
    &  \lim_{T \to \infty} \lim_{n \to \infty} 
 \bigg\vert  \frac{1}{n}  \sum_{i=1}^n \psi (u^*_{i}, \tu^{T+1}_i, \ldots, \tu^{T+\ell}_i, \tf^{T+1}_i, \ldots \tf^{T+\ell-1}_i)\nonumber\\
&\hspace{9em} - 
\frac{1}{n}  \sum_{i=1}^n \psi (u^*_{i}, \hu^1_i, \ldots, \hu^{\ell}_i, \hf^1_i, \ldots \hf^{\ell-1}_i)  \bigg\vert =0, \qquad \text{ for } \ \   2 \leq \ell  \leq t.
\label{eq:ind_hyp}
\end{align}
Using the definitions of $\tbf^{T+t}$ and $\hbf^t$ in \eqref{eq:AMPfake2} and \eqref{eq:modAMP} and applying the Cauchy-Schwarz inequality, we have
\begin{align}
   &  \frac{1}{n}\| \tbf^{T+t} - \hbf^{t} \|^2 \leq 
   \frac{(t+1)}{n}\Bigg( \| \bX (\tbu^{T+t} - \hbu^t) \|^2 + 
    \sum_{\ell=2}^t \| \tsb_{T+t,T+\ell} \tbu^{T+\ell} - \bar{\sb}_{t,\ell} \hbu^\ell \|^2  \nonumber \\
   & \hspace{2in}   + \Big\|  \sum_{i=1}^{T+1} \tsb_{T+t,i} \tbu^i  - \bar{\sb}_{t,1} \hbu^1\Big\|^2 \Bigg).
   \label{eq:fTt_ft_bnd}
\end{align}
For the first term on the right, we have  $ \| \bX (\tbu^{T+t} - \hbu^t) \|^2 \le \| \bX \|^2_{\rm op} \| \tbu^{T+t} - \hbu^t \|^2$. 
Since $\| \bX \|_{\rm op} \to | G^{-1}(1/\alpha) |$, using the induction hypothesis we obtain
\begin{equation}
    \lim_{T \to \infty} \lim_{n \to \infty} \frac{1}{n}  \| \bX (\tbu^{T+t} - \hbu^t) \|^2 =0 \quad \text{ almost surely}.
    \label{eq:XuTthut}
\end{equation}
Next consider $ \frac{1}{n}\| \tsb_{T+t,T+\ell} \tbu^{T+\ell} - \bar{\sb}_{t,\ell} \hbu^\ell \|^2$, which, for $\ell \in [2,t]$ can be bounded as 
\begin{equation}
   \frac{1}{n}\| \tsb_{T+t, T+\ell} \tbu^{T+\ell} - \bar{\sb}_{t,\ell} \hbu^\ell \|^2 \le 2 \tsb_{T+t, T+ \ell} 
   \frac{\| \tbu^{T+\ell} - \hbu^\ell \|^2}{n} + 2 \frac{\| \hbu^\ell \|^2}{n}(\tsb_{T+t, T+\ell} - \bar{\sb}_{t,\ell})^2.
   \label{eq:uTlul_bnd}
\end{equation}
By the induction hypothesis, we almost surely have
\begin{equation}
    \lim_{T \to \infty} \lim_{n \to \infty} \frac{\| \tbu^{T+\ell} - \hbu^\ell \|^2}{n}  = 0, \quad \text{ and }
    \label{eq:tuTt_ht_conv}
\end{equation}
\begin{align}
\lim_{n \to \infty} \frac{\| \hbu^\ell \|^2}{n} = \lim_{T \to \infty}   \lim_{n \to \infty} \frac{\| \tbu^{T+\ell} \|^2}{n} = 
     \lim_{T \to \infty}
    (\tmu_{T+\ell}^2 + \tsigma_{T+\ell, T+\ell} ) = \mu_\ell^2 + \sigma_{\ell,\ell},
    \label{eq:huTl}
\end{align}
where the last equality is due to \eqref{eq:fake_true_SEconv}. Furthermore, $\tsb_{T+t, T+t}= \kappa_1 \to \kappa_1^\infty = \bar{\sb}_{t,t}$ as $n \to \infty$. For $\ell \in [2, t-1]$, from \eqref{eq:Phase2_onsager} we have
 $   \tsb_{T+t, T+\ell}= \kappa_{t-\ell+1} \prod_{i=\ell+1}^t \<  \su'_{i}(\tbf^{T+i-1})\> $.
Proposition \ref{prop:tilSE} implies that the empirical distribution of $\tbf^{T+i-1}$ converges almost surely in Wasserstein-2 distance to the law of $\tF_{T+i-1} \equiv \tmu_{T+i-1}U_* + \tilde{Z}_{T+i-1}$. Therefore, applying Lemma \ref{lem:lipderiv}, we almost surely have
\begin{align}
    \lim_{n \to  \infty} \tsb_{T+t, T+\ell} = 
    \kappa_{t-\ell+1}^\infty \prod_{i=\ell+1}^t \E\{ \su_i'(\tF_{T+i-1}) \}.
    \label{eq:tsb_tl_nlim}
\end{align}
Since $\tF_{T+i-1}$ converges in distribution to 
$F_{i-1} \equiv \mu_{i-1}U_* + Z_{i-1}$ as $T \to \infty$, applying Lemma \ref{lem:lipderiv} once again, we obtain
\begin{align}
\lim_{T \to \infty} \lim_{n \to  \infty} \tsb_{T+t, T+\ell} = \kappa_{t-\ell+1}^\infty \prod_{i=\ell+1}^t \E\{ \su_i'(F_{i-1}) \}.
    \label{eq:tsb_tl_lim}
\end{align}
Using \eqref{eq:tuTt_ht_conv}, \eqref{eq:huTl} and \eqref{eq:tsb_tl_lim} in \eqref{eq:uTlul_bnd},
we obtain
\begin{align}
       \lim_{T \to \infty} \lim_{n\to \infty} \frac{1}{n}\| \tsb_{T+t, T+\ell} \tbu^{T+\ell} - \bar{\sb}_{t,\ell} \hbu^\ell \|^2 =0 \quad \text{ almost surely  for } \ell \in [2,t]. 
       \label{eq:tsbtu_barsbhu}
\end{align}
To bound the last term in \eqref{eq:fTt_ft_bnd}, we write it as
\begin{align}
    \Big\|  \sum_{i=1}^{T+1} \tsb_{T+t,i} \tbu^i  - \bar{\sb}_{t,1} \bu^1\Big\|^2 = \Big\|  \sum_{j=0}^{T} \tsb_{T+t,T+1-j} \tbu^{T+1-j}  - \bar{\sb}_{t,1} \hbu^1\Big\|^2,
\end{align}
where from \eqref{eq:Phase2_onsager} we have
\begin{equation}
    \tsb_{T+t,T+1-j}= \kappa_{t+j} \alpha^{-j} \prod_{i=2}^t
    \< \su_i'(\tbf^{T+i-1})\>, \qquad 0 \le j \le T.  
\end{equation}
Using this together with the formula for $\bar{\sb}_{t,1}$ in \eqref{eq:modAMP_onsager}, we have 
\begin{align}
    & \frac{1}{n} \Big\|  \sum_{i=1}^{T+1} \tsb_{T+t,i} \tbu^i  - \bar{\sb}_{t,1} \hbu^1\Big\|^2 \nonumber \\
    & = \frac{1}{n} \Big\| \prod_{\ell=2}^t 
    \< \su_\ell'(\tbf^{T+\ell-1})\>  \sum_{j=0}^{T} \kappa_{t+j} \alpha^{-j} \, \tbu^{T+1-j} \,   -  \,
 \prod_{\ell=2}^t \E\{  \su_\ell'(F_{\ell-1}) \} \,  
 \sum_{i=0}^\infty \kappa_{t+i}^\infty \alpha^{-i} \hbu^1 \Big \|^2  \nonumber \\
    & \le 3\Bigg( \frac{1}{n} \Big\| \prod_{\ell=2}^t 
    \< \su_\ell'(\tbf^{T+\ell-1})\>  \sum_{j=0}^{T} \kappa_{t+j} \alpha^{-j} \, \tbu^{T+1-j}
    - \prod_{\ell=2}^t 
    \E\{  \su_\ell'(F_{\ell-1}) \}  \sum_{j=0}^{T} \kappa^\infty_{t+j} \alpha^{-j} \, \tbu^{T+1-j} \Big\|^2  \nonumber \\
    & \quad + \,  \frac{1}{n}\Big\| \prod_{\ell=2}^t 
    \E\{  \su_\ell'(F_{\ell-1}) \}  \sum_{j=0}^{T} \kappa^\infty_{t+j} \alpha^{-j} \, (\tbu^{T+1-j} -\hbu^1) \Big\|^2  +  \nonumber \\
    & \quad  + \frac{1}{n} \Big\| \prod_{\ell=2}^t 
    \E\{  \su_\ell'(F_{\ell-1}) \} \sum_{i=T+1}^\infty \kappa_{t+i}^\infty \alpha^{-i} \hbu^1 \Big\|^2 \Bigg) 
    := 3(S_1 + S_2 + S_3).
    \label{eq:tbusum_hbu1_split}
\end{align}
Considering the second term $S_2$ first, we have
\begin{align}
    & \frac{1}{n} \Big\|  \sum_{j=0}^{T} \kappa^\infty_{t+j} \alpha^{-j} \, (\tbu^{T+1-j} -\hbu^1) \Big\|^2 \nonumber  \\
    & \le 2 \Bigg( \frac{1}{n} \Big\|\sum_{j=0}^{T} \kappa^\infty_{t+j} \alpha^{-j} \, (\tbu^{T+1-j} -\tbu^{T+1}) \Big\|^2 \, + \,  
    \Big( \sum_{j=0}^{T} \kappa^\infty_{t+j} \alpha^{-j}  \Big)^2 \frac{\| \tbu^{T+1} - \hbu^1 \|^2}{n} \Bigg).
\end{align}
By an argument similar to \eqref{eq:four_sigma_split1}-\eqref{eq:tuTj_tuT_conv}, we have
\beq
\lim_{T \to \infty} \lim_{n \to \infty} 
\frac{1}{n} \Big\|\sum_{j=0}^{T} \kappa^\infty_{t+j} \alpha^{-j} \, (\tbu^{T+1-j} -\tbu^{T+1}) \Big\|^2 =0  \quad \text{ almost surely}.
\eeq
Moreover, since $R(1/\alpha) < \infty$, from \eqref{eq:R} we have
\[ 
\lim_{T \to \infty} \sum_{j=0}^{T} \kappa^\infty_{t+j} \alpha^{-j} = \alpha^{t-1}\Big(R(1/\alpha) - \sum_{i=0}^{t-2} \kappa^\infty_{i+1} \alpha^{-i} \Big). \] 
Combining this with  \eqref{eq:ut1_hu1_conv}, we have that almost surely 
\beq
\lim_{T \to \infty} \lim_{n \to \infty}  S_2 =0.
\label{eq:S2_lim}
\eeq
Next consider $S_3$. Since the series $\sum_{j=0}^{\infty} \kappa^\infty_{t+j} \alpha^{-j}$ converges, $ \lim_{T \to 0} \sum_{i=T+1}^\infty \kappa_{t+i}^\infty \alpha^{-i} = 0$. Furthermore, by \eqref{eq:hu1_lim}, $\| \hbu^1 \|^2/n$ converges almost surely to a finite value. Therefore
\beq
\lim_{T \to \infty} \lim_{n \to \infty}  S_3 =0.
\label{eq:S3_lim}
\eeq
Finally, we consider the term $S_1$ in \eqref{eq:tbusum_hbu1_split}. We have
\begin{align}
    S_1 & \leq  2\Big( \prod_{\ell=2}^t \< \su_\ell'(\tbf^{T+\ell-1})\> \Big)^2  \frac{1}{n} \Big\|  \sum_{j=0}^{T}   (\kappa_{t+j} - \kappa^\infty_{t+j}) \alpha^{-j} \, \tbu^{T+1-j}  \Big\|^2 \nonumber  \\
    &   \quad   + \, 2 \Big( \prod_{\ell=2}^t \< \su_\ell'(\tbf^{T+\ell-1})\>
    - \prod_{\ell=2}^t \E\{  \su_\ell'(F_{\ell-1}) \} \Big)^2
   \frac{1}{n} \Big\| \sum_{j=0}^{T} \kappa^\infty_{t+j} \alpha^{-j} \, \tbu^{T+1-j} \Big \|^2.
\end{align}
Proposition \ref{prop:tilSE} implies that for $\ell \in [2,t]$, the empirical distribution of $\tbf^{T+\ell-1}$ converges almost surely in Wasserstein-2 distance to the law of $\tF_{T+\ell-1}$, which converges in distribution to $F_{\ell-1}$
(due to \eqref{eq:fake_true_SEconv}).
Therefore, applying Lemma \ref{lem:lipderiv} twice (as in \eqref{eq:tsb_tl_nlim}-\eqref{eq:tsb_tl_lim}) we  almost surely have
\begin{align}
     \prod_{\ell=2}^t \< \su_\ell'(\tbf_{T+\ell-1})\> =  
     \prod_{\ell=2}^t \E\{  \su_\ell'(F_{\ell-1}) \}.
     \label{eq:prod_u'_Fell}
\end{align}
Next, we have already shown that $\lim_{T \to \infty} \lim_{n \to\infty} \frac{1}{n} \|  \sum_{j=0}^{T}   (\kappa_{t+j} - \kappa^\infty_{t+j}) \alpha^{-j} \, \tbu^{T+1-j}  \|^2 =0$ almost surely. (See \eqref{eq:kappj_kappj_inf}-\eqref{eq:uiuj_inner_prod} and the subsequent argument.)
This, together with \eqref{eq:prod_u'_Fell} implies that that $\lim_{T \to \infty} \lim_{n \to \infty} S_1=0$ almost surely. Thus, using \eqref{eq:S2_lim} and \eqref{eq:S3_lim} in \eqref{eq:tbusum_hbu1_split}, 
we have 
\beq
\lim_{T\to\infty} \lim_{n \to \infty}  \frac{1}{n} \Big\|  \sum_{i=1}^{T+1} \tsb_{T+t,i} \tbu^i  - \bar{\sb}_{t,1} \hbu^1\Big\|^2 =0 \ \text{ almost surely}.
\label{eq:tbusum_hbu1_zero}
\eeq

Using \eqref{eq:XuTthut}, \eqref{eq:tsbtu_barsbhu}, and \eqref{eq:tbusum_hbu1_split}   in \eqref{eq:fTt_ft_bnd}, we conclude
\beq
\lim_{T\to\infty} \lim_{n \to \infty} \frac{1}{n} \| \tbf^{T+t} - \hbf^t \|^2 =0 \quad \text{ almost surely }. 
\label{eq:fTt_ft_lim}
\eeq
Since $\tbu^{T+t+1}= \su_{t+1}( \tbf^{T+t})$ and $\hbu^{t+1} = \su_{t+1}(\hbf^{t})$, with $\su_{t+1}$ Lipschitz, \eqref{eq:fTt_ft_lim} implies that
\beq
\lim_{T\to\infty} \lim_{n \to \infty} \frac{1}{n} \| \tbu^{T+t+1} - \hbu^{t+1} \|^2 =0 \quad \text{ almost surely }. 
\eeq
Using the arguments in \eqref{eq:tbu_triangle}-\eqref{eq:hu2_lim}, we also have almost surely:
\beq
\begin{split}
&\lim_{T\to\infty} \lim_{n \to \infty} \frac{\| \tbf^{t+T} \|^2}{n} =\lim_{n \to \infty} \frac{\| \hbf^{t} \|^2}{n} = \E\{ F_t^2 \},\\
&\lim_{T\to\infty} \lim_{n \to \infty} \frac{\| \tbu^{T+t+1} \|^2}{n} =\lim_{n \to \infty} \frac{\| \hbu^{t+1} \|^2}{n} =  \E\{ \su_{t+1}(F_t)^2 \}.
\end{split}
\eeq
Using these together with the induction hypothesis \eqref{eq:ind_hyp} in \eqref{eq:PL2_psi_bnd} completes the proof that 
\beq
\begin{split}
& \left\vert \frac{1}{n}  \sum_{i=1}^n \psi (u^*_{i},  \tu^{T+1}_i, \ldots, \tu^{T+t+1}_i, \tf^{T+1}_i, \ldots \tf^{T+t}_i)
\, - \, 
\frac{1}{n}  \sum_{i=1}^n \psi (u^*_{i}, \hu^1_i, \ldots, \hu^{t+1}_i, \hf^{1}_i, \ldots \hf^{t}_i)  \right\vert \\
& =0 \quad  \text{ almost surely}.
\end{split}
\eeq
\end{proof}

 \subsection{Proof of Theorem \ref{thm:square}} \label{subsec:app_thm_prof_sq}
We will first use Lemma \ref{lem:sec_phase_sq} to prove that the state evolution result holds for the iterates of the modified AMP, i.e., for $\psi \in \PL(2)$:
\begin{equation}
    \lim_{n \to \infty} \frac{1}{n} \sum_{i=1}^n \psi(u_i^*, \hu_i^1, \ldots, \hu^{t+1}_i, \hf^1_i, \ldots, \hf^t_i)
    = \E\{ \psi(U_*, U_1, \ldots, U_{t+1}, F_1, \ldots, F_t) \}.
    \label{eq:hatAMP_conv}
\end{equation}
Using the triangle inequality, for $T >0$ we have the bound
\begin{equation}
    \begin{split}
      &   \abs{\frac{1}{n} \sum_{i=1}^n \psi(u_i^*, \hu_i^1, \ldots, \hu^{t+1}_i, \hf^1_i, \ldots, \hf^t_i) - \E\{ \psi(U_*, U_1, \ldots, U_{t+1}, F_1, \ldots, F_t) \} } \\
    & \leq \abs{\frac{1}{n} \sum_{i=1}^n \psi(u_i^*, \hu_i^1, \ldots, \hu^{t+1}_i, \hf^1_i, \ldots, \hf^t_i) - 
    \frac{1}{n} \sum_{i=1}^n \psi(u_i^*, \tu_i^{T+1}, \ldots, \tu^{T+t+1}_i, \tf^{T+1}_i, \ldots, \tf^{T+t}_i)}  \\
    &  + \bigg|\frac{1}{n} \sum_{i=1}^n \psi(u_i^*, \tu_{i}^{T+1}, \ldots, \tu^{T+t+1}_i, \tf^{T+1}_i, \ldots, \tf^{T+t}_i) \\
    &\hspace{10em}- \E\{ \psi(U_*, \tU_{T+1}, \ldots, \tU_{T+t+1}, \tF_{T+1}, \ldots, \tF_{T+t}) \}\bigg| \\ 
    &  + \abs{\E\{ \psi(U_*, \tU_{T+1}, \ldots, \tU_{T+t+1}, \tF_{T+1}, \ldots, \tF_{T+t}) \} \, -  \, \E\{ \psi(U_*, U_1, \ldots, U_{t+1}, F_1, \ldots, F_t) \}} \\
    & := S_1 + S_2 +S_3.
        \end{split}
        \label{eq:S1S2S3_sq}
\end{equation}
First consider $S_3$. From \eqref{eq:fake_true_SEconv}, 
$(U_*, \tU_{T+1}, \ldots, \tU_{T+t+1}, \tF_{T+1}, \ldots, \tF_{T+t}))$ converges in distribution to the the law of $(U_*, U_1, \ldots, \su_{t+1}, F_1, \ldots, F_t)$ as $T \to \infty$. By Skorokhod's representation theorem \cite{billingsley2008probability}, to compute the expectations in $S_3$, we can take the sequence of random vectors $(U_*, \tU_{T+1}, \ldots, \tU_{T+t+1}, \tF_{T+1}, \ldots, \tF_{T+t})$ to be such that they belong to the same  probability space and converge almost surely to $(U_*, U_1, \ldots, U_{t+1}, F_1, \ldots, F_t)$ as $T \to \infty$.  
Then, using the pseudo-Lipschitz property of $\psi$ and using Cauchy-Schwarz inequality (twice, as in \eqref{eq:PL2_psi_bnd}), we obtain 
\begin{equation}
\begin{split}
        S_3 &  \le 2C(t+2)\left( 2 + \sum_{\ell=1}^{t+1} (\E\{ \tU_{T+\ell}^2 \} + \E\{ U_{\ell}^2 \} ) +  \sum_{\ell=1}^{t} (\E\{ \tF_{T+\ell}^2 \} + 
        \E\{ F_{\ell}^2 \} ) \right)^{1/2} \\
    & \quad \cdot  \left( \sum_{\ell=1}^{t+1} \E \big\{ (\tU_{T+\ell} - U_{\ell})^2 \big\} \,  + \,  \sum_{\ell=1}^{t} \E \big\{ (\tF_{T+\ell} - F_{\ell})^2 \big\}   \right)^{1/2}.
\end{split}
\label{eq:S3_sq_bnd}
\end{equation}
From Lemma \ref{lem:sec_phase_sq}, we have $ \lim_{T \to \infty} \E \{ \tF_{T+\ell}^2 \} =  \E\{ F_{\ell}^2 \}$ and $\lim_{T \to \infty} \E \{ \tU_{T+\ell}^2 \} =  \E\{ U_{\ell}^2 \}$. Moreover, since for each $\ell$, 
\begin{align}
    \E \big\{ (\tF_{T+\ell} - F_{\ell})^2 \big\} \leq 2 \E \big\{ \tF_{T+\ell}^2 \} + 2 \E \big\{ F_{\ell}^2 \} < \infty \quad \forall \,T,
\end{align}
by dominated convergence we have $\lim_{T \to \infty} \E \{ (\tU_{T+\ell} - U_{\ell})^2 \} = \lim_{T \to \infty} \E \{ (\tF_{T+\ell} - F_{\ell})^2 \} =0$.  Therefore $\lim_{T \to \infty} S_3=0$. Furthermore,  by Lemma \ref{lem:sec_phase_sq} and Proposition \ref{prop:tilSE}, we also have $\lim_{T \infty} \lim_{n \to \infty} S_1 = \lim_{T \to \infty} \lim_{n \to \infty} S_2=0$ almost surely. This proves the state evolution result \eqref{eq:hatAMP_conv} for the modified AMP. 

We now prove the result of Theorem \ref{thm:square} by showing that for $t \ge 1$, almost surely:
\begin{align}
    & \lim_{n \to \infty} \abs{\frac{1}{n} \sum_{i=1}^n \psi(u_i^*, u_i^1, \ldots, u^{t+1}_i, f^1_i, \ldots, f^t_i) - 
\frac{1}{n} \sum_{i=1}^n \psi(u_i^*, \hu_i^1, \ldots, \hu^{t+1}_i, \hf^1_i, \ldots, \hf^t_i)} =0, \label{eq:lim_bubf_hbuhbf_joint} \\
 & \lim_{n \to \infty} \frac{\| \bdff^{t} - \hbf^{t} \|^2}{n}=0, \quad  \lim_{n \to \infty} \frac{\| \bu^{t+1} - \hbu^{t+1} \|^2}{n}=0.
 \label{eq:lim_bubf_hbuhbf_sep}
\end{align}
The proof of \eqref{eq:lim_bubf_hbuhbf_joint}-\eqref{eq:lim_bubf_hbuhbf_sep} is by induction and similar to that of \eqref{eq:fake_modified_conv}. Noting that $\bu^1=\hbu^{1} = \sqrt{n} \bu_{\rm PCA}$, assume towards induction that \eqref{eq:lim_bubf_hbuhbf_joint}-\eqref{eq:lim_bubf_hbuhbf_sep} hold with $t$ replaced by $t-1$. Since $\psi \in \PL(2)$, by the same arguments as in \eqref{eq:PL2_psi_bnd} we have 
\begin{align}
& 
\left\vert \frac{1}{n}  \sum_{i=1}^n \psi (u^*_{i},  u^{1}_i, \ldots, u^{t+1}_i, f^1_i, \ldots f^{t}_i)
\, - \, 
\frac{1}{n}  \sum_{i=1}^n \psi (u^*_{i}, \hu^1_i, \ldots, \hu^{t+1}_i, \hf^1_i, \ldots \hf^{t}_i)  \right\vert  \nonumber \\
  & \le 2C(t+2) \left[ 1 + \frac{\| \bu^* \|^2}{n}  +  
 \sum_{\ell=1}^{t+1} \Big( \frac{\| \bu^{\ell} \|^2}{n} +
 \frac{\| \hbu^{\ell} \|^2}{n} \Big) + 
 \sum_{\ell=1}^{t} \Big( \frac{\| \bdff^{\ell} \|^2}{n}
 + \frac{\| \hbf^{\ell} \|^2}{n} \Big) \right]^{\frac{1}{2}} \nonumber  \\ 
 &  \quad \cdot \left( \frac{\| \hbu^{1} - \bu^1 \|^2}{n}+ \ldots + \frac{\| \bu^{t+1} - \hbu^{t+1} \|^2}{n}  +\frac{\| \bdff^{1} - \hbf^1  \|^2}{n} + \ldots  + \frac{\| \bdff^{t} - \hbf^t  \|^2}{n}\right)^{\frac{1}{2}}. \label{eq:PL2_uf_hbuf_bnd}
\end{align}
 Using the definitions of $\bdff^{t}$ and $\hbf^t$ in \eqref{eq:AMP} and \eqref{eq:modAMP}, and applying the Cauchy-Schwarz inequality, we have
\begin{align}
   &  \frac{1}{n}\| \bdff^{t} - \hbf^{t} \|^2 \leq 
   \frac{(t+1)}{n}\Big( \| \bX (\bu^{t} - \hbu^t) \|^2 + 
    \sum_{\ell=1}^t \| \sb_{t,\ell} \bu^{\ell} - \bar{\sb}_{t,\ell} \hbu^\ell \|^2 \Big) \nonumber \\
    & \leq (t+1)\Big( \| \bX \|^2_{\rm op} \,  \frac{1}{n} \|\bu^{t} - \hbu^t \|^2 +  \sum_{\ell=1}^t 
    \frac{2}{n} \| \sb_{t,\ell} \bu^{\ell} - \bar{\sb}_{t,\ell} \bu^{\ell} \|^2 
    + \frac{2}{n} \| \bar{\sb}_{t,\ell} \bu^{\ell} - \bar{\sb}_{t,\ell} \hbu^{\ell} \|^2 \Big).
   \label{eq:ft_hft_bnd}
\end{align}
Recall that $\| \bX \|_{\rm op}$ converges almost surely to $|G^{-1}(1/\alpha)|$ and by the induction hypothesis, $\frac{1}{n} \|\bu^{\ell} - \hbu^\ell \|^2 \to 0$, for $\ell \in [1, t]$. Next, we note that $\sb_{t, t}= \kappa_1 \to \kappa_1^\infty = \bar{\sb}_{t,t}$ as $n \to \infty$. For $\ell \in [2, t-1]$,  we have
 $   \sb_{t, \ell}= \kappa_{t-\ell+1} \prod_{i=\ell+1}^t \<  \su'_{i}(\bdff^{i-1})\> $.
The induction hypothesis \eqref{eq:lim_bubf_hbuhbf_joint} implies that the empirical distribution of $\bdff^{i-1}$ converges almost surely in Wasserstein-2 distance to the law of $F_{i-1}$ for $i \in [1, t]$. Therefore, applying Lemma \ref{lem:lipderiv} we almost surely have
\begin{align}
    \lim_{n \to  \infty} \sb_{t,\ell} = 
    \kappa_{t-\ell+1}^\infty \prod_{i=\ell+1}^t \E\{ \su_i'(F_{i-1}) \}.
\end{align}
This shows that $\lim_{n \to \infty} \frac{1}{n}\| \bdff^{t} - \hbf^{t} \|^2=0$ almost surely. Since $\bu^{t+1} = \su_{t+1}(\bdff^t)$ with $\su_{t+1}$ Lipschitz, we also have $\lim_{n \to \infty} \frac{1}{n}\| \bu^{t+1} - \hbu^{t+1} \|^2=0$ almost surely. Moreover using a triangle inequality argument similar to \eqref{eq:tbu_triangle}, for $\ell \in [1, t]$, we almost surely have 
\beq
\lim_{n \to \infty} \frac{\| \bdff^\ell \|^2}{n} = \lim_{n \to \infty} \frac{\| \hbf^\ell \|^2}{n} = \E\{ F_\ell^2  \}, \qquad 
\lim_{n \to \infty} \frac{\| \bu^{\ell+1} \|^2}{n} = \lim_{n \to \infty} \frac{\| \hbu^{\ell+1} \|^2}{n} = \E\{ \su_{\ell+1}(F_\ell)^2  \}.
\label{eq:fl_hfl_ul_hul_sq}
\eeq
Using this in  \eqref{eq:PL2_uf_hbuf_bnd}, we conclude that
\begin{equation}
  \lim_{n \to \infty}  \left\vert \frac{1}{n}  \sum_{i=1}^n \psi (u^*_{i},  u^{1}_i, \ldots, u^{t+1}_i, f^1_i, \ldots f^{t}_i)
\, - \, 
\frac{1}{n}  \sum_{i=1}^n \psi (u^*_{i}, \hu^1_i, \ldots, \hu^{t+1}_i, \hf^1_i, \ldots \hf^{t}_i)  \right\vert =0, 
\end{equation}
which combined with \eqref{eq:hatAMP_conv} completes the proof of the theorem. \qed

\section{Proof of Theorem \ref{thm:rect}}\label{app:pfrect}

This appendix is organized as follows. In Appendix \ref{subsec:proof_sketch_rect}, we present the artificial AMP for the rectangular model \eqref{eq:defrect}, and  provide a sketch of the proof. In Appendix \ref{app:serect}, we present the state evolution recursion associated with the artificial AMP iteration.
In Appendix \ref{app:fpsefirstrect}, we prove that the first phase of this state evolution admits a unique fixed point. Using this fact, in Appendix \ref{app:firstestrect}, we prove that the artificial AMP iterate at the end of the first phase approaches the left singular vector produced by PCA. Then, in Appendix  \ref{sec:app_sec_phase_analysis_rect}, we show that \emph{(i)} the iterates in the second phase of the artificial AMP are close to the true AMP iterates, and \emph{(ii)} the related state evolutions also remain close. Finally, in Appendix \ref{subsec:app_thm_prof_rect}, we give the proof of Theorem \ref{thm:rect}.

\subsection{Proof Sketch} \label{subsec:proof_sketch_rect}

\paragraph{First phase.} We consider the following artificial AMP algorithm. We initialize with 
\begin{equation}
    \begin{split}
        \tilde{\bu}^1 = \sqrt{\Delta_{\rm PCA}} \bu^* + \sqrt{1-\Delta_{\rm PCA}}\bn,  \ \quad \  \tilde{\bg}^1= \bX^{\sT} \tbu^{1}, \ \quad \  \tbv^1 = 
        \frac{\gamma}{\alpha} \tilde{\bg}^1, \ \quad \ \tbf^1 = \bX \tbv^1 - \kappa_2 \frac{\gamma}{\alpha} \tbu^1. 
    \end{split}
    \label{eq:AMPfake_init_rect}
\end{equation} 
Here, $\bn$ has i.i.d. standard Gaussian components and $\Delta_{\rm PCA}$ is the (limiting) normalized  squared correlation of the left PCA estimate, given in \eqref{eq:DeltaPCA_def}. As in the square case, the initialization of the artificial AMP is impractical. However, this is not a problem, as the artificial AMP is only used as a proof technique.
Then, for $2 \le t\le T+1$, the artificial AMP iterates are
\begin{equation}\label{eq:AMPfake1rect}
\begin{split}
   & \tilde{\bu}^{t} = \frac{1}{\alpha}\tilde{\bdff}^{t-1}, \qquad 
      \tilde{\bg}^t = \bX^\sT \tilde{\bu}^t - \sum_{i=1}^{t-1} \tilde{\sb}_{t,i}\tilde{\bv}^i, \\
   & \tilde{\bv}^{t} = \frac{\gamma}{\alpha}\tilde{\bg}^{t}, \qquad \tilde{\bdff}^t = \bX \tilde{\bv}^t - \sum_{i=1}^t \tilde{\sa}_{t,i}\tilde{\bu}^i,
\end{split}
\end{equation}
where $\tilde{\sb}_{t,t-j}=\kappa_{2j}\frac{\gamma}{\alpha}\left(\frac{\gamma}{\alpha^2}\right)^{j-1}$ for $j\in [1, t-1]$, and $\tilde{\sa}_{t,t-j}=\kappa_{2(j+1)}\frac{\gamma}{\alpha}\left(\frac{\gamma}{\alpha^2}\right)^{j}$ for $j\in [0, t-1]$. We claim that, for sufficiently large $T$, $\tilde{\bu}^{T+1}$ approaches the left PCA estimate $\bu_{\rm PCA}$, that is,
$ \lim_{T\to\infty}\lim_{n\to\infty} \frac{1}{\sqrt{m}} \|\tilde{\bu}^{T+1}-\sqrt{m}\bu_{\rm PCA}\| \, = \, 0.  $
This result is proved in Lemma \ref{lemma:convrect} in Appendix \ref{app:firstestrect}. Here we give a heuristic sanity check. Assume that the iterates $\tilde{\bu}^{T+1}$ and $\tilde{\bv}^{T+1}$ converge to the limits $\tilde{\bu}^\infty$ and $\tilde{\bv}^\infty$, respectively, in the sense that 
$ \lim_{T \to \infty} \lim_{n\to \infty}\frac{1}{\sqrt{m}}\|\tilde{\bu}^{T+1} - \tilde{\bu}^\infty \| =0$ and $ \lim_{T \to \infty} \lim_{n\to \infty}\frac{1}{\sqrt{n}}\|\tilde{\bv}^{T+1} - \tilde{\bv}^\infty \| =0$. Then, from \eqref{eq:AMPfake1rect}, the limits  $\tilde{\bu}^\infty$ and $\tilde{\bv}^\infty$ satisfy
\begin{equation}\label{eq:heur}
\begin{split}
    \tilde{\bu}^\infty &= \frac{1}{\alpha} \bX \tilde{\bv}^\infty - \sum_{i=1}^\infty\kappa_{2i} \left(\frac{\gamma}{\alpha^2}\right)^{i}\tilde{\bu}^\infty ,\\
    \tilde{\bv}^\infty & = \frac{\gamma}{\alpha}\bX^\sT\tilde{\bu}^\infty-\gamma\sum_{i=1}^\infty\kappa_{2i} \left(\frac{\gamma}{\alpha^2}\right)^{i} \tilde{\bv}^\infty.
\end{split}
\end{equation}
By using \eqref{eq:Rrect}, we can re-write \eqref{eq:heur} as
\begin{equation}
\begin{split}
    \left(1+R\left(\frac{\gamma}{\alpha^2}\right)\right)\tilde{\bu}^\infty &= \frac{1}{\alpha} \bX \tilde{\bv}^\infty ,\\
    \left(1+\gamma R\left(\frac{\gamma}{\alpha^2}\right)\right)\tilde{\bv}^\infty & = \frac{\gamma}{\alpha}\bX^\sT\tilde{\bu}^\infty,
\end{split}
\end{equation}
which leads to
\begin{equation}
\left(1+\gamma R\left(\frac{\gamma}{\alpha^2}\right)\right)    \left(1+R\left(\frac{\gamma}{\alpha^2}\right)\right)\tilde{\bu}^\infty=\frac{\gamma}{\alpha^2}\bX\bX^\sT\tilde{\bu}^\infty.
\end{equation}
As a result, $\tilde{\bu}^\infty$ is an eigenvector of $\bX\bX^\sT$. Furthermore, by using \eqref{eq:RD}, the eigenvalue $$\frac{\alpha^2}{\gamma}\left(1+\gamma R\left(\frac{\gamma}{\alpha^2}\right)\right)    \left(1+R\left(\frac{\gamma}{\alpha^2}\right)\right)$$ can be re-written as $$\left(D^{-1}\left(\frac{\gamma}{\alpha^2}\right)\right)^2.$$ Recall that, for $\tilde{\alpha}>\tilde{\alpha}_{\rm s}$, $\bX$ exhibits a spectral gap and its largest singular value converges to $D^{-1}\left(\frac{\gamma}{\alpha^2}\right)$. Thus, $\bu^\infty$ must be aligned with the left principal singular vector of $\bX$, as desired. 

A key step in our analysis is to show that, as $T\to\infty$, the state evolution of the artificial AMP in the first phase has a unique fixed point. This is established in Lemma \ref{lem:SE_FP_phase1rect},  proved in Appendix \ref{app:fpsefirstrect}. As for the square case, we follow the approach of \cite[Section 7]{fan2020approximate}. The crucial difference with \cite{fan2020approximate} is that we provide a result for all $\tilde{\alpha}>\tilde{\alpha}_{\rm s}$, while the analysis of \cite{fan2020approximate} requires that $\tilde{\alpha}$ is sufficiently large. To achieve this goal, we exploit the expression \eqref{eq:DeltaPCA_def} of the limit correlation between $\bu_{\rm PCA}$ and $\bu^*$, and  show that, as soon as the left PCA estimate is correlated with the signal $\bu^*$, state evolution is close to a limit map which is a contraction. For this approach to work, we need the rectangular free cumulants to be non-negative. 

\paragraph{Second phase.} The second phase is designed so that the iterates $(\tbg^{T+k}, \tbf^{T+k})$ are close to $(\bg^{k}, \bdff^{k})$, for $k \ge 2$. For $t \ge (T+2)$, the artificial AMP computes
\begin{equation}
    \begin{split}
            & \tbu^{t} = \su_{t-T}(\tbf^{t-1}), \qquad \tilde{\bg}^{t} = \bX^{\sT} \tbu^{t} - \sum_{i=1}^{t-1} \tsb_{t, i} \tbv^{i}, \\
            & \tbv^{t} = \sv_{t-T}(\tbg^t), \qquad  \tbf^t = \bX \tbv^t - \sum_{i=1}^t \tsa_{t,i} \tbu^t.
     \end{split}
    \label{eq:AMPfake2rect}
\end{equation}
Here, the functions $\{ v_k, u_{k} \}_{k \ge 2}$ are the ones used in the true AMP \eqref{eq:AMP_rectg}. Additionally, letting $u_1(x) = x/\alpha$ and $v_1(x) = \gamma x /\alpha$,  the coefficients $\{ \tsa_{t,i} \}$ and $\{ \tsb_{t,i} \}$ are given by:
\begin{align}
    & \tsa_{t, t-j} = \kappa_{2(j+1)}\< \sv'_{t-T}(\tbg^t) \> 
    \left( \frac{\gamma}{\alpha^2}\right)^{(T+1- (t-j))_+}
    \hspace{-10pt}
    \prod_{i=\max\{ t-j+1, \, T+2 \}} ^t \< \su'_{i-T}(\tbf^{i-1}) \> \<\sv'_{i-1-T}(\tbg^{i-1}) \>, \nonumber  \\
    & \hspace{3.8in} (t-j) \in [1, t], \label{eq:tsa_def} \\
    & \tsb_{t,t-j} = \gamma \kappa_{2j} \< \su'_{t-T}(\tbf^{t-1}) \> 
    \left( \frac{\gamma}{\alpha^2} \right)^{(T- (t-j))_+} 
    \prod_{i=\max\{t-j+1,\,  T+1\}}^{t-1} \< \sv'_{i-T}(\tbg^{i}) \>
    \< \su'_{i-T}(\tbf^{i-1}) \>, \nonumber \\
    & \hspace{3.6in} (t-j) \in [1, t-1].
\end{align}
Since the artificial AMP is initialized with $\tbu^1$ that is correlated with $\bu^*$ and independent of the noise matrix $\bW$, a state evolution result for it can be obtained directly from \cite[Theorem 1.4]{fan2020approximate}. We then show in Lemma \ref{lem:sec_phase_rect} in Appendix \ref{sec:app_sec_phase_analysis_rect} that  the second phase iterates in \eqref{eq:AMPfake2rect} are close to the true AMP iterates in \eqref{eq:AMP_rectg}, and that their state evolution parameters are also close. This result yields Theorem \ref{thm:rect}, as shown in Appendix \ref{subsec:app_thm_prof_rect}.

\subsection{State Evolution for the Artificial AMP}\label{app:serect}

Consider the artificial AMP iteration defined in \eqref{eq:AMPfake1rect} and \eqref{eq:AMPfake2rect}, with initialization 
$\tilde{\bu}^1 = \sqrt{\Delta_{\rm PCA}} \bu^* + \sqrt{1-\Delta_{\rm PCA}}\bn$. Then, its associated state evolution recursion is expressed in terms of a sequence of mean vectors $\tilde{\bmu}_K=(\tilde{\mu}_t)_{t\in [0, K]}$, $\tilde{\bnu}_K=(\tilde{\nu}_t)_{t\in [1, K]}$ and covariance matrices $\tilde{\bSigma}_{K}=(\tilde{\sigma}_{s, t})_{s, t\in [0, K]}$, $\tilde{\bOmega}_{K}=(\tilde{\Omega}_{s, t})_{s, t\in [1, K]}$ defined recursively as follows. We initialize with 
\begin{equation}\label{eq:SEfakeinitrect}
\begin{split}
   \tmu_0 = \alpha \sqrt{\Delta_{\rm PCA}}, \qquad  \tilde{\sigma}_{0, 0}&=\alpha^2(1-\Delta_{\rm PCA}), \quad \tilde{\sigma}_{0, t}= \tilde{\sigma}_{t, 0}=0, \quad \mbox{ for }t\ge 1.
\end{split}
\end{equation}
Given $\tilde{\bmu}_K, \tilde{\bSigma}_K, \tilde{\bnu}_K, \tilde{\bOmega}_K$, let
\begin{align}
  &  (\tF_0, \ldots, \tF_K) = \tilde{\bmu}_K U_* +  (\tY_0, \ldots, \tY_K), \text{ where }    (\tY_0, \ldots, \tY_K) \sim \normal(\bzero, \tilde{\bSigma}_K),    \label{eq:F0FKtilde_rect}\\
  & \tU_t = \tsu_t(\tF_{t-1}) \ \text{ where } \  \tsu_t(x) = \begin{cases}
  x/\alpha, & 1 \leq t \leq (T+1), \\
  \su_{t-T}(x), & t \ge T+2,\end{cases}  \label{eq:Utilde_rect}  \\
    &  ( \tG_1, \ldots, \tG_K) = \tilde{\bnu}_K V_* +  (\tZ_1, \ldots, \tZ_K), \text{ where }    (\tZ_1, \ldots, \tZ_K) \sim \normal(\bzero, \tilde{\bOmega}_K),   \label{eq:G0GKtilde_rect}\\
     & \tV_t = \tilde{v}_t(\tG_t) \ \text{ where } \ \tilde{v}_t(x) = 
     \begin{cases}
  \gamma x/\alpha, & 1 \leq t \leq T+1, \\
  \sv_{t-T}(x), & t \ge T+2.\end{cases} \label{eq:Vtilde_rect}
\end{align}
Given $\tilde{\bmu}_K$ and $\tilde{\bSigma}_K$, the entries of $\tilde{\bnu}_{K+1}$ are given by $\tilde{\nu}_t = \alpha \E\{ \tilde{U}_t U_* \}$ (for $t \in [1, K+1]$), and the entries of $\tilde{\bOmega}_{K+1}$ (for $s+1,t+1 \in[1, K+1]$) are given by
\begin{equation}
   \begin{split}
      &   \tilde{\omega}_{s+1, t+1}  = \gamma \sum_{j=0}^s \sum_{k=0}^t \hspace{-1pt}
 \Big( \hspace{-5pt} \prod_{i=s-j+2}^{s+1}  \hspace{-6pt} \E\{ \tsu_i'(\tF_{i-1})\} \E\{ \tsv_{i-1}'(\tG_{i-1})\} \Big) 
  \Big( \hspace{-5pt} \prod_{i=t-k+2}^{t+1} \hspace{-6pt}   \E\{ \tsu_i'(\tF_{i-1})\} \E\{ \tsv_{i-1}'(\tG_{i-1})\}\Big)   \\
   &\ \Big[ \kappa_{2(j+ k+1)}^\infty \E\{ \tU_{s+1-j} \tU_{t+1-k} \}  +  \kappa_{2(j+k+2)}^\infty \E\{ \tsu'_{s+1-j}(\tF_{s-j}) \} \E\{ \tsu'_{t+1-k}(\tF_{t-k}) \} 
   \E\{ \tV_{s-j} \tV_{t-k} \} \Big].
 \end{split}
 \label{eq:tomega_update}
\end{equation}
(We use the convention that $\tV_0=0$.)
Next,  given $\tilde{\bnu}_{K+1}$ and $\tilde{\bOmega}_{K+1}$ for some $K \ge 1$, the entries of $\tilde{\bmu}_{K+1}$ are given by $\tilde{\mu}_t = \frac{\alpha}{\gamma} \E\{ \tV_t V_* \}$ (for $t \in [0, K+1]$), and the entries of $\tilde{\bSigma}_{K+1}$ (for $s,t \in [0, K+1]$) are given by
\begin{align}
    \tsigma_{s,t} & = \sum_{j=0}^{s-1} \sum_{k=0}^{t-1}
    \Big( \prod_{i=s-j+1}^{s} \E\{ \tsu_i'(\tF_{i-1})\} \E\{ \tv_{i}'(\tG_{i})\}\Big) 
  \Big( \prod_{i=t-k+1}^{t} \E\{ \tsu_i'(\tF_{i-1})\} \E\{ \tv_{i}'(\tG_{i})\}\Big) \nonumber  \\
    &   \cdot \Big[ \kappa_{2(j+ k+1)}^\infty \E\{ \tV_{s-j} \tV_{t-k} \} +  \kappa_{2(j+k+2)}^\infty \E\{ \tsv'_{s-j}(\tG_{s-j}) \} \E\{ \tsv'_{t-k}(\tG_{t-k}) \}  \E\{ \tU_{s-j} \tU_{t-k} \} \Big].
    \label{eq:tsigma_update}
\end{align}

\begin{proposition}[State evolution for artificial AMP]
\label{prop:tilSErect}
Consider the setting of Theorem \ref{thm:rect}, the artificial AMP iteration described in  \eqref{eq:AMPfake1rect} and \eqref{eq:AMPfake2rect}, with initialization given by \eqref{eq:AMPfake_init_rect}, and the corresponding state evolution parameters defined in \eqref{eq:SEfakeinitrect}-\eqref{eq:tsigma_update}.

Then,
for $t \geq 1$ and any  \PL($2$) functions $\psi:\reals^{2t+2} \to\reals$ and $\varphi: \reals^{2t+1} \to \reals$, the following hold almost surely:
\begin{align}
\lim_{m \to \infty} \frac{1}{m} \sum_{i=1}^m 
 \psi (u^*_{i}, \tu^1_i, \ldots, \tu^{t+1}_i, \tf^1_i, \ldots \tf^{t}_i)
 = \E \left\{ \psi(U_*, \tU_1, \ldots, \tU_{t+1}, \tF_1, \ldots, \tF_t) \right\},  \\
 \lim_{n \to \infty} \frac{1}{n} \sum_{i=1}^n \varphi (v^*_{i}, \tv^1_i, \ldots, \tv^t_i, \tg^1_i, \ldots \tg^{t}_i)
 = \E \left\{ \varphi(V_*,\tV_1, \ldots, \tV_t, \tG_1, \ldots, \tG_t) \right\}.
\end{align}
\end{proposition}

The proposition follows directly from Theorem 1.4 in \cite{fan2020approximate} since the initialization $\tilde{\bu}^1$ of the artificial AMP is independent of $\bW$.

\subsection{Fixed Point of State Evolution for the First Phase}\label{app:fpsefirstrect}

From \eqref{eq:SEfakeinitrect}-\eqref{eq:tsigma_update}, we note that the state evolution recursion for the first phase $(t \in [1, T+1])$ has the following form:
\begin{equation}\label{eq:SEfake1rect}
\begin{split}
    \tilde{\mu}_t &= \tilde{\nu}_t=\alpha \sqrt{\Delta_{\rm PCA}}, \quad \mbox{ for }t\in [1, T+1],\\
    \tilde{\sigma}_{s, t} &= \sum_{j=0}^{s-1} \sum_{k=0}^{t-1}  \left(\frac{\gamma}{\alpha^2}\right)^{j+k}\bigg(\kappa_{2(j+k+1)}^\infty\left(\frac{\gamma}{\alpha}\right)^2(\alpha^2\Delta_{\rm PCA}+\tilde{\omega}_{s-j, t-k})\\
    &\hspace{3em}+\kappa_{2(j+k+2)}^\infty\left(\frac{\gamma}{\alpha^2}\right)^2(\alpha^2\Delta_{\rm PCA}+\tilde{\sigma}_{s-j-1, t-k-1})\bigg), \quad \mbox{ for }s, t\in [1, T+1].\\
    \tilde{\omega}_{s, t} &= \gamma\sum_{j=0}^{s-1} \sum_{k=0}^{t-1}  \left(\frac{\gamma}{\alpha^2}\right)^{j+k}\bigg(\kappa_{2(j+k+1)}^\infty\frac{1}{\alpha^2}(\alpha^2\Delta_{\rm PCA}+\tilde{\sigma}_{s-j-1, t-k-1})\\
    &\hspace{3em}+\kappa_{2(j+k+2)}^\infty\left(\frac{\gamma}{\alpha^2}\right)^2(\alpha^2\Delta_{\rm PCA}+\tilde{\omega}_{s-j-1, t-k-1})\bigg), \quad \mbox{ for }s, t\in [1, T+1].
\end{split}
\end{equation}

In this section, we prove the following result characterizing the fixed point of state evolution for the first phase in the rectangular setting.

\begin{lemma}[Fixed point of state evolution for first phase -- Rectangular matrices] Consider the setting of Theorem \ref{thm:rect}, and the state evolution recursion for the first phase given by \eqref{eq:SEfake1rect}. Assume that $\kappa_{2i}^\infty\ge 0$ for all $i\ge 2$, and that $\tilde{\alpha}>\tilde{\alpha}_{\rm s}$. Pick any $\xi<1$ such that $\tilde{\alpha}\sqrt{\xi}>\tilde{\alpha}_{\rm s}$. Then, \begin{equation}
\begin{split}
    \lim_{T\to\infty}\max_{s, t\in [0, T]}&\xi^{\max(s, t)}|\tilde{\sigma}_{T+1-s, T+1-t}-a^*|=0,\\
\lim_{T\to\infty}\max_{s, t\in [0, T]}&\xi^{\max(s, t)}|\tilde{\omega}_{T+1-s, T+1-t}-b^*|=0,    
\end{split}
\end{equation}
where
\begin{equation}\label{eq:asbs}
\begin{split}
a^*&= \alpha^2(1-\Delta_{\rm PCA}),\\
b^*&= \frac{\Delta_{\PCA} \gamma\alpha^2 (x R'(x) - R(x))  \, + \, \gamma R'(x)}{1 + \gamma R(x) - \gamma x R'(x)},\  \text{ with }  \ x=\frac{\gamma}{\alpha^2}.
\end{split}
\end{equation}
\label{lem:SE_FP_phase1rect}
\end{lemma}

As for the case of square matrices, we consider the space of infinite matrices $\bx=(x_{s, t} : s, t\le 0)$ equipped with the weighted $\ell_\infty$-norm defined in  \eqref{eq:norm}. Let $\mathcal X=\{\bx : \|\bx\|_\xi<\infty\}$ and, for any compact set $I\subset \mathbb R$, define $\mathcal X_I$ as in \eqref{eq:defXI}. Recall that both $\mathcal X$ and $\mathcal X_I$ are complete under $\|\cdot\|_\xi$. We embed the matrices $\btSigma_{\bar{T}}, \btOmega_{\bar{T}}$ as elements $\bx, \by\in\mathcal X$ with the following coordinate identification:
\begin{equation*}
    \begin{split}
        \tilde{\sigma}_{s, t}&=x_{s-\bar{T}, t-\bar{T}},\quad \tilde{\omega}_{s, t}=y_{s-\bar{T}, t-\bar{T}},\\
        x_{s, t}&=0, \quad y_{s, t}=0,\quad \mbox{ if }s<-\bar{T} \mbox{ or }t<-\bar{T}
    \end{split}
\end{equation*}
The idea is to approximate the maps $(\btSigma_{\bar{T}-1}, \btOmega_{\bar{T}-1})\mapsto \btOmega_{\bar{T}}$ and $(\btSigma_{\bar{T}-1}, \btOmega_{\bar{T}})\mapsto \btSigma_{\bar{T}}$ with the \emph{fixed limit} maps $h^{\Sigma}$ and $h^\Omega$, respectively, which are defined as
\begin{equation}\label{eq:fixedmaprect}
\begin{split}
    h_{s, t}^{\Omega}(\bx, \by) &= \gamma\sum_{j=0}^{\infty} \sum_{k=0}^{\infty}  \left(\frac{\gamma}{\alpha^2}\right)^{j+k}\bigg(\kappa_{2(j+k+1)}^\infty\frac{1}{\alpha^2}(\alpha^2\Delta_{\rm PCA}+x_{s-j, t-k})\\
    &\hspace{3em}+\kappa_{2(j+k+2)}^\infty\left(\frac{\gamma}{\alpha^2}\right)^2(\alpha^2\Delta_{\rm PCA}+y_{s-j, t-k})\bigg),\\
    h_{s, t}^{\Sigma}(\bx, \by) &= \sum_{j=0}^{\infty} \sum_{k=0}^{\infty}  \left(\frac{\gamma}{\alpha^2}\right)^{j+k}\bigg(\kappa_{2(j+k+1)}^\infty\left(\frac{\gamma}{\alpha}\right)^2(\alpha^2\Delta_{\rm PCA}+y_{s-j, t-k})\\
    &\hspace{3em}+\kappa_{2(j+k+2)}^\infty\left(\frac{\gamma}{\alpha^2}\right)^2(\alpha^2\Delta_{\rm PCA}+x_{s-j, t-k})\bigg).
    \end{split}
\end{equation}

First, we show that $(h^\Omega(\mathcal X_{I_\Sigma^*}, \mathcal X_{I_\Omega^*}), h^\Sigma(\mathcal X_{I_\Sigma^*}, \mathcal X_{I_\Omega^*}))\subseteq (\mathcal X_{I_\Omega^*}, \mathcal X_{I_\Sigma^*})$ for suitably defined compact sets $I_\Omega^*, I_\Sigma^*$.

\begin{lemma}[Image of limit maps -- Rectangular matrices]\label{lemma:imagerect}
    Consider the maps $h^\Omega, h^{\Sigma}$ defined in \eqref{eq:fixedmaprect}. Assume that $\kappa_{2i}^\infty\ge 0$ for all $i\ge 1$, and that $\tilde{\alpha}>\tilde{\alpha}_{\rm s}$. Then, there exist $I_\Omega^*=[-a_\Omega, a_\Omega]$ and $I_\Sigma^*=[-a_\Sigma, a_\Sigma]$ such that, if $(\bx, \by)\in \mathcal X_{I_\Sigma^*}\times\mathcal X_{I^*_{\Omega}}$, then $(h^\Omega(\bx, \by), h^\Sigma(\bx, \by))\in\mathcal X_{I^*_{\Omega}}\times \mathcal X_{I_\Sigma^*}$.
\end{lemma}

\begin{proof}
Let $(\bx, \by)\in \mathcal X_{I_\Sigma^*}\times\mathcal X_{I^*_{\Omega}}$. Then, the following chain of inequalities holds:
\begin{equation}\label{eq:hOm}
    \begin{split}
        |h_{s, t}^\Omega(\bx, \by)| &\stackrel{\mathclap{\mbox{\footnotesize (a)}}}{\le} \gamma\sum_{j=0}^{\infty} \sum_{k=0}^{\infty}  \left(\frac{\gamma}{\alpha^2}\right)^{j+k}\bigg(\kappa_{2(j+k+1)}^\infty\frac{1}{\alpha^2}(\alpha^2\Delta_{\rm PCA}+|x_{s-j, t-k}|)\\
    &\hspace{3em}+\kappa_{2(j+k+2)}^\infty\left(\frac{\gamma}{\alpha^2}\right)^2(\alpha^2\Delta_{\rm PCA}+|y_{s-j, t-k}|)\bigg)\\
        &\stackrel{\mathclap{\mbox{\footnotesize (b)}}}{\le} \gamma\sum_{j=0}^{\infty} \sum_{k=0}^{\infty}  \left(\frac{\gamma}{\alpha^2}\right)^{j+k}\bigg(\kappa_{2(j+k+1)}^\infty\frac{1}{\alpha^2}(\alpha^2\Delta_{\rm PCA}+a_\Sigma)\\
    &\hspace{3em}+\kappa_{2(j+k+2)}^\infty\left(\frac{\gamma}{\alpha^2}\right)^2(\alpha^2\Delta_{\rm PCA}+a_\Omega)\bigg)\\
    &\stackrel{\mathclap{\mbox{\footnotesize (c)}}}{=} \gamma\left(\left(\Delta_{\rm PCA}+\frac{a_\Sigma}{\alpha^2}\right)R'\left(\frac{\gamma}{\alpha^2}\right)+(\alpha^2\Delta_{\rm PCA}+a_\Omega)\left(\frac{\gamma}{\alpha^2}R'\left(\frac{\gamma}{\alpha^2}\right)-R\left(\frac{\gamma}{\alpha^2}\right)\right)\right).
    \end{split}
\end{equation}
Here, (a) follows from the hypothesis that $\kappa_i^\infty\ge 0$ for $i\ge 2$; (b) holds since $(\bx, \by)\in \mathcal X_{I_\Sigma^*}\times\mathcal X_{I^*_{\Omega}}$; and (c) uses \eqref{eq:R1rect}-\eqref{eq:R20rect}. With similar passages, we also obtain that
\begin{equation}\label{eq:hSi}
    |h_{s, t}^\Sigma(\bx, \by)|\le \left(\gamma^2\Delta_{\rm PCA}+\frac{\gamma^2 a_\Omega}{\alpha^2}\right)R'\left(\frac{\gamma}{\alpha^2}\right)+(\alpha^2\Delta_{\rm PCA}+a_\Sigma)\left(\frac{\gamma}{\alpha^2}R'\left(\frac{\gamma}{\alpha^2}\right)-R\left(\frac{\gamma}{\alpha^2}\right)\right).
\end{equation}
Set $x=\gamma/\alpha^2$. Then, by using \eqref{eq:hOm} and \eqref{eq:hSi}, we obtain that the desired result holds if the following pair of inequalities is satisfied:
\begin{equation}\label{eq:cnd1}
    \begin{split}
        &\Delta_{\rm PCA}(\gamma R'(x)+\gamma\alpha^2(x R'(x)-R(x)))+a_\Sigma x R'(x)+ a_\Omega \gamma (x R'(x)-R(x))\le a_\Omega,\\
        &\Delta_{\rm PCA}(\gamma^2 R'(x)+\alpha^2(x R'(x)-R(x)))+a_\Sigma (x R'(x)-R(x))+ a_\Omega \gamma x R'(x)\le a_\Sigma.
    \end{split}
\end{equation}
Set $\beta= a_\Sigma/a_\Omega$. Then, \eqref{eq:cnd1} can be rewritten as
\begin{equation*}
\begin{split}
    &\Delta_{\rm PCA}(\gamma R'(x)+\gamma\alpha^2(x R'(x)-R(x)))+a_\Omega \left(\beta x R'(x)+ \gamma (x R'(x)-R(x))\right)\le a_\Omega,\\
    &\Delta_{\rm PCA}(\gamma^2 R'(x)+\alpha^2(x R'(x)-R(x)))+a_\Omega\left(\beta (x R'(x)-R(x))+ \gamma x R'(x)\right)\le \beta a_\Omega.
\end{split}
\end{equation*}
This pair of inequalities holds for a sufficiently large $a_\Omega$ if
\begin{equation}\label{eq:cnd2}
    \begin{split}
        &\beta x R'(x)+ \gamma (x R'(x)-R(x)) < 1,\\
        &\beta (x R'(x)-R(x))+ \gamma x R'(x)<\beta.
    \end{split}
\end{equation}
Recall that, above the spectral threshold, namely, when $\tilde{\alpha} > \tilde{\alpha}_{\rm s}$, the PCA estimator $\bu_{\rm PCA}$ has strictly positive correlation with the signal $\bu^*$:
\begin{equation*}
    \frac{\langle\bu_{\rm PCA}, \bu^*\rangle^2}{n}\stackrel{\mathclap{\mbox{\footnotesize a.s.}}}{\longrightarrow}\Delta_{\rm PCA}>0.
\end{equation*}
Furthermore, from \cite[Eq. (7.32)]{fan2020approximate}, we have that $\Delta_{\rm PCA}$ can be expressed as 
\[
\Delta_{\PCA} = \frac{T(R(x))-x T'(R(x))R'(x)}{1+ \gamma R(x)},
\]
where $T(z)=(1+z)(1+\gamma z)$. We therefore obtain that
\begin{equation}\label{eq:usefulin}
    T(R(x))-x T'(R(x))R'(x) >0.
\end{equation}
 By using \eqref{eq:usefulin}, one can readily verify that $1-xR'(x)+R(x)> 0$. Furthermore, we have that $xR'(x)> 0$, as $x> 0$ and the rectangular free cumulants are non-negative. Since $xR'(x)> 0$ and $1-xR'(x)+R(x)> 0$, \eqref{eq:cnd2} can be rewritten as
\begin{equation*}
\begin{split}
    & \frac{\gamma x R'(x)}{1-xR'(x)+R(x)} < \beta< \frac{1-\gamma xR'(x)+\gamma R(x)}{x R'(x)}.
    \end{split}
\end{equation*}
These above inequalities can be simultaneously satisfied for some value of $\beta$ if
\begin{equation}\label{eq:ineqR}
\frac{\gamma x R'(x)}{1-xR'(x)+R(x)} < 
    \frac{1-\gamma xR'(x)+\gamma R(x)}{x R'(x)}.
\end{equation}
By using again that $xR'(x)> 0$ and $1-xR'(x)+R(x)> 0$, \eqref{eq:ineqR} can be rewritten as
\begin{equation}\label{eq:lastineqrect}
    1-(1+\gamma)(xR'(x)-R(x))+\gamma (xR'(x)-R(x))^2 > \gamma (x R'(x))^2.
\end{equation}
The inequality \eqref{eq:lastineqrect} can be readily obtained from \eqref{eq:usefulin}, and the proof is complete. 
\end{proof}

Next, we compute a fixed point of $(h^{\Sigma}, h^\Omega)$.

\begin{lemma}[Fixed point of limit maps -- Rectangular matrices]\label{lemma:fixedrect}
    Consider the maps $h^\Omega, h^{\Sigma}$ defined in \eqref{eq:fixedmaprect}. Let $\bx^*=(x^*_{s, t} : s, t\le 0)$ and $\by^*=(y^*_{s, t} : s, t\le 0)$ with $x^*_{s, t}=a^*$ and $y^*_{s, t}=b^*$, where $a^*$ and $b^*$ are defined in \eqref{eq:asbs}. Assume that $\tilde{\alpha}>\tilde{\alpha}_{\rm s}$. Then, $(\bx^*, \by^*)$ is a fixed point of $(h^\Sigma, h^{\Omega})$.
\end{lemma}

\begin{proof}
Note that, for $z=\gamma/\alpha^2$, the power series expansion \eqref{eq:R1rect} of $R'$ converges to a finite limit as $\tilde{\alpha}>\tilde{\alpha}_{\rm s}$. Hence, by using the definition \eqref{eq:fixedmaprect}, we have that 
\begin{equation}\label{eq:fprect1}
\begin{split}
        h_{s, t}^\Omega(\bx^*, \by^*) &=\gamma\left(\left(\Delta_{\rm PCA}+\frac{a^*}{\alpha^2}\right)R'\left(\frac{\gamma}{\alpha^2}\right)+(\alpha^2\Delta_{\rm PCA}+b^*)\left(\frac{\gamma}{\alpha^2}R'\left(\frac{\gamma}{\alpha^2}\right)-R\left(\frac{\gamma}{\alpha^2}\right)\right)\right),\\
        h_{s, t}^\Sigma(\bx^*, \by^*) &=  \left(\gamma^2\Delta_{\rm PCA}+\frac{\gamma^2 b^*}{\alpha^2}\right)R'\left(\frac{\gamma}{\alpha^2}\right)+(\alpha^2\Delta_{\rm PCA}+a^*)\left(\frac{\gamma}{\alpha^2}R'\left(\frac{\gamma}{\alpha^2}\right)-R\left(\frac{\gamma}{\alpha^2}\right)\right).  
\end{split}
\end{equation}
Since a fixed point should satisfy $h_{s, t}^\Omega(\bx^*, \by^*)=b^*$ and $h_{s, t}^\Sigma(\bx^*, \by^*)=a^*$, writing $x=\gamma/\alpha^2$,  \eqref{eq:fprect1} becomes
\begin{equation}
\left\{\begin{array}{l}
        \gamma\Delta_{\rm PCA}(R'(x)+\alpha^2(xR'(x)-R(x)))
        +a^* x R'(x)+b^* \gamma (xR'(x)-R(x))=b^*, \\
        \Delta_{\rm PCA}(\gamma^2 R'(x)+\alpha^2(xR'(x)-R(x)))
        +a^* (x R'(x)-R(x))+b^* \gamma xR'(x)=a^*.
\end{array}\right.
\label{eq:asbs_solve}
\end{equation}
Solving \eqref{eq:asbs_solve} for $a^*$ and $b^*$, and  using the expression for $\Delta_{\rm PCA}$ given in \cite[Eq. (7.32)]{fan2020approximate}, we obtain the formulas for $(a^*, b^*)$ given in  \eqref{eq:asbs}. 
\end{proof}

The next step is to show Lipschitz bounds on the maps $h^{\Sigma}, h^\Omega$.

\begin{lemma}[Lipschitz bounds on limit maps]\label{lemma:contrect}
    Consider the map $(h^\Omega(\bx, \by), h^\Sigma(\bx, \by)):\mathcal X_{I^*_{\Omega}}\times \mathcal X_{I_\Sigma^*}\to \mathcal X_{I^*_{\Omega}}\times \mathcal X_{I_\Sigma^*}$ defined in \eqref{eq:fixedmaprect} and where $I^*_\Omega$, $I^*_\Sigma$ are given by Lemma \ref{lemma:imagerect}. Assume that $\kappa_{2i}^\infty\ge 0$ for all $i\ge 1$, and let $\xi <1$ be such that $\tilde{\alpha}\sqrt{\xi}>\tilde{\alpha}_{\rm s}$. Then, for any $(\bx, \by) \in \mathcal X_{I^*_{\Sigma}}\times \mathcal X_{I_\Omega^*}$, 
    \begin{equation}\label{eq:Liprect}
    \begin{split}
            \|h^\Omega(\bx, \by)-h^\Omega(\bx', \by')\|_\xi &\le \tilde{x} R'(\tilde{x}) \|\bx-\bx'\|_\xi+\gamma\left(\tilde{x}R'(\tilde{x})-R(\tilde{x})\right)\|\by-\by'\|_\xi,
    \end{split}
    \end{equation}
\begin{equation}\label{eq:Liprect2}
    \begin{split}
\|h^\Sigma(\bx, \by)-h^\Sigma(\bx', \by')\|_\xi & \le \gamma \tilde{x} R'(\tilde{x}) \|\by-\by'\|_\xi+\left(\tilde{x}R'(\tilde{x})-R(\tilde{x})\right)\|\bx-\bx'\|_\xi,
            \end{split}
    \end{equation}
    where we have set $\tilde{x}=\gamma/(\xi\alpha^2)$.
\end{lemma}

\begin{proof}
Since $\kappa_{2i}^\infty\ge 0$ for $i\ge 1$, we have
\begin{equation}\label{eq:int1rect}
    \begin{split}
        |h_{s, t}^\Omega(\bx, \by)-h_{s, t}^\Omega(\bx', \by')| &\le \gamma \sum_{j=0}^\infty\sum_{k=0}^\infty \left(\frac{\gamma}{\alpha^2}\right)^{j+k}\bigg( \kappa_{2(j+k+1)}^\infty\frac{1}{\alpha^2}|x_{s-j, t-k}-x'_{s-j, t-k}|\\
        &\hspace{5em}+\kappa_{2(j+k+2)}^\infty\frac{\gamma^2}{\alpha^4}|y_{s-j, t-k}-y'_{s-j, t-k}|\bigg).
    \end{split}
\end{equation}
Note that 
\begin{equation}\label{eq:int2rect}
    \begin{split}
        |x_{s-j, t-k}-x'_{s-j, t-k}|&\le\|\bx-\bx'\|_\xi\xi^{-\max(|s-j|, |t-k|)}\le\|\bx-\bx'\|_\xi\xi^{-\max(|s|, |t|)-j-k}, \\
        |y_{s-j, t-k}-y'_{s-j, t-k}|&\le\|\by-\by'\|_\xi \xi^{-\max(|s-j|, |t-k|)}\le\|\by-\by'\|_\xi \xi^{-\max(|s|, |t|)-j-k}.
    \end{split}
\end{equation}
Thus, by combining \eqref{eq:int1rect} and \eqref{eq:int2rect}, we have
\begin{equation}\label{eq:int3rect}
    \|h^\Omega(\bx, \by)-h^\Omega(\bx', \by')\|_\xi\le \frac{\gamma}{\alpha^2} \sum_{j=0}^\infty\sum_{k=0}^\infty \left(\frac{\gamma}{\xi\alpha^2}\right)^{j+k}
    \hspace{-4pt}\bigg( \kappa_{2(j+k+1)}^\infty\|\bx-\bx'\|_\xi+\kappa_{2(j+k+2)}^\infty\frac{\gamma^2}{\alpha^2}\|\by-\by'\|_\xi\bigg). 
\end{equation}
By using \eqref{eq:R1rect} and \eqref{eq:R20rect} to compute the sums in \eqref{eq:int3rect}, we deduce that
  \begin{equation}\label{eq:Lipint1}
    \begin{split}
            \|h^\Omega(\bx, \by)-h^\Omega(\bx', \by')\|_\xi &\le \frac{\gamma}{\alpha^2} R'\left(\frac{\gamma}{\xi\alpha^2}\right) \|\bx-\bx'\|_\xi\\
            &\hspace{1em}+ \xi^2\gamma
            \left(\frac{\gamma}{\xi\alpha^2} R'\left(\frac{\gamma}{\xi\alpha^2}\right)-R\left(\frac{\gamma}{\xi\alpha^2}\right)\right)\|\by-\by'\|_\xi.
    \end{split}
    \end{equation}
Recall that $\xi<1$ and note from \eqref{eq:R20rect} that $\tilde{x}R'(\tilde{x})\ge R(\tilde{x})\ge 0$ with $\tilde{x}=\gamma/(\xi\alpha^2)$. Thus, the claim \eqref{eq:Liprect} readily follows from \eqref{eq:Lipint1}.

The proof of \eqref{eq:Liprect2} is analogous. First, we use that $\kappa_{2i}^\infty\ge 0$ for $i\ge 1$ and obtain
\begin{equation}\label{eq:int1rect2}
    \begin{split}
        |h_{s, t}^\Sigma(\bx, \by)-h_{s, t}^\Sigma(\bx', \by')| &\le  \sum_{j=0}^\infty\sum_{k=0}^\infty \left(\frac{\gamma}{\alpha^2}\right)^{j+k}\bigg( \kappa_{2(j+k+1)}^\infty\frac{\gamma^2}{\alpha^2}|y_{s-j, t-k}-y'_{s-j, t-k}|\\
        &\hspace{5em}+\kappa_{2(j+k+2)}^\infty\frac{\gamma^2}{\alpha^4}|x_{s-j, t-k}-x'_{s-j, t-k}|\bigg).
    \end{split}
\end{equation}
Thus, by using \eqref{eq:int2rect}, we have
\begin{equation}\label{eq:int3rect2}
    \|h^\Sigma(\bx, \by)-h^\Sigma(\bx', \by')\|_\xi\le \sum_{j=0}^\infty\sum_{k=0}^\infty \left(\frac{\gamma}{\xi\alpha^2}\right)^{j+k}\bigg( \kappa_{2(j+k+1)}^\infty\frac{\gamma^2}{\alpha^2}\|\by-\by'\|_\xi+\kappa_{2(j+k+2)}^\infty\frac{\gamma^2}{\alpha^4}\|\bx-\bx'\|_\xi\bigg).
\end{equation}
Finally, by using \eqref{eq:R1rect} and \eqref{eq:R20rect} to compute the sums in \eqref{eq:int3rect2}, we deduce that
\begin{equation}
    \begin{split}
\|h^\Sigma(\bx, \by)-h^\Sigma(\bx', \by')\|_\xi & \le \frac{\gamma^2}{\alpha^2} R'\left(\frac{\gamma}{\xi\alpha^2}\right) \|\by-\by'\|_\xi\\
&\hspace{1em}+ \xi^2\left(\frac{\gamma}{\xi\alpha^2} R'\left(\frac{\gamma}{\xi\alpha^2}\right)-R\left(\frac{\gamma}{\xi\alpha^2}\right)\right)\|\bx-\bx'\|_\xi,
            \end{split}
    \end{equation}
which readily leads to \eqref{eq:Liprect2}.
\end{proof}

Let us consider the map $G^{\Omega, \Sigma}$ obtained by the successive composition of $(\bx, \by)\mapsto (\bx, h^\Omega(\bx, \by))$ and $(\bx, \by)\mapsto (h^\Sigma(\bx, \by), \by)$, i.e.,
\begin{equation}\label{eq:Gmap}
    G^{\Omega, \Sigma}(\bx, \by) = (G_x^{\Omega, \Sigma}(\bx, \by), G_y^{\Omega, \Sigma}(\bx, \by)) = \left(h^\Sigma(\bx, h^\Omega(\bx, \by)), h^\Omega(\bx, \by)\right).
\end{equation}
Given $\beta>0$, define the norm $\|\cdot\|_{\xi, \beta}$ as
\begin{equation}
    \|(\bx, \by)\|_{\xi, \beta}=\|\bx\|_\xi+\beta\|\by\|_\xi.
\end{equation}
We now use the Lipschitz bounds of Lemma \ref{lemma:contrect} to prove that $G^{\Omega, \Sigma}$ is a contraction for a certain value of $\beta$. 

\begin{lemma}[Composition of limit maps is a contraction]\label{lemma:contrect2}
    Consider the map $G^{\Omega, \Sigma}$ defined in \eqref{eq:Gmap}, and let $I^*_\Omega$, $I^*_\Sigma$ be the sets given by Lemma \ref{lemma:imagerect}. Assume that $\kappa_{2i}^\infty\ge 0$ for all $i\ge 1$, and let $\xi <1$ be such that $\tilde{\alpha}\sqrt{\xi}>\tilde{\alpha}_{\rm s}$. Then, if $(\bx, \by) \in \mathcal X_{I^*_{\Sigma}}\times \mathcal X_{I_\Omega^*}$, we have that $G^{\Omega, \Sigma}(\bx, \by)\in \mathcal X_{I^*_{\Omega}}\times \mathcal X_{I_\Sigma^*}$. Furthermore, 
there exists $\beta^*>0$ and $\tau<1$ such that, for any $(\bx, \by) \in \mathcal X_{I^*_{\Sigma}}\times \mathcal X_{I_\Omega^*}$, 
    \begin{equation}\label{eq:Liprect3}
    \begin{split}
            \|G^{\Omega, \Sigma}(\bx, \by)-G^{\Omega, \Sigma}(\bx', \by')\|_{\xi, \beta^*} &\le \tau \|(\bx, \by)-(\bx', \by')\|_{\xi, \beta^*}.
    \end{split}
    \end{equation}
\end{lemma}

\begin{proof}
The claim that $G^{\Omega, \Sigma}: \mathcal X_{I^*_{\Omega}}\times \mathcal X_{I_\Sigma^*}\to \mathcal X_{I^*_{\Omega}}\times \mathcal X_{I_\Sigma^*}$ follows directly from Lemma \ref{lemma:imagerect}. We now show that \eqref{eq:Liprect3} holds. By using the definition \eqref{eq:Gmap} and the Lipschitz bounds \eqref{eq:Liprect}-\eqref{eq:Liprect2} of Lemma \ref{lemma:contrect}, we obtain that 
\begin{equation}
\begin{split}
     \|G^{\Omega, \Sigma}(\bx, \by)-G^{\Omega, \Sigma}(\bx', \by')\|_{\xi, \beta}&\le \|\bx-\bx'\|_\xi \left(\tilde{x}R'(\tilde{x})-R(\tilde{x})+\gamma(\tilde{x}R'(\tilde{x}))^2+\beta \tilde{x}R'(\tilde{x})\right)\\
     &\hspace{-3.5em}+\|\by-\by'\|_\xi\left(\gamma^2(\tilde{x}R'(\tilde{x}))^2-\gamma^2\tilde{x}R'(\tilde{x})R(\tilde{x}) +\beta\gamma(\tilde{x}R'(\tilde{x})-R(\tilde{x}))\right),
\end{split}
\end{equation}
where we have set $\tilde{x}=\gamma/(\xi\alpha^2)$. Hence, the claim of the lemma holds if there exists $\beta^*>0$ and $\tau<1$ such that
\begin{equation}\label{eq:taubeta}
    \begin{split}
        &\beta^* \tilde{x} R'(\tilde{x})+ \gamma (\tilde{x} R'(\tilde{x}))^2 -R(\tilde{x})+\tilde{x} R'(\tilde{x}) \le \tau,\\
        &\beta^* \gamma(\tilde{x} R'(\tilde{x})-R(\tilde{x}))+ \gamma^2 (\tilde{x} R'(\tilde{x}))^2-\gamma^2 \tilde{x}R(\tilde{x})R'(\tilde{x})\le \tau\beta^*.
    \end{split}
\end{equation}
We note that, as $\tilde{\alpha}\sqrt{\xi} > \tilde{\alpha}_{\rm s}$, \eqref{eq:usefulin} holds with $\tilde{x}$ in place of $x$. Hence, one readily verifies that $1-\gamma\tilde{x}R'(\tilde{x})+R(\tilde{x})> 0$. Furthermore, we have that $\tilde{x}R'(\tilde{x})>0$, as $\tilde{x}>0$ and the rectangular free cumulants are non-negative. Thus, the two inequalities in \eqref{eq:taubeta} can be satisfied simultaneously if there exists $\beta^*>0$ such that
\begin{equation*}
    \frac{\gamma^2 (\tilde{x} R'(\tilde{x}))^2-\gamma^2 \tilde{x}R(\tilde{x})R'(\tilde{x})}{1-\gamma\tilde{x}R'(\tilde{x})+\gamma R(\tilde{x})} < \beta^*< \frac{1-\gamma (\tilde{x}R'(\tilde{x}))^2-\tilde{x}R'(\tilde{x})+ R(x)}{\tilde{x} R'(\tilde{x})}.
\end{equation*}
These last two inequalities can be satisfied simultaneously if
\begin{equation}\label{eq:ineqR2}
\frac{\gamma^2 (\tilde{x} R'(\tilde{x}))^2-\gamma^2 \tilde{x}R(\tilde{x})R'(\tilde{x})}{1-\gamma\tilde{x}R'(\tilde{x})+\gamma R(\tilde{x})} < 
    \frac{1-\gamma (\tilde{x}R'(\tilde{x}))^2-\tilde{x}R'(\tilde{x})+ R(x)}{\tilde{x} R'(\tilde{x})}.
\end{equation}
By using again that $1-\gamma\tilde{x}R'(\tilde{x})+R(\tilde{x})> 0$ and $\tilde{x}R'(\tilde{x})>0$, \eqref{eq:ineqR2} can be rewritten as
\begin{equation*}
\begin{split}
&\left(1-\gamma (\tilde{x}R'(\tilde{x}))^2-\tilde{x}R'(\tilde{x})+ R(x)\right)\left(1-\gamma\tilde{x}R'(\tilde{x})+\gamma R(\tilde{x})\right) \\
&\hspace{12em}> \tilde{x} R'(\tilde{x}) \left(\gamma^2 (\tilde{x} R'(\tilde{x}))^2-\gamma^2 \tilde{x}R(\tilde{x})R'(\tilde{x})\right),
\end{split}
\end{equation*}
which again follows from \eqref{eq:usefulin} with $\tilde{x}$ in place of $x$. Thus, there exists $\beta^* >0$ and $\tau <1$ such that \eqref{eq:taubeta} is satisfied, completing the proof.
\end{proof}

At this point, we show that the state evolution of $\btSigma_{\bar{T}}$, $\btOmega_{\bar{T}}$ can be approximated via the fixed maps $h^\Sigma, h^\Omega$.

\begin{lemma}[Limit maps approximate SE maps -- Rectangular matrices]\label{lemma:apprect}
Consider the map $(h^\Omega(\bx, \by), h^\Sigma(\bx, \by)):\mathcal X_{I^*_{\Omega}}\times \mathcal X_{I_\Sigma^*}\to \mathcal X_{I^*_{\Omega}}\times \mathcal X_{I_\Sigma^*}$ defined in \eqref{eq:fixedmaprect}, where $I^*_\Omega$, $I^*_\Sigma$ are given by Lemma \ref{lemma:imagerect}. Assume that $\kappa_{2i}^\infty\ge 0$ for all $i\ge 1$, and let $\xi <1$ be such that $\tilde{\alpha}\sqrt{\xi}>\tilde{\alpha}_{\rm s}$. Then, for any $(\bx, \by) \in \mathcal X_{I^*_{\Sigma}}\times \mathcal X_{I_\Omega^*}$, 
\begin{equation}\label{eq:convrect1}
\begin{split}
    \|\btOmega_{\bar{T}} - h^\Omega(\bx, \by)\|_\xi &\le \tilde{x} R'(\tilde{x}) \|\btSigma_{\bar{T}-1} - \bx\|_\xi + \gamma(\tilde{x}R'(\tilde{x})-R(\tilde{x}))\|\btOmega_{\bar{T}-1} - \by\|_\xi+F_1(\bar{T}),
    \end{split}
\end{equation}
\begin{equation}\label{eq:convrect2}
\begin{split}
    \|\btSigma_{\bar{T}} - h^\Sigma(\bx, \by)\|_\xi &\le \gamma\tilde{x} R'(\tilde{x}) \|\btOmega_{\bar{T}-1} - \by\|_\xi + (\tilde{x}R'(\tilde{x})-R(\tilde{x}))\|\btSigma_{\bar{T}-1} - \bx\|_\xi+F_2(\bar{T}),
    \end{split}
\end{equation}
where $\tilde{x}=\gamma/(\xi\alpha^2)$ and
\begin{equation}
    \lim_{\bar{T}\to\infty}F_1(\bar{T})=0, \quad     \lim_{\bar{T}\to\infty}F_2(\bar{T})=0.
\end{equation}
\end{lemma}

\begin{proof}
First, we write
\begin{equation*}
\begin{split}
    \|\btOmega_{\bar{T}} - h^\Omega(\bx, \by)\|_\xi &= \sup_{s, t\le 0} \xi^{\max(|s|, |t|)}|(\btOmega_{\bar{T}})_{s, t} - h^\Omega_{s, t}(\bx, \by)|\\
    &= \max\Bigg(\sup_{\substack{s, t\le 0\\\max(|s|, |t|)<\bar{T}}} \xi^{\max(|s|, |t|)}|(\btOmega_{\bar{T}})_{s, t} - h^\Omega_{s, t}(\bx, \by)|,\\
    &\hspace{5em}\sup_{\substack{s, t\le 0\\\max(|s|, |t|)\ge \bar{T}}} \xi^{\max(|s|, |t|)}|(\btOmega_{\bar{T}})_{s, t} - h^\Omega_{s, t}(\bx, \by)|\Bigg),
    \end{split}
\end{equation*}
where $(\btOmega_{\bar{T}})_{s, t}=\tilde{\omega}_{s+\bar{T}, t+\bar{T}}$ if $s\ge -\bar{T}$ and $t\ge -\bar{T}$, and $(\btOmega_{\bar{T}})_{s, t}=0$ otherwise. 

Let us look at the case $\max(|s|, |t|)<\bar{T}$, and define $I_1 = \{(j, k): j\ge s+\bar{T} \mbox{ or }k\ge t+\bar{T}\}$. Then, 
\begin{equation}\label{eq:case1int1rect}
\begin{split}
&    |(\btOmega_{\bar{T}})_{s, t} - h^\Omega_{s, t}(\bx, \by)| \\
&= \bigg|\gamma\sum_{j=0}^{s+\bar{T}-1}\sum_{k=0}^{t+\bar{T}-1}\left(\frac{\gamma}{\alpha^2}\right)^{j+k}
    \bigg(\kappa_{2(j+k+1)}^\infty\frac{1}{\alpha^2}\left(\alpha^2\Delta_{\rm PCA}+\tilde{\sigma}_{s-j+\bar{T}-1, t-k+\bar{T}-1}\right)\\
    &\hspace{12em}+\kappa_{2(j+k+2)}^\infty\frac{\gamma^2}{\alpha^4}\left(\alpha^2\Delta_{\rm PCA}+\tilde{\omega}_{s-j+\bar{T}-1, t-k+\bar{T}-1}\right)\bigg)\\
    &\hspace{2em}-\gamma\sum_{j=0}^{\infty}\sum_{k=0}^{\infty}\left(\frac{\gamma}{\alpha^2}\right)^{j+k}
    \bigg(\kappa_{2(j+k+1)}^\infty\frac{1}{\alpha^2}\left(\alpha^2\Delta_{\rm PCA}+x_{s-j, t-k}\right)\\
    &\hspace{12em}+\kappa_{2(j+k+2)}^\infty\frac{\gamma^2}{\alpha^4}\left(\alpha^2\Delta_{\rm PCA}+y_{s-j, t-k}\right)\bigg)\bigg|\\
    &\le \bigg|\gamma\sum_{j=0}^{s+\bar{T}-1}\sum_{k=0}^{t+\bar{T}-1}\left(\frac{\gamma}{\alpha^2}\right)^{j+k}
    \bigg(\kappa_{2(j+k+1)}^\infty\frac{1}{\alpha^2}\left(x_{s-j, t-k}-\tilde{\sigma}_{s-j+\bar{T}-1, t-k+\bar{T}-1}\right)\\
    &\hspace{12em}+\kappa_{2(j+k+2)}^\infty\frac{\gamma^2}{\alpha^4}\left(y_{s-j, t-k}-\tilde{\omega}_{s-j+\bar{T}-1, t-k+\bar{T}-1}\right)\bigg)\bigg|\\
    &+\bigg|\gamma\sum_{j, k\in I_1}\left(\frac{\gamma}{\alpha^2}\right)^{j+k}
    \bigg(\kappa_{2(j+k+1)}^\infty\frac{1}{\alpha^2}\left(\alpha^2\Delta_{\rm PCA}+x_{s-j, t-k}\right)\\
    &\hspace{12em}+\kappa_{2(j+k+2)}^\infty\frac{\gamma^2}{\alpha^4}\left(\alpha^2\Delta_{\rm PCA}+y_{s-j, t-k}\right)\bigg)\bigg|:= T_1 + T_2.
\end{split}
\end{equation}
The term $T_1$ can be upper bounded as follows:
\begin{equation*}
    \begin{split}
        T_1 &\stackrel{\mathclap{\mbox{\footnotesize (a)}}}{\le} \gamma\sum_{j=0}^{s+\bar{T}-1}\sum_{k=0}^{t+\bar{T}-1}\left(\frac{\gamma}{\alpha^2}\right)^{j+k}
    \bigg(\kappa_{2(j+k+1)}^\infty\frac{1}{\alpha^2}\left|x_{s-j, t-k}-\tilde{\sigma}_{s-j+\bar{T}-1, t-k+\bar{T}-1}\right|\\
    &\hspace{12em}+\kappa_{2(j+k+2)}^\infty\frac{\gamma^2}{\alpha^4}\left|y_{s-j, t-k}-\tilde{\omega}_{s-j+\bar{T}-1, t-k+\bar{T}-1}\right|\bigg)\\
        &\le\|\btSigma_{\bar{T}-1}-\bx\|_\xi\xi^{-\max(|s|, |t|)}\gamma\sum_{j=0}^{s+\bar{T}-1}\sum_{k=0}^{t+\bar{T}-1}\left(\frac{\gamma}{\xi\alpha^2}\right)^{j+k}
    \kappa_{2(j+k+1)}^\infty\frac{1}{\alpha^2}\\
    &\hspace{4em}+\|\btOmega_{\bar{T}-1}-\by\|_\xi\xi^{-\max(|s|, |t|)}\gamma\sum_{j=0}^{s+\bar{T}-1}\sum_{k=0}^{t+\bar{T}-1}\left(\frac{\gamma}{\xi\alpha^2}\right)^{j+k}
    \kappa_{2(j+k+2)}^\infty\frac{\gamma}{\alpha^4}
    \end{split}
\end{equation*}    
    \begin{equation}\label{eq:case1int2rect}
    \begin{split}
      \phantom{T_1\hspace{2em}}  &\stackrel{\mathclap{\mbox{\footnotesize (b)}}}{\le}\|\btSigma_{\bar{T}-1}-\bx\|_\xi\xi^{-\max(|s|, |t|)}\gamma\sum_{j=0}^{\infty}\sum_{k=0}^{\infty}\left(\frac{\gamma}{\xi\alpha^2}\right)^{j+k}
    \kappa_{2(j+k+1)}^\infty\frac{1}{\alpha^2}\\
    &\hspace{4em}+\|\btOmega_{\bar{T}-1}-\by\|_\xi\xi^{-\max(|s|, |t|)}\gamma\sum_{j=0}^{\infty}\sum_{k=0}^{\infty}\left(\frac{\gamma}{\xi\alpha^2}\right)^{j+k}
    \kappa_{2(j+k+2)}^\infty\frac{\gamma}{\alpha^4}\\
        &\stackrel{\mathclap{\mbox{\footnotesize (c)}}}{\le} \|\btSigma_{\bar{T}-1}-\bx\|_\xi\xi^{-\max(|s|, |t|)} \tilde{x}R'(\tilde{x})+\|\btOmega_{\bar{T}-1}-\by\|_\xi\xi^{-\max(|s|, |t|)}\gamma (\tilde{x}R'(\tilde{x})-R(\tilde{x})),
    \end{split}
\end{equation}
where $\tilde{x}=\gamma/(\xi\alpha^2)$.
Here, (a) and (b) follow from the hypothesis that $\kappa_{2i}^\infty\ge 0$ for $i\ge 1$, (c) uses \eqref{eq:R1rect}, \eqref{eq:R20rect} and that $\xi\le 1$. By using that $(\bx, \by) \in \mathcal X_{I^*_{\Sigma}}\times \mathcal X_{I_\Omega^*}$, the term $T_2$ can be upper bounded as follows:
\begin{equation}\label{eq:case1int3rect}
    \begin{split}
        T_2&\le C_1\sum_{j, k\in I_1}\left(\frac{\gamma}{\alpha^2}\right)^{j+k}(\kappa_{2(j+k+1)}^\infty+\kappa_{2(j+k+2)}^\infty),
    \end{split}
\end{equation}
where $C_1$ is a constant independent of $s, t, \bar{T}$.
Note that, if $(j, k)\in I_1$, then $j+k\ge-\max(|s|,|t|)+\bar{T}$. Consequently, the RHS of \eqref{eq:case1int3rect} can upper bounded by \begin{equation}\label{eq:case1int4rect}
    C_2\sum_{i=\bar{T}-\max(|s|, |t|)}^\infty \left(\frac{\gamma}{\alpha^2}\right)^i (i+1)\kappa^\infty_{2(i+1)},
\end{equation}
where $C_2$ is a constant independent of $s, t, \bar{T}$.
By combining \eqref{eq:case1int1rect}, \eqref{eq:case1int2rect}, \eqref{eq:case1int3rect} and \eqref{eq:case1int4rect}, we obtain that 
\begin{equation}\label{eq:case1finrect}
\begin{split}
    &\sup_{\substack{s, t\le 0\\\max(|s|, |t|)<\bar{T}}} \xi^{\max(|s|, |t|)}|(\btOmega_{\bar{T}})_{s, t} - h^\Omega_{s, t}(\bx, \by)|\le \|\btSigma_{\bar{T}-1} - \bx\|_\xi \tilde{x}R'(\tilde{x})\\
    &\hspace{4em}+\|\btOmega_{\bar{T}-1} - \by\|_\xi\gamma(\tilde{x}R'(\tilde{x})-R(\tilde{x}))+C_2 \sup_{0\le t\le \bar{T}}\xi^t\sum_{i=\bar{T}-t}^\infty \left(\frac{\gamma}{\alpha^2}\right)^i (i+1)\kappa^\infty_{2(i+1)}.
\end{split}
\end{equation}

Let us now look at the case $\max(|s|, |t|)\ge \bar{T}$. Recall that $|h^\Omega_{s, t}(\bx, \by)|\le a_\Omega$, $\tilde{\sigma}_{0, 0}=(1-\Delta_{\rm PCA})\alpha^2$ and $\tilde{\sigma}_{0, t}=0$ for $t\in [1, \bar{T}]$. Thus,
\begin{equation*}
    |(\btOmega_{\bar{T}})_{s, t} - h^\Omega_{s, t}(\bx, \by)|\le C_3,
\end{equation*}
where $C_3$ is a constant independent of $s, t, \bar{T}$. This immediately implies that
\begin{equation*}
    \sup_{\substack{s, t\le 0\\\max(|s|, |t|)\ge \bar{T}}} \xi^{\max(|s|, |t|)}|(\btOmega_{\bar{T}})_{s, t} - h^\Omega_{s, t}(\bx, \by)|\le C_3 \xi^{\bar{T}},
\end{equation*}
which combined with \eqref{eq:case1finrect} allows us to conclude that
\begin{equation}\label{eq:case1finrect2}
\begin{split}
    &\|\btOmega_{\bar{T}} - h^\Omega(\bx, \by)\|_\xi\le \|\btSigma_{\bar{T}-1} - \bx\|_\xi \tilde{x}R'(\tilde{x})\\
    &\hspace{2em}+\|\btOmega_{\bar{T}-1} - \by\|_\xi\gamma(\tilde{x}R'(\tilde{x})-R(\tilde{x}))+C_2 \sup_{0\le t\le \bar{T}}\xi^t\sum_{i=\bar{T}-t}^\infty \left(\frac{\gamma}{\alpha^2}\right)^i (i+1)\kappa^\infty_{2(i+1)}+C_3\xi^{\bar{T}}.
\end{split}
\end{equation}
As $\tilde{\alpha}>\tilde{\alpha}_{\rm s}$ and the series in \eqref{eq:R1rect} is convergent for $z<1/(\tilde{\alpha}_{\rm s})^2$, one readily verifies that 
\begin{equation}\label{eq:limver}
\lim_{\bar{T}\to\infty}    \sup_{0\le t\le \bar{T}}\xi^t\sum_{i=\bar{T}-t}^\infty \left(\frac{\gamma}{\alpha^2}\right)^i (i+1)\kappa^\infty_{2(i+1)}=0,
\end{equation}
which concludes the proof of \eqref{eq:convrect1}.

The proof of \eqref{eq:convrect2} follows similar passages, and we outline them below. First, we write
\begin{equation*}
\begin{split}
    \|\btSigma_{\bar{T}} - h^\Sigma(\bx, \by)\|_\xi &= \sup_{s, t\le 0} \xi^{\max(|s|, |t|)}|(\btSigma_{\bar{T}})_{s, t} - h^\Sigma_{s, t}(\bx, \by)|\\
    &= \max\Bigg(\sup_{\substack{s, t\le 0\\\max(|s|, |t|)<\bar{T}}} \xi^{\max(|s|, |t|)}|(\btSigma_{\bar{T}})_{s, t} - h^\Sigma_{s, t}(\bx, \by)|,\\
    &\hspace{5em}\sup_{\substack{s, t\le 0\\\max(|s|, |t|)\ge \bar{T}}} \xi^{\max(|s|, |t|)}|(\btSigma_{\bar{T}})_{s, t} - h^\Sigma_{s, t}(\bx, \by)|\Bigg),
    \end{split}
\end{equation*}
where $(\btSigma_{\bar{T}})_{s, t}=\tilde{\sigma}_{s+\bar{T}, t+\bar{T}}$ if $s\ge -\bar{T}$ and $t\ge -\bar{T}$, and $(\btSigma_{\bar{T}})_{s, t}=0$ otherwise. For the case $\max(|s|, |t|)<\bar{T}$, we have
\begin{equation}\label{eq:lcase1}
\begin{split}
&    |(\btSigma_{\bar{T}})_{s, t} - h^\Sigma_{s, t}(\bx, \by)| \\
    &\le \bigg|\sum_{j=0}^{s+\bar{T}-1}\sum_{k=0}^{t+\bar{T}-1}\left(\frac{\gamma}{\alpha^2}\right)^{j+k}
    \bigg(\kappa_{2(j+k+1)}^\infty\frac{\gamma^2}{\alpha^2}\left(y_{s-j, t-k}-\tilde{\omega}_{s-j+\bar{T}, t-k+\bar{T}}\right)\\
    &\hspace{12em}+\kappa_{2(j+k+2)}^\infty\frac{\gamma^2}{\alpha^4}\left(x_{s-j, t-k}-\tilde{\sigma}_{s-j+\bar{T}-1, t-k+\bar{T}-1}\right)\bigg)\bigg|\\
    &+\bigg|\sum_{j, k\in I_1}\left(\frac{\gamma}{\alpha^2}\right)^{j+k}
    \bigg(\kappa_{2(j+k+1)}^\infty\frac{\gamma^2}{\alpha^2}\left(\alpha^2\Delta_{\rm PCA}+y_{s-j, t-k}\right)\\
    &\hspace{12em}+\kappa_{2(j+k+2)}^\infty\frac{\gamma^2}{\alpha^4}\left(\alpha^2\Delta_{\rm PCA}+x_{s-j, t-k}\right)\bigg)\bigg|:= T_3 + T_4.
\end{split}
\end{equation}
By using \eqref{eq:R1rect}, \eqref{eq:R20rect} and the non-negativity of the rectangular free cumulants, the term $T_3$ can be upper bounded as follows:
\begin{equation}\label{eq:lcase2}
    \begin{split}
        T_3 &\le \|\btOmega_{\bar{T}}-\by\|_\xi\xi^{-\max(|s|, |t|)} \gamma\tilde{x}R'(\tilde{x})+\|\btSigma_{\bar{T}-1}-\bx\|_\xi\xi^{-\max(|s|, |t|)} (\tilde{x}R'(\tilde{x})-R(\tilde{x})).
    \end{split}
\end{equation}
Furthermore, the term $T_4$ can be upper bounded as
\begin{equation}\label{eq:lcase3}
   T_4\le  C_4\sum_{i=\bar{T}-\max(|s|, |t|)}^\infty \left(\frac{\gamma}{\alpha^2}\right)^i (i+1)\kappa^\infty_{2(i+1)},
\end{equation}
where $C_4$ is a constant independent of $s, t, \bar{T}$.
For the case $\max(|s|, |t|)\ge \bar{T}$, we have
\begin{equation}\label{eq:lcase4}
    \sup_{\substack{s, t\le 0\\\max(|s|, |t|)\ge \bar{T}}} \xi^{\max(|s|, |t|)}|(\btSigma_{\bar{T}})_{s, t} - h^\Sigma_{s, t}(\bx, \by)|\le C_5 \xi^{\bar{T}},
\end{equation}
where $C_5$ is a constant independent of $s, t, \bar{T}$. By combining \eqref{eq:lcase1}, \eqref{eq:lcase2}, \eqref{eq:lcase3} and \eqref{eq:lcase4}, we conclude that 
\begin{equation*}
\begin{split}
    &\|\btSigma_{\bar{T}} - h^\Sigma(\bx, \by)\|_\xi\le \|\btOmega_{\bar{T}} - \by\|_\xi \gamma\tilde{x}R'(\tilde{x})\\
    &\hspace{2em}+\|\btSigma_{\bar{T}-1} - \bx\|_\xi(\tilde{x}R'(\tilde{x})-R(\tilde{x}))+C_4 \sup_{0\le t\le \bar{T}}\xi^t\sum_{i=\bar{T}-t}^\infty \left(\frac{\gamma}{\alpha^2}\right)^i (i+1)\kappa^\infty_{2(i+1)}+C_5\xi^{\bar{T}},
\end{split}
\end{equation*}
which, together with \eqref{eq:limver}, concludes the proof of \eqref{eq:convrect2}.
\end{proof}

Finally, we can put everything together and prove Lemma \ref{lem:SE_FP_phase1rect}. 

\begin{proof}[Proof of Lemma \ref{lem:SE_FP_phase1rect}]
Fix $\epsilon>0$ and denote by $\left(G^{\Omega, \Sigma}\right)^{T_0}$ the $T_0$-fold composition of the map $G^{\Omega, \Sigma}$ defined in \eqref{eq:Gmap}. Note that Lemma \ref{lemma:fixedrect} implies that $(\bx^*, \by^*)$ is a fixed point of $G^{\Omega, \Sigma}$, and Lemma \ref{lemma:contrect2} implies that this fixed point is unique. Then, for any $(\bx, \by) \in \mathcal X_{I^*_{\Sigma}}\times \mathcal X_{I_\Omega^*}$,
\begin{equation}\label{eq:pt1rect}
\begin{split}
    \|\left(G^{\Omega, \Sigma}\right)^{T_0}(\bx, \by)-(\bx^*, \by^*)\|_{\xi, \beta^*} &= \|\left(G^{\Omega, \Sigma}\right)^{T_0}(\bx, \by)-\left(G^{\Omega, \Sigma}\right)^{T_0}(\bx^*, \by^*)\|_{\xi, \beta^*} \\ & \le \tau^{T_0}\|(\bx, \by)-(\bx^*, \by^*)\|_{\xi, \beta^*},
    \end{split}
\end{equation}
where the inequality follows from Lemma \ref{lemma:contrect2}. Note that $\tau<1$ and $\mathcal X_{I^*_{\Omega}}\times \mathcal X_{I_\Sigma^*}$ is bounded under $\|\cdot\|_{\xi, \beta^*}$. Hence, we can make the RHS of \eqref{eq:pt1rect} smaller than $\epsilon/2$ by choosing a sufficiently large $T_0$. Furthermore, an application of Lemma \ref{lemma:apprect} gives that
\begin{equation}
    \begin{split}
&        \|(\btSigma_{\bar{T}}, \btOmega_{\bar{T}})-G^{\Omega, \Sigma}(\bx, \by)\|_{\xi, \beta^*}\le \|\btSigma_{\bar{T}-1}-\bx\|_\xi \left(\tilde{x}R'(\tilde{x})-R(\tilde{x})+\gamma(\tilde{x}R'(\tilde{x}))^2+\beta^* \tilde{x}R'(\tilde{x})\right)\\
        &\hspace{4em}+\|\btOmega_{\bar{T}-1}-\by\|_\xi \left(\gamma^2(\tilde{x}R'(\tilde{x}))^2-\gamma^2 \tilde{x}R'(\tilde{x})R(\tilde{x})+\beta^*\gamma(\tilde{x}R'(\tilde{x})-R(\tilde{x}))\right)+H(\bar{T})\\
     &\le   \tau\|(\btSigma_{\bar{T}-1}, \btOmega_{\bar{T}-1})-(\bx, \by)\|_{\xi, \beta^*}+H(\bar{T}),
    \end{split}
\end{equation}
where $\lim_{\bar{T}\to\infty}H(\bar{T})=0$ and the inequality follows from \eqref{eq:taubeta}. Therefore, for  all sufficiently large $\bar{T}$,
\begin{equation}\label{eq:pt2rect}
    \|(\btSigma_{\bar{T}+T_0}, \btOmega_{\bar{T}+T_0})-\left(G^{\Omega, \Sigma}\right)^{T_0}(\bx, \by)\|_{\xi, \beta^*}\le\tau^{T_0}\|(\btSigma_{\bar{T}}, \btOmega_{\bar{T}})-(\bx, \by)\|_{\xi, \beta^*}+\frac{\epsilon}{4}. 
\end{equation}
Note that $(\bx, \by) \in \mathcal X_{I^*_{\Sigma}}\times \mathcal X_{I_\Omega^*}$ implies that $\|\bx\|_\xi\le a_\Sigma$ and $\|\by\|_\xi\le a_\Omega$. In addition, by following the same argument as in Lemma \ref{lemma:imagerect}, one can show that $|\tilde{\omega}_{s, t}|\le a_\Omega$ and $|\tilde{\sigma}_{s, t}|\le a_\Sigma$, which in turn implies that $\|\btOmega_{\bar{T}}\|_\xi\le a_\Omega$ and $\|\btSigma_{\bar{T}}\|_\xi\le a_\Sigma$. As a result, we can make the RHS of \eqref{eq:pt2rect} smaller than $\epsilon/2$ by choosing a sufficiently large $T_0$. As the RHS of both \eqref{eq:pt1rect} and \eqref{eq:pt2rect} can be made smaller than $\epsilon/2$, an application of the triangle inequality gives that  
\begin{equation}
    \limsup_{\bar{T}\to\infty}\|(\btSigma_{\bar{T}}, \btOmega_{\bar{T}}) -(\bx^*, \by^*)\|_{\xi, \beta^*}\le \epsilon,
\end{equation}
which, after setting $\bar{T}=T+1$, implies the desired result.
\end{proof}

\subsection{Convergence to PCA Estimator for the First Phase}\label{app:firstestrect}

In this section, we prove that the artificial AMP iterate at the end of the first phase converges in normalized $\ell_2$-norm to the left singular vector produced by PCA.

\begin{lemma}[Convergence to PCA estimator  -- Rectangular matrices]\label{lemma:convrect}
Consider the setting of Theorem \ref{thm:rect}, and the first phase of the artificial AMP iteration described in \eqref{eq:AMPfake1rect}, with initialization given by \eqref{eq:AMPfake_init_rect}. Assume that $\kappa_{2i}^\infty\ge 0$ for all $i\ge 1$, and that $\tilde{\alpha}>\tilde{\alpha}_{\rm s}$. Then, 
\begin{equation}
\lim_{T \to \infty} \lim_{n \to \infty} \frac{1}{\sqrt{m}} \|\tilde{\bu}^{T+1}-\sqrt{m}\bu_{\rm PCA}\| =0\,\,\, \text{ a.s.}
     \label{eq:xstxT_diffrect}
\end{equation}

\end{lemma}

\begin{proof}
Consider the following decomposition of $\tilde{\bu}^{T+1}$:
\begin{equation}\label{eq:dcrect}
    \tilde{\bu}^{T+1} = \zeta_{T+1} \bu_{\rm PCA}+\br^{T+1},
\end{equation}
where $\zeta_{T+1} = \langle \tilde{\bu}^{T+1}, \bu_{\rm PCA}\rangle$ and $\langle \br^{T+1}, \bu_{\rm PCA}\rangle=0$. Define
\begin{equation}\label{eq:errrect}
    \be^{T+1} = \left(\bX\bX^\sT-\left(D^{-1}\left(1/\tilde{\alpha}^2\right)\right)^2 \bI_m\right)\tilde{\bu}^{T+1},
\end{equation}
where $D^{-1}$ is the inverse of the $D$-transform of $\Lambda$. Then, by using \eqref{eq:dcrect}, \eqref{eq:errrect} can be rewritten as
\begin{equation}
    \left(\bX\bX^\sT-\left(D^{-1}\left(1/\tilde{\alpha}^2\right)\right)^2 \bI_m\right)\br^{T+1} = \be^{T+1} - \left(\bX\bX^\sT-\left(D^{-1}\left(1/\tilde{\alpha}^2\right)\right)^2 \bI_m\right)\zeta_{T+1}\bu_{\rm PCA}.
\end{equation}
Note that $\bX$ (and consequently $\bX\bX^\sT$) has a spectral gap, in the sense that, almost surely, $\sigma_1(\bX)\to D^{-1}(1/\tilde{\alpha}^2)$ and $\sigma_2(\bX)\to b<D^{-1}(1/\tilde{\alpha}^2)$. Furthermore, $\br^{T+1}$ is orthogonal to the left singular vector associated to the singular value $\sigma_1(\bX)$. Thus, by following passages analogous to \eqref{eq:comprt}, \eqref{eq:comprt2} and \eqref{eq:Lambda}, we obtain that 
\begin{equation}\label{eq:lb1lemmarect}
\left\|    \left(\bX\bX^\sT-\left(D^{-1}\left(1/\tilde{\alpha}^2\right)\right)^2 \bI_m\right)\br^{T+1}\right\|\ge c\|\br^{T+1}\|,
\end{equation}
where $c>0$ is a constant (independent of $n, m, T$).

Next, we prove that almost surely
\begin{equation}\label{eq:lastlimrect}
    \lim_{T\to\infty}\lim_{n\to\infty}\frac{1}{\sqrt{m}}\left\|\be^{T+1} - \left(\bX\bX^\sT-\left(D^{-1}\left(1/\tilde{\alpha}^2\right)\right)^2 \bI_m\right)\zeta_{T+1}\bu_{\rm PCA}\right\|=0.
\end{equation}
An application of the triangle inequality gives that 
\begin{equation}\label{eq:pt2rerect}
\begin{split}
&\left\|\be^{T+1} - \left(\bX\bX^\sT-\left(D^{-1}\left(1/\tilde{\alpha}^2\right)\right)^2 \bI_m\right)\zeta_{T+1}\bu_{\rm PCA}\right\|\\
&\hspace{3em}\le \left\|\be^{T+1}\right\|+\left\| \left(\bX\bX^\sT-\left(D^{-1}\left(1/\tilde{\alpha}^2\right)\right)^2 \bI_m\right)\zeta_{T+1}\bu_{\rm PCA}\right\|.
\end{split}
\end{equation}
The second term on the RHS of \eqref{eq:pt2rerect} is equal to 
\begin{equation}
    |\zeta_{T+1}| \left|\lambda_1(\bX\bX^\sT)-\left(D^{-1}\left(1/\tilde{\alpha}^2\right)\right)^2\right|.
\end{equation}
By using Theorem 2.8 of \cite{benaych2012singular}, we have that, for $\tilde{\alpha}>\tilde{\alpha}_{\rm s}$, almost surely,
\begin{equation}\label{eq:0limrect}
   \lim_{m\to\infty} \left|\lambda_1(\bX\bX^\sT)-\left(D^{-1}\left(1/\tilde{\alpha}^2\right)\right)^2\right|=0.
\end{equation}
Furthermore,
\begin{equation*}
    \frac{1}{\sqrt{m}}|\zeta_{T+1}|\le \frac{1}{\sqrt{m}}\|\tilde{\bu}^{T+1}\|=\frac{1}{\alpha\sqrt{m}}\|\tilde{\bdff}^{T}\|.
\end{equation*}
By Proposition \ref{prop:tilSErect}, we have that
\begin{equation*}
    \lim_{m\to\infty}\frac{1}{\alpha\sqrt{m}}\|\tilde{\bdff}^{T}\| = \frac{1}{\alpha}\sqrt{\tilde{\mu}_{T}^2+\tilde{\sigma}_{T, T}},
\end{equation*}
which, for sufficiently large $T$, is upper bounded by a constant independent of $n, m, T$, as $\tilde{\mu}_{T}=\alpha\sqrt{\Delta_{\rm PCA}}$ and $\tilde{\sigma}_{T, T}$ converges to $\alpha^2(1-\Delta_{\rm PCA})$ as $T\to\infty$ by Lemma \ref{lem:SE_FP_phase1rect}. By combining this result with \eqref{eq:0limrect}, we deduce that 
\begin{equation}\label{eq:scdrect}
   \lim_{T\to\infty}\lim_{m\to\infty}\frac{1}{\sqrt{m}}  \left\| \left(\bX\bX^\sT-\left(D^{-1}\left(1/\tilde{\alpha}^2\right)\right)^2 \bI_m\right)\zeta_{T+1}\bu_{\rm PCA}\right\|=0.
\end{equation}
In order to bound the first term on the RHS of \eqref{eq:pt2rerect}, we proceed as follows:
\begin{equation}\label{eq:etbdrect}
    \begin{split}
&\lim_{m\to\infty}        \frac{1}{m}\|\be^{T+1}\|^2 =\lim_{m\to\infty} \frac{1}{m}\left\|\left(\bX\bX^\sT-\left(D^{-1}\left(1/\tilde{\alpha}^2\right)\right)^2 \bI_m\right)\tilde{\bu}^{T+1}\right\|^2\\
        &\stackrel{\mathclap{\mbox{\footnotesize (a)}}}{=} \lim_{m\to\infty}\frac{1}{m}\bigg\|\tilde{\alpha}^2\Bigg(\frac{1}{\alpha} \tilde{\bdff}^{T+1}+\frac{1}{\tilde{\alpha}^2}\sum_{i=1}^{T+1} \kappa_{2(T-i+2)}\left(\frac{1}{\tilde{\alpha}^2}\right)^{T-i+1}\tilde{\bu}^i+\frac{\gamma}{\tilde{\alpha}^2}\sum_{i=1}^{T}\kappa_{2(T-i+1)} \left(\frac{1}{\tilde{\alpha}^2}\right)^{T-i}\\
        &\hspace{4em}\cdot \bigg(\tilde{\bu}^{i+1}+\frac{1}{\tilde{\alpha}^2}\sum_{j=1}^i \kappa_{2(i-j+1)}\left(\frac{1}{\tilde{\alpha}^2}\right)^{i-j}\tilde{\bu}^j\bigg)\Bigg) -\left(D^{-1}\left(1/\tilde{\alpha}^2\right)\right)^2\tilde{\bu}^{T+1}\bigg\|^2\\
        &\stackrel{\mathclap{\mbox{\footnotesize (b)}}}{=} \lim_{m\to\infty}\frac{1}{m}\bigg\|\tilde{\alpha}^2\Bigg( \frac{1}{\alpha}\tilde{\bdff}^{T+1}+\frac{1}{\tilde{\alpha}^2}\sum_{i=1}^{T+1} \kappa_{2(T-i+2)}^\infty\left(\frac{1}{\tilde{\alpha}^2}\right)^{T-i+1}\tilde{\bu}^i+\frac{\gamma}{\tilde{\alpha}^2}\sum_{i=1}^{T}\kappa_{2(T-i+1)}^\infty \left(\frac{1}{\tilde{\alpha}^2}\right)^{T-i}\\
        &\hspace{4em}\cdot \bigg(\tilde{\bu}^{i+1}+\frac{1}{\tilde{\alpha}^2}\sum_{j=1}^i \kappa_{2(i-j+1)}^\infty\left(\frac{1}{\tilde{\alpha}^2}\right)^{i-j}\tilde{\bu}^j\bigg)\Bigg) -\left(D^{-1}\left(1/\tilde{\alpha}^2\right)\right)^2\tilde{\bu}^{T+1}\bigg\|^2\\
        &\stackrel{\mathclap{\mbox{\footnotesize (c)}}}{=} \mathbb E\Bigg\{\bigg |\tilde{\alpha}^2\Bigg( \frac{1}{\alpha}\tilde{F}_{T+1}+\frac{1}{\tilde{\alpha}^2}\sum_{i=1}^{T+1} \kappa_{2(T-i+2)}^\infty\left(\frac{1}{\tilde{\alpha}^2}\right)^{T-i+1}\tilde{U}_i+\frac{\gamma}{\tilde{\alpha}^2}\sum_{i=1}^{T}\kappa_{2(T-i+1)}^\infty \left(\frac{1}{\tilde{\alpha}^2}\right)^{T-i}\\
        &\hspace{4em}\cdot \bigg(\tilde{U}_{i+1}+\frac{1}{\tilde{\alpha}^2}\sum_{j=1}^i \kappa_{2(i-j+1)}^\infty\left(\frac{1}{\tilde{\alpha}^2}\right)^{i-j}\tilde{U}_j\bigg)\Bigg) -\left(D^{-1}\left(1/\tilde{\alpha}^2\right)\right)^2\tilde{U}_{T+1}\bigg|^2\Bigg\}.
    \end{split}
\end{equation}
Here, (a) uses the iteration \eqref{eq:AMPfake1rect} of the first phase of the artificial AMP; (b) uses that, for all $i$, $\kappa_{2i}\to\kappa_{2i}^\infty$ as $n\to\infty$, as well as an argument similar to \eqref{eq:kappj_kappj_inf}-\eqref{eq:uiuj_inner_prod}; and (c) follows from Proposition \ref{prop:tilSErect}, where $\tilde{U}_t$ for $t\in [1, T+1]$ and $\tilde{F}_{T+1}$ are defined in \eqref{eq:F0FKtilde_rect} and \eqref{eq:Utilde_rect}. After some manipulations we can upper bound the RHS of \eqref{eq:etbdrect} by triangle inequality as
\begin{equation}\label{eq:etbd2rect}
\begin{split}
    5 \cdot &\E\Bigg\{\bigg(\tilde{\alpha}^2+\sum_{i=0}^{T}\kappa_{2(i+1)}^\infty\left(\frac{1}{\tilde{\alpha}^2}\right)^i+\gamma\sum_{i=0}^{T-1}\kappa_{2(i+1)}^\infty\left(\frac{1}{\tilde{\alpha}^2}\right)^i+\gamma \sum_{i=1}^{T}\sum_{j=0}^{T-i}\kappa_{2i}^\infty\kappa_{2(j+1)}^\infty\left(\frac{1}{\tilde{\alpha}^2}\right)^{i+j}\\
&\hspace{24em}    -\left(D^{-1}\left(1/\tilde{\alpha}^2\right)\right)^2\bigg)^2\tilde{U}_{T+1}^2\Bigg\} \\
    &+ 5\cdot\E\left\{\left(\gamma\sum_{i=1}^{T}\sum_{j=0}^{T-i} \kappa_{2i}^\infty\kappa_{2(j+1)}^\infty\left(\frac{1}{\tilde{\alpha}^2}\right)^{i+j}(\tilde{U}_{T-i-j+1}-\tilde{U}_{T+1})\right)^2\right\}\\
    &+ 5\cdot\E\left\{\left(\gamma\sum_{i=0}^{T-1} \kappa_{2(i+1)}^\infty\left(\frac{1}{\tilde{\alpha}^2}\right)^{i}(\tilde{U}_{T-i+1}-\tilde{U}_{T+1})\right)^2\right\}\\
    &+ 5\cdot\E\left\{\left(\sum_{i=0}^{T} \kappa_{2(i+1)}^\infty\left(\frac{1}{\tilde{\alpha}^2}\right)^{i}(\tilde{U}_{T-i+1}-\tilde{U}_{T+1})\right)^2\right\}\\
    &+ 5\cdot\E\left\{\tilde{\alpha}^4\left(\frac{1}{\alpha}\tilde{F}_{T+1}-\tilde{U}_{T+1}\right)^2\right\}:= S_1 + S_2 + S_3+ S_4+ S_5.
\end{split}
\end{equation}
The term $S_5$ can be expressed as
\begin{equation*}
        S_5 = 5\frac{\alpha^2}{\gamma^2} (\tilde{\sigma}_{T+1, T+1}-2\tilde{\sigma}_{T+1, T}+\tilde{\sigma}_{T, T}).
\end{equation*}
Thus, by Lemma \ref{lem:SE_FP_phase1rect}, we have that
\begin{equation}\label{eq:S5limrect}
    \lim_{T\to\infty} S_5 = 0.
\end{equation}
The term $S_1$ can be expressed as 
\begin{equation*}
        \begin{split}
& S_1 =5 \,  \frac{\tilde{\mu}_{T}^2+\tilde{\sigma}_{T, T}}{\alpha^2}\cdot \bigg(\tilde{\alpha}^2+\sum_{i=0}^{T}\kappa_{2(i+1)}^\infty\left(\frac{1}{\tilde{\alpha}^2}\right)^i+\gamma\sum_{i=0}^{T-1}\kappa_{2(i+1)}^\infty\left(\frac{1}{\tilde{\alpha}^2}\right)^i\\
&\hspace{10em}+\gamma \sum_{i=1}^{T}\sum_{j=0}^{T-i}\kappa_{2i}^\infty\kappa_{2(j+1)}^\infty\left(\frac{1}{\tilde{\alpha}^2}\right)^{i+j}
    -\left(D^{-1}\left(1/\tilde{\alpha}^2\right)\right)^2\bigg)^2.
\end{split}
\end{equation*}
Thus, by Lemma \ref{lem:SE_FP_phase1rect}, we have that
\begin{equation}\label{eq:S1limrect}
\begin{split}
 &   \lim_{T\to\infty} S_1 =5\cdot \bigg(\tilde{\alpha}^2+\sum_{i=0}^{\infty}\kappa_{2(i+1)}^\infty\left(\frac{1}{\tilde{\alpha}^2}\right)^i+\gamma\sum_{i=0}^{\infty}\kappa_{2(i+1)}^\infty\left(\frac{1}{\tilde{\alpha}^2}\right)^i\\
&\hspace{10em}+\gamma \sum_{i=1}^{\infty}\sum_{j=0}^{\infty}\kappa_{2i}^\infty\kappa_{2(j+1)}^\infty\left(\frac{1}{\tilde{\alpha}^2}\right)^{i+j}
    -\left(D^{-1}\left(1/\tilde{\alpha}^2\right)\right)^2\bigg)^2=0,
\end{split}
\end{equation}
where the last equality follows from  \eqref{eq:Rrect} and \eqref{eq:RD}.
The term $S_4$ can be expressed as 
\begin{equation*}
\begin{split}
S_4 =\frac{5}{\alpha^2}\sum_{i, j=0}^{T}\kappa_{2(i+1)}^\infty\kappa_{2(j+1)}^\infty \left(\frac{1}{\tilde{\alpha}^2}\right)^{i+j}\big(&\tilde{\sigma}_{T-j, T-i}+\tilde{\sigma}_{T, T}
-\tilde{\sigma}_{T, T-i}-\tilde{\sigma}_{T, T-j}\big),
\end{split}
\end{equation*}
which can upper bounded by 
\begin{equation}\label{eq:S1rerect}
\begin{split}
    &\frac{5}{\alpha^2}\sum_{i, j=0}^{T}\kappa_{2(i+1)}^\infty\kappa_{2(j+1)}^\infty \left(\frac{1}{\tilde{\alpha}^2}\right)^{i+j} \\
    &\hspace{2em}\big(|\tilde{\sigma}_{T-j, T-i}-\alpha^2(1-\Delta_{\rm PCA})|+|\tilde{\sigma}_{T, T}-\alpha^2(1-\Delta_{\rm PCA})|\\
    &\hspace{4em}+|\tilde{\sigma}_{T, T-i}-\alpha^2(1-\Delta_{\rm PCA})|+|\tilde{\sigma}_{T, T-j}-\alpha^2(1-\Delta_{\rm PCA})|\big).
\end{split}
\end{equation}
By Lemma \ref{lem:SE_FP_phase1rect}, for any $\epsilon>0$, there exists $T^*(\epsilon)$ such that for all $T>T^*(\epsilon)$, the quantity in \eqref{eq:S1rerect} is upper bounded by 
\begin{equation*}
\begin{split}
\epsilon &\cdot\frac{5}{\alpha^2}\sum_{i, j=0}^{T}\kappa_{2(i+1)}^\infty\kappa_{2(j+1)}^\infty \left(\frac{1}{\tilde{\alpha}^2}\right)^{i+j}\cdot  \big(\xi^{-\max(i, j)}+1+\xi^{-i}+\xi^{-j}\big)\\
&\stackrel{\mathclap{\mbox{\footnotesize (a)}}}{\le}\epsilon \cdot\frac{20}{\alpha^2}\sum_{i, j=0}^{T}\kappa_{2(i+1)}^\infty\kappa_{2(j+1)}^\infty \left(\frac{1}{\xi\tilde{\alpha}^2}\right)^{i+j}\\
&\stackrel{\mathclap{\mbox{\footnotesize (b)}}}{\le}\epsilon \cdot\frac{20}{\alpha^2}\sum_{i, j=0}^{\infty}\kappa_{2(i+1)}^\infty\kappa_{2(j+1)}^\infty \left(\frac{1}{\xi\tilde{\alpha}^2}\right)^{i+j}\\
&\stackrel{\mathclap{\mbox{\footnotesize (c)}}}{\le}\epsilon \cdot\frac{20}{\alpha^2}\left(R\left(\frac{1}{\xi\tilde{\alpha}^2}\right)\right)^2.
\end{split}
\end{equation*}
Here, (a) uses that $\xi<1$, (b) uses that $\kappa_{2i}\ge 0$ for $i\ge 1$, and (c) uses the power series expansion \eqref{eq:Rrect} of $R$, which converges to a finite limit as $\sqrt{\xi}\tilde{\alpha}>\tilde{\alpha}_{\rm s}$. Since $\epsilon$ can be taken arbitrarily small, we deduce that 
\begin{equation}\label{eq:S4limrect}
    \lim_{T\to\infty} S_4 = 0.
\end{equation}
By using the same argument, we also have that 
\begin{equation}\label{eq:S3limrect}
    \lim_{T\to\infty} S_3 = 0.
\end{equation}
Finally, the term $S_2$ is upper bounded by
\begin{equation}\label{eq:S1larerect}
    \begin{split}
        \frac{5\gamma^2}{\alpha^2}&\sum_{i=1}^{T}\sum_{j=0}^{T-i}\sum_{k=1}^{T}\sum_{\ell=0}^{T-k}\kappa^\infty_{2i}\kappa^\infty_{2k}\kappa^\infty_{2(j+1)}\kappa^\infty_{2(\ell+1)}\left(\frac{1}{\tilde{\alpha}^2}\right)^{i+j+k+\ell}\\
& \cdot     \big(|\sigma_{T, T}-\alpha^2(1-\Delta_{\rm PCA})|+|\sigma_{T, T-i-j}-\alpha^2(1-\Delta_{\rm PCA})|\\
&+|\sigma_{T, T-k-\ell}-\alpha^2(1-\Delta_{\rm PCA})|+|\sigma_{T-i-j, T-k-\ell}-\alpha^2(1-\Delta_{\rm PCA})|\big).
    \end{split}
\end{equation}
By Lemma \ref{lem:SE_FP_phase1rect}, for any $\epsilon>0$, there exists $T^*(\epsilon)$ such that for all $T>T^*(\epsilon)$, the quantity in \eqref{eq:S1larerect} is upper bounded by
\begin{equation*}
    \begin{split}
    \epsilon&\cdot     \frac{20\gamma^2}{\alpha^2}\sum_{i=1}^{T}\sum_{j=0}^{T-i}\sum_{k=1}^{T}\sum_{\ell=0}^{T-k}\kappa^\infty_{2i}\kappa^\infty_{2k}\kappa^\infty_{2(j+1)}\kappa^\infty_{2(\ell+1)}\left(\frac{1}{\xi\tilde{\alpha}^2}\right)^{i+j+k+\ell}\\
&\le \epsilon\cdot     \frac{20\gamma^2}{\alpha^2}\sum_{i=1}^{\infty}\sum_{j=0}^{\infty}\sum_{k=1}^{\infty}\sum_{\ell=0}^{\infty}\kappa^\infty_{2i}\kappa^\infty_{2k}\kappa^\infty_{2(j+1)}\kappa^\infty_{2(\ell+1)}\left(\frac{1}{\xi\tilde{\alpha}^2}\right)^{i+j+k+\ell}\\
&\le \epsilon\cdot    \frac{20\gamma^2}{\alpha^2}\left(R\left(\frac{1}{\xi\tilde{\alpha}^2}\right)\right)^4,
    \end{split}
\end{equation*}
where we use again that $\kappa_{2i}\ge 0$ for $i\ge 1$ and the power series expansion \eqref{eq:Rrect} of $R$. Since $\epsilon$ can be taken arbitrarily small, we deduce that 
\begin{equation}\label{eq:S2limrect}
    \lim_{T\to\infty} S_2 = 0.
\end{equation}
By combining \eqref{eq:etbdrect}, \eqref{eq:etbd2rect}, \eqref{eq:S5limrect}, \eqref{eq:S1limrect}, \eqref{eq:S4limrect}, \eqref{eq:S3limrect} and \eqref{eq:S2limrect}, we conclude that 
\begin{equation}
   \lim_{T\to\infty}\lim_{m\to\infty}\frac{1}{\sqrt{m}}  \left\|\be^{T+1}\right\|=0,
\end{equation}
which, combined with \eqref{eq:scdrect}, gives \eqref{eq:lastlimrect}. Finally, by using \eqref{eq:lb1lemmarect} and \eqref{eq:lastlimrect}, we have that
\begin{equation}
    \lim_{T\to\infty}\lim_{m\to\infty}\frac{1}{\sqrt{m}}  \left\|\br^{T+1}\right\|=0. 
\end{equation}
Thus, from the decomposition \eqref{eq:dcrect}, we conclude that, as $m\to\infty$ and $T\to\infty$, $\tilde{\bu}^{T+1}$ is aligned with $\bu_{\rm PCA}$. Furthermore, from another application of Proposition \ref{prop:tilSErect}, we obtain
\begin{equation}
 \lim_{T\to\infty}\lim_{m\to\infty}\frac{1}{\sqrt{m}}\|\tilde{\bu}^{T+1}\| = \lim_{T\to\infty}\frac{1}{\alpha}\sqrt{\tilde{\mu}_{T}^2+\tilde{\sigma}_{T, T}} = 1, 
\end{equation}
which implies that $\lim_{T\to\infty}\lim_{m\to\infty}\zeta_{T+1} = 1$ and concludes the proof.  
\end{proof}

\subsection{Analysis for the Second Phase}\label{sec:app_sec_phase_analysis_rect}

As in the proof of the square case, we define a  modified version of the true AMP algorithm, in which the memory coefficients $\{\sa_{t,i}, \sb_{t+1,i}\}_{i \in [1, t]}$ are replaced by deterministic values obtained from state evolution.  This modified AMP is initialized with
\begin{align}
\label{eq:AMP_rect_init_mod}
    \hbu^1 = \sqrt{m} \, \bu_{\rm PCA}, \quad 
    \hbg^1 = \Bigg( 1 + \gamma \sum_{i=1}^\infty \kappa_{2i}^\infty \Big( \frac{\gamma}{\alpha^2}\Big)^i\Bigg)^{-1}\hspace{-5pt}\bX^{\sT}\hbu^1, \,  \,  \quad \hbv^1 = v_1(\hbg^1)=\frac{\gamma}{\alpha} \hbg^1.
\end{align}
Then, for $t \geq 1$, we iteratively compute:
\begin{align}
    \hat{\bdff}^t \hspace{-.15em}=\hspace{-.15em}\bX \hbv^t \hspace{-.15em} - \hspace{-.15em} \sum_{i=1}^t  \bar{\sa}_{t,i} \hbu^i, \hspace{.95em}\hbu^{t+1}\hspace{-.15em}=\hspace{-.15em}\su_{t+1}(\hbf^t), \hspace{.95em}      \hbg^{t+1}\hspace{-.15em} =\hspace{-.15em}\bX \hbu^{t+1} \hspace{-.15em} - \hspace{-.15em} \sum_{i=1}^{t}  \bar{\sb}_{t+1,i} \hbv^i, \hspace{.95em} \hbv^{t+1}\hspace{-.15em}=\hspace{-.15em} \sv_{t+1}(\hbg^{t+1}).
     \label{eq:AMP_rectg_mod}
\end{align}
 The deterministic memory coefficients are: $\bar{\sa}_{1,1} =\alpha \sum_{i=1}^\infty \kappa_{2i}^\infty \big( \frac{\gamma}{\alpha^2}\big)^i$, and for $t \geq 2$:
\begin{align}
    & \bar{\sa}_{t,1} = \E\{  \sv_t'(G_t) \} \prod_{i=2}^t 
    \E\{ \su_i'(F_{i-1}) \}
    \E\{  \sv_{i-1}'(G_{i-1}) \} \left( \sum_{i=0}^\infty \kappa_{2(i+t)}^{\infty} \Big( \frac{\gamma}{\alpha^2}\Big)^i \right), \\
    & \bar{\sa}_{t,t-j} = \E\{  \sv_t'(G_t) \} \prod_{i=t-j+1}^t   \E\{ \su_i'(F_{i-1})\} \E\{  \sv_{i-1}'(G_{i-1}) \} \kappa_{2(j+1)}^{\infty}, \qquad \mbox{ for }\,\, (t-j) \in [2,t].
\end{align}
Furthermore, for $t \geq 1$,
\begin{align}
   &  \bar{\sb}_{t+1,1} =  \gamma \E\{ \su'_{t+1}(F_t) \} \prod_{i=2}^{t} \E\{ \sv_i'(G_i) \} \E\{ \su_i'(F_{i-1})\} 
   \left( \kappa_{2t}^{\infty} \, +  \, \sum_{i=1}^\infty \kappa_{2(i+t)}^{\infty} \Big( \frac{\gamma}{\alpha^2}\Big)^i\right), \\
   & \bar{\sb}_{t+1, t+1-j} = \gamma \E\{ \su'_{t+1}(F_t) \} \prod_{i=t+2-j}^{t} 
   \E\{ \sv_i'(G_i) \} \E \{ \su_i'(F_{i-1}) \}  \, \kappa_{2j}^{\infty}, \qquad \mbox{ for }\,\,(t+1-j) \in [2, t].
\end{align}
We recall that $\{ \kappa_{2i}^\infty \}$ are the rectangular free cumulants of the limiting singular value distribution $\Lambda$, and the random variables $\{F_i, G_i\}$ are given by \eqref{eq:F0FK_rect}-\eqref{eq:G0GK_rect}.  The following lemma shows that, as $T$ grows, the iterates of the second phase of the artificial AMP  (described in Section \ref{subsec:proof_sketch_rect}) approach those of the modified AMP algorithm above, as do the corresponding state evolution parameters.

\begin{lemma}\label{lem:sec_phase_rect}
Consider the setting of Theorem \ref{thm:rect}. Assume that $\kappa_{2i}^\infty\ge 0$ for all $i\ge 1$, and that $\tilde{\alpha}>\tilde{\alpha}_{\rm s}$. Consider the modified version of the true AMP in \eqref{eq:AMP_rect_init_mod}-\eqref{eq:AMP_rectg_mod}, and the artificial AMP in \eqref{eq:AMPfake_init_rect}, \eqref{eq:AMPfake1rect}, and \eqref{eq:AMPfake2rect}  along with its state evolution recursion given by \eqref{eq:SEfakeinitrect}-\eqref{eq:tsigma_update}. Then, the following results hold for $s,t \ge 1$:
 \begin{enumerate}
\item 
\begin{align}
    \lim_{T \to \infty} \tilde{\mu}_{T+t} = \mu_t, \qquad 
    \lim_{T \to \infty} \tilde{\sigma}_{T+s, T+t} = \sigma_{s,t}, 
    \label{eq:fake_true_SEconv_rect1} \\
    \lim_{T \to \infty} \tilde{\nu}_{T+t} = \nu_t, \qquad 
    \lim_{T \to \infty} \tilde{\omega}_{T+s, T+t} = \omega_{s,t}, 
    \label{eq:fake_true_SEconv_rect2}
\end{align} 
\item For any $\PL(2)$ functions $\psi:\reals^{2t+2} \to \reals$ and $\varphi: \reals^{2t+1} \to \reals$, we almost surely have:
\beq
\begin{split}
& \lim_{T \to \infty} \lim_{n \to \infty} 
 \bigg\vert  \frac{1}{m}  \sum_{i=1}^m \psi (u^*_{i}, \tu^{T+1}_i, \ldots, \tu^{T+t+1}_i, \tf^{T+1}_i, \ldots \tf^{T+t}_i)\\
&\hspace{10em} - \, 
\frac{1}{m}  \sum_{i=1}^m \psi (u^*_{i}, \hu^1_i, \ldots, \hu^{t+1}_i, \hf^1_i, \ldots \hf^{t}_i)  \bigg\vert  =0,
\end{split}
\label{eq:fake_modified_conv_uf}
\eeq
\beq
\begin{split}
& \lim_{T \to \infty} \lim_{n \to \infty} 
 \bigg\vert  \frac{1}{n}  \sum_{i=1}^n \varphi (v^*_{i}, \tv^{T+1}_i, \ldots, \tv^{T+t}_i, \tg^{T+1}_i, \ldots \tg^{T+t}_i) \\ 
 &\hspace{10em} - \, 
\frac{1}{n}  \sum_{i=1}^n \varphi (v^*_{i}, \hv^1_i, \ldots, \hv^{t}_i, \hg^1_i, \ldots \hg^{t}_i)  \bigg\vert  =0.
\end{split}
\label{eq:fake_modified_conv_vg}
\eeq
\end{enumerate}
\end{lemma}

\begin{proof}
\textbf{Proof of \eqref{eq:fake_true_SEconv_rect1}-
\eqref{eq:fake_true_SEconv_rect2}}. For $t \in [1, T+1]$, from \eqref{eq:SEfake1rect} we have $\tmu_t = \tnu_{t} = \alpha \sqrt{\Delta_{\PCA}} = \mu_1 = \nu_1$.  Next, Lemma \ref{lemma:fixedrect} shows that $\lim_{T \to \infty} \tsigma_{T+1, T+1} = a^*$ and $\lim_{T \to \infty} \tomega_{T+1, T+1}=b^*$, where $a^*, b^*$ are defined in \eqref{eq:asbs}. We now verify that $\sigma_{11} = a^*$ and $\omega_{11} =b^*$. Setting $s=t=0$ in  \eqref{eq:omega_rect} and solving for $\omega_{11}$, we obtain:
\beq
\omega_{1,1} = b^* = \frac{\Delta_{\PCA} \gamma\alpha^2 (x R'(x) - R(x))  \, + \, \gamma R'(x)}{1 + \gamma R(x) - \gamma x R'(x)},\  \text{ where }  \ x=\frac{\gamma}{\alpha^2}. 
\label{eq:omega_11}
\eeq
Here, we have used \eqref{eq:R1rect} and \eqref{eq:R20rect} to express the double sums in terms of $R(x)$ and $R'(x)$.
Similarly, from \eqref{eq:sigma_rect}, we obtain
\beq
\sigma_{1,1} = \gamma x R'(x) (\alpha^2 \Delta_{\PCA} + \omega_{1,1}) \, + \, \gamma R'(x) - \alpha^2 R(x), \  \text{ where }  \ x= \gamma/\alpha^2.
\eeq
Using the formula for $\Delta_{\PCA}$ in \cite[Eq. (7.32)]{fan2020approximate}, it can be verified that the above expression for $\sigma_{1,1}$ reduces to  $a^* = \alpha^2(1-\Delta_{\PCA})$, as required. 

Assume towards induction that  the following holds for $ 1 \le k, \ell \le t$:
\beq
\lim_{T \to \infty} \tmu_{T+\ell}= \mu_\ell, \quad 
\lim_{T \to \infty} \tsigma_{T+k, T+\ell}= 
\sigma_{k, \ell},  \quad 
\lim_{T \to \infty} \tnu_{T+\ell}= \nu_{\ell}, \quad 
\lim_{T \to \infty} \tomega_{T+k, T+\ell}= \omega_{k,\ell}.
\label{eq:ind_hyp_SE_rect}
\eeq
Consider $\tnu_{T+t+1} = \alpha \E\{\tU_{T+t+1} U_* \} = 
\alpha \E\{ \su_{t+1}(\tF_{T+t}) U_* \}$. By the induction hypothesis $\tF_{T+t} = \tmu_{T+t} U_*  +  \tY_{T+t}$ converges in distribution to $F_t= \mu_tU_* + Y_t$, and by arguments similar to \eqref{eq:UI_utU},  the sequence of random variables $\{ \su_{t+1}(\tF_{T+t}) U_*\}$ is uniformly integrable. Hence, 
\beq
\lim_{T \to \infty} \tnu_{T+t+1} = 
\alpha \E\{ \su_{t+1}( F_{t}) U_* \} = \nu_{t+1}.
\label{eq:nu_t1_conv}
\eeq
Next, for $s \le t$,  consider $\tomega_{T+s+1, T+t+1}$ which is defined via \eqref{eq:tomega_update}. We write  $\tomega_{T+s+1, T+t+1} = O_1 + O_2 + O_3 + O_4$ , where
\begin{align}
    O_1  & = \gamma \sum_{j=0}^{s-1} \sum_{k=0}^{t-1} 
    \Big( \prod_{i=s-j+2}^{s+1}  \hspace{-5pt} \E\{ \su_i'(\tF_{T+i-1})\} \E\{ \sv_{i-1}'(\tG_{T+i-1})\} \Big)  \nonumber \\ 
    & \ \ \Big( \prod_{i=t-k+2}^{t+1}  \hspace{-5pt} \E\{ \su_i'(\tF_{T+i-1})\} \E\{ \sv_{i-1}'(\tG_{T+i-1})\} \Big)  \cdot \Big[ \kappa_{2(j+ k+1)}^\infty \E\{ \tU_{T+s+1-j} \tU_{T+t+1-k} \}  \nonumber   \\
    &  \quad + \kappa_{2(j+k+2)}^\infty \E\{ \su'_{s+1-j}(\tF_{T+s-j}) \} \, \E\{\su'_{t+1-k}(\tF_{T+t-k}) \}  \E\{ \tV_{T+s-j} \tV_{T+t-k} \} \Big],
     \label{eq:O1_def_rect}
\end{align}
\begin{align}
    O_2 & = \gamma \sum_{j=0}^{s-1} \sum_{k=t}^{T+t} 
    \left(\frac{\gamma}{\alpha^2} \right)^{k-t} \Big( \prod_{i=s-j+2}^{s+1}  \hspace{-5pt} \E\{ \su_i'(\tF_{T+i-1})\} \E\{ \sv_{i-1}'(\tG_{T+i-1})\} \Big)  \nonumber \\ 
    & \ \ \Big( \prod_{i=2}^{t+1}    \E\{ \su_i'(\tF_{T+i-1})\} \E\{ \sv_{i-1}'(\tG_{T+i-1})\} \Big)  \cdot \Big[ \kappa_{2(j+ k+1)}^\infty \E\{ \tU_{T+s+1-j} \tU_{T+t+1-k} \}  \nonumber   \\
    &  \quad + \kappa_{2(j+k+2)}^\infty \frac{1}{\alpha} \, \E\{ \su'_{s+1-j}(\tF_{T+s-j}) \}  
     \E\{ \tV_{T+s-j} \tV_{T+t-k} \} \Big],
      \label{eq:O2_def_rect}
\end{align}
\begin{align}
    O_3 & = \gamma \sum_{j=s}^{T+s} \sum_{k=0}^{t-1} 
    \left(\frac{\gamma}{\alpha^2} \right)^{j-s} \Big( \prod_{i=2}^{s+1}   \E\{ \su_i'(\tF_{T+i-1})\} \E\{ \sv_{i-1}'(\tG_{T+i-1})\} \Big)  \nonumber \\ 
    & \ \ \Big( \prod_{i=t-k+2}^{t+1}    \E\{ \su_i'(\tF_{T+i-1})\} \E\{ \sv_{i-1}'(\tG_{T+i-1})\} \Big)  \cdot \Big[ \kappa_{2(j+ k+1)}^\infty \E\{ \tU_{T+s+1-j} \tU_{T+t+1-k} \}  \nonumber   \\
    &  \quad + \kappa_{2(j+k+2)}^\infty \frac{1}{\alpha} \, 
    \E\{ \su'_{t+1-k}(\tF_{T+t-k}) \}  
     \E\{ \tV_{T+s-j} \tV_{T+t-k} \} \Big],
      \label{eq:O3_def_rect}
\end{align}
\begin{align}
    O_4 & = \gamma \sum_{j=s}^{T+s} \sum_{k=t}^{T+t} 
    \left(\frac{\gamma}{\alpha^2} \right)^{j+k -s -t} \Big( \prod_{i=2}^{s+1}  \E\{ \su_i'(\tF_{T+i-1})\} \E\{ \sv_{i-1}'(\tG_{T+i-1})\} \Big)  \nonumber \\ 
    & \ \ \Big( \prod_{i=2}^{t+1}    \E\{ \su_i'(\tF_{T+i-1})\} \E\{ \sv_{i-1}'(\tG_{T+i-1})\} \Big)  \cdot \Big[ \kappa_{2(j+ k+1)}^\infty \E\{ \tU_{T+s+1-j} \tU_{T+t+1-k} \}  \nonumber   \\
    &  \quad + \kappa_{2(j+k+2)}^\infty \frac{1}{\alpha^2}  \, \E\{ \tV_{T+s-j} \tV_{T+t-k} \} \Big].
      \label{eq:O4_def_rect}
\end{align}
By the induction hypothesis, for $i \in [2,t+1]$, we have  $\tF_{T+i-1} \stackrel{d}{\to} F_{i-1}$ and $\tG_{T+i-1} \stackrel{d}{\to} G_{i-1}$. Since $u_i$ and $v_{i-1}$ are Lipschitz and continuously differentiable, Lemma \ref{lem:lipderiv} implies
\beq
\begin{split}
 \lim_{T \to \infty} \E\{ \su_i'(\tF_{T+i-1})\} = 
\E\{ \su_i'(F_{i-1}) \}, \quad 
& \lim_{T \to \infty} \E\{ \sv_{i-1}'(\tG_{T+i-1})\} = 
\E\{ \sv_{i-1}'(G_{i-1}) \}, \\
& \qquad \text{ for } \ i\in [2, t+1].
\end{split}
\label{eq:uivi_SEconv}
\eeq
Next, note that 
\beq
\begin{split}
    (\tU_{T+s+1-j}, \, \tV_{T+s-j} ) = 
    \begin{cases} 
    (\su_{s+1-j}(\tF_{T+s-j}), \, \sv_{s-j}(\tG_{T+s-j}) ), & 0 \leq j \leq s-1, \\
    (\tF_{T+s-j}/\alpha, \, \tG_{T+s-j} \gamma/\alpha ), & s \leq j \leq T+s-1, \\
    (\tF_{0}/\alpha, \, 0 ), & j=T+s.
    \end{cases}
\end{split}
\label{eq:tUtVsj}
\eeq 
An analogous set of expressions  holds for the pair $ (\tU_{T+t+1-k}, \, \tV_{T+t-k})$.
For $j \in [0, s-1]$ and $k \in [0, t-1]$, using an argument similar to that used to obtain \eqref{eq:usuk_conv}, we deduce that the sequences $\{ \su_{s+1-j}(\tF_{T+s-j}) \su_{t+1-k}(\tF_{T+t-k}) \}$ and 
$\{ \sv_{s-j}(\tG_{T+s-j}) \sv_{t-k}(\tG_{T+t-k}) \}$ are each uniformly integrable. This, together with the induction hypothesis, implies that 
\begin{equation}
    \begin{split}
      \lim_{T \to \infty} O_1 &  =     \gamma \sum_{j=0}^{s-1} \sum_{k=0}^{t-1}      \Big( \prod_{i=s-j+2}^{s+1}  \hspace{-5pt} \E\{ \su_i'(F_{i-1})\} \E\{ \sv_{i-1}'(G_{i-1})\} \Big) \\
     & \qquad \Big( \prod_{i=t-k+2}^{t+1}  \hspace{-5pt} \E\{ \su_i'(F_{i-1})\} \E\{ \sv_{i-1}'(G_{i-1})\} \Big)  \cdot \Big[ \kappa_{2(j+ k+1)}^\infty \E\{ U_{s+1-j} U_{t+1-k} \}  \\
    &  \qquad  \qquad + \,  \kappa_{2(j+k+2)}^\infty \E\{ \su'_{s+1-j}(F_{s-j}) \} \, \E\{\su'_{t+1-k}(F_{t-k}) \}  \E\{ V_{s-j} V_{t-k} \} \Big].
    \end{split}
     \label{eq:O1_lim_rect}
\end{equation}
Next consider the term $O_4$. In this case, for  $j \in [s, T+s-1]$ and $k \in [t, T+t-1]$:
\begin{equation}
    \begin{split}
        & \E\{ \tU_{T+s+1-j} \tU_{T+t+1-k} \} = \frac{1}{\alpha^2} \E\{\tF_{T+s-j} \tF_{T+t-k} \} = \Delta_{\PCA} + \frac{1}{\alpha^2} \tsigma_{T- (j-s),\, T-(k-t)}, \\
        & \E\{ \tV_{T+s-j} \tV_{T+t-k} \} = \frac{\gamma^2}{\alpha^2}\E\{\tG_{T+s-j} \tG_{T+t-k} \}
        = \frac{\gamma^2}{\alpha^2} (\alpha^2\Delta_{\PCA}  +  \tomega_{T- (j-s),\, T-(k-t)}).
    \end{split}
    \label{eq:UVTsUVtk}
\end{equation}
When $j=T+s$ or $k=T+t$, the formula above for $\E\{ \tU_{T+s+1-j} \tU_{T+t+1-k} \}$ still holds, while the one for $\E\{ \tV_{T+s-j} \tV_{T+t-k} \}$ becomes 0 as $\tV_0 = 0$.
From Lemma \ref{lem:SE_FP_phase1rect}, for any $\delta >0$, for sufficiently large $T$ we have 
\beq
\begin{split}
    & |\tilde{\sigma}_{T+s-j, T+t-k}-a^*|  < \delta \xi^{-\max\{j+1-s, k+1-t \}}, \\
    & |\tilde{\omega}_{T+s-j, T+t-k}-b^*| < \delta \xi^{-\max\{j+1-s, k+1-t \}},
 \quad j\in [s, T+s], \ k \in [t, T+t],
    \label{eq:sig_omega_conv}
\end{split}
\eeq
for some $\xi >0$ such that $ \tilde{\alpha} \sqrt{\xi} > \tilde{\alpha}_{\rm s}$.   From \eqref{eq:F0FK_rect}-\eqref{eq:V0VK_rect}, we note that 
$\E\{ U_{s-j} U_{t-k} \} = \frac{1}{\alpha^2}\E\{ F_0^2 \} =1$ and 
$\E\{ V_{s-j} V_{t-k} \} = \frac{\gamma^2}{\alpha^2} \E\{ G_1^2 \} =
\frac{\gamma^2}{\alpha^2}(  \alpha^2 \Delta_{\PCA} + b^* )$. 
Combining this with  \eqref{eq:UVTsUVtk} and \eqref{eq:sig_omega_conv}, we have for sufficiently large $T$:
\begin{equation}\label{eq:sig_omega_conv2}
    \begin{split}
       & |\E\{ \tU_{T+1+s-j} \tU_{T+1+t-k} \} - \E\{ U_{s-j} U_{t-k} \}| < \frac{\delta}{\alpha^2} \xi^{-\max\{j+1-s, k+1-t \}}, \  
       \\
        &|\E\{ \tV_{T+s-j} \tV_{T+t-k} \} - \E\{ V_{s-j} V_{t-k} \}| < \frac{\gamma^2 \delta}{\alpha^2} \xi^{-\max\{j+1-s, k+1-t \}}, \  \text{ for } j \ge s, k \ge t.
    \end{split}
\end{equation}
We now write $O_4$ in \eqref{eq:O4_def_rect} as 
\begin{align}
         O_4 &   = 
\gamma \Big( \prod_{i=2}^{s+1}  \E\{ \su_i'(\tF_{T+i-1})\} \E\{ \sv_{i-1}'(\tG_{T+i-1})\} \Big)  \Big( \prod_{i=2}^{t+1}    \E\{ \su_i'(\tF_{T+i-1})\} \E\{ \sv_{i-1}'(\tG_{T+i-1})\} \Big) \nonumber  \\ 
    & \Bigg[\sum_{j=s}^{T+s} \sum_{k=t}^{T+t}      \left(\frac{\gamma}{\alpha^2} \right)^{j+k -s -t} \Big[ \kappa_{2(j+ k+1)}^\infty \E\{ U_{s+1-j} U_{t+1-k} \}   + \kappa_{2(j+k+2)}^\infty \frac{1}{\alpha^2}  \, \E\{ V_{s-j} V_{t-k}  \} \Big]  \nonumber \\
    & \hspace{2in} \,  + \  \Delta_{4U} \,  + \,  \Delta_{4V} \Bigg],
    \label{eq:O4_delsplit}
\end{align}
where
\begin{equation}
    \begin{split}
            \Delta_{4U} & =  \sum_{j=s}^{T+s} \sum_{k=t}^{T+t}      \left(\frac{\gamma}{\alpha^2} \right)^{j+k -s -t} \kappa_{2(j+k+1)}^\infty [ \E\{ \tU_{T+1+s-j} \tU_{T+1+t-k} \}  - \E\{ U_{s+1-j} U_{t+1-k} \} ],  \\
    \Delta_{4V} & = \frac{1}{\alpha^2} \sum_{j=s}^{T+s} \sum_{k=t}^{T+t}      \left(\frac{\gamma}{\alpha^2} \right)^{j+k -s -t} \kappa_{2(j+k+2)}^\infty [ \E\{ \tV_{T+s-j} \tV_{T+t-k} \}  
    - \E\{ V_{s-j} V_{t-k} \} ].
    \end{split}
\end{equation}
Using \eqref{eq:sig_omega_conv2}, for sufficiently large $T$ we have
\begin{equation}
    \begin{split}
        |\Delta_{4U}| & < \frac{\delta}{\alpha^2} \sum_{j=0}^T \sum_{k=0}^T \left(\frac{\gamma}{\xi \alpha^2} \right)^{j+k} \kappa^\infty_{2(j+k +s+t+1)} \, < \, \delta C_{s,t} , \\
        |\Delta_{4V}| & < \frac{\gamma^2\delta}{\alpha^2} \sum_{j=0}^T \sum_{k=0}^T \left(\frac{\gamma}{\xi \alpha^2} \right)^{j+k} \kappa^\infty_{2(j+k +s+t+2)} \,  < \, \delta C_{s,t},
    \end{split}
    \label{eq:abs_Del4_bnd}
\end{equation}
for a positive constant $C_{s,t}$, since each of the double sums in \eqref{eq:abs_Del4_bnd} is bounded as $T \to \infty$, for $  \xi \tilde{\alpha}^2 := \xi\alpha^2/\gamma > \ \tilde{\alpha}_{\rm s}^2$. Therefore, $\Delta_{4U}, \Delta_{4V}$ both tend to $0$ as $T \to \infty$. Using this in \eqref{eq:O4_delsplit} along with \eqref{eq:uivi_SEconv}, we obtain 
\begin{equation}
    \begin{split}
      & \lim_{T \to \infty} O_4  =  \gamma \, \prod_{i=2}^{s+1}  \E\{ \su_i'(F_{i-1})\} \E\{ \sv_{i-1}'(G_{i-1})\}    \prod_{i=2}^{t+1}    \E\{ \su_i'(F_{i-1})\} \E\{ \sv_{i-1}'(G_{i-1})\}   \\ 
    &  \qquad \sum_{j=s}^{\infty} \sum_{k=t}^{\infty}      \left(\frac{\gamma}{\alpha^2} \right)^{j+k -s -t} \Big[ \kappa_{2(j+ k+1)}^\infty \E\{ U_{s+1-j} U_{t+1-k} \}   + \kappa_{2(j+k+2)}^\infty \frac{1}{\alpha^2}  \, \E\{ V_{s-j} V_{t-k}  \} \Big].
    \end{split}
    \label{eq:O4lim}
\end{equation}

Next, consider $O_2$ in \eqref{eq:O2_def_rect}, which we write as 
\begin{equation}
    \begin{split}
     & O_2  = \gamma \Big( \prod_{i=2}^{t+1}    \E\{ \su_i'(\tF_{T+i-1})\} \E\{ \sv_{i-1}'(\tG_{T+i-1})\} \Big)  \sum_{j=0}^{s-1}  \prod_{i=s-j+2}^{s+1}  \hspace{-5pt} \E\{ \su_i'(\tF_{T+i-1})\} \E\{ \sv_{i-1}'(\tG_{T+i-1})\}   \\
     &  \Bigg[  \sum_{k=t}^{T+t} 
    \left(\frac{\gamma}{\alpha^2} \right)^{k-t}    \Big[ \kappa_{2(j+ k+1)}^\infty \E\{ U_{s+1-j} U_{t+1-k} \}   +    \frac{\kappa_{2(j+k+2)}^\infty}{\alpha} \, \E\{ \su'_{s+1-j}(\tF_{T+s-j}) \}    \E\{ V_{s-j} V_{t-k} \} \Big]  \\
    & \quad \Delta_{3U, j} \, + \, \Delta_{3V, j}\Bigg],
    \end{split}
    \label{eq:O2_del_split}
\end{equation}
where
\begin{equation}
    \begin{split}
        \Delta_{3U, j} & = \sum_{k=t}^{T+t} 
    \left(\frac{\gamma}{\alpha^2} \right)^{k-t}   \kappa_{2(j+ k+1)}^\infty [ \E\{ \tU_{T+s+1-j} \tU_{T+t+1-k} \} \, - \,  \E\{ U_{s+1-j} U_{t+1-k} \}], \\
    \Delta_{3V, j} & =   \frac{1}{\alpha} \E\{ \su'_{s+1-j}(\tF_{T+s-j}) \} \sum_{k=t}^{T+t} 
    \left(\frac{\gamma}{\alpha^2} \right)^{k-t}  \kappa_{2(j+k+2)}^\infty  \, 
  [  \E\{ \tV_{T+s-j} \tV_{T+t-k} \} -  \E\{ V_{s-j} V_{t-k} \}].
    \end{split}
    \label{eq:Del3UV_def}
\end{equation}
From \eqref{eq:tUtVsj}, we recall that for $j \in [0, s-1]$, $k \in [t, T+t]$:
\begin{equation}
    \begin{split}
        & \E\{ \tU_{T+s+1-j} \tU_{T+t+1-k} \} = \frac{1}{\alpha} \E\{  \su_{s+1-j}(\tF_{T+s-j}) \tF_{T-(k-t)}  \}, \\ 
        & \E\{ \tV_{T+s-j} \tV_{T+t-k}  \} = 
        \frac{\gamma}{\alpha}\E\{  \sv_{s-j}(\tG_{T+s-j}) \tG_{T-(k-t)}\}.
    \end{split}
\end{equation}
Using the induction hypothesis and arguments similar to \eqref{eq:tUjtUk_split}-\eqref{eq:Delta2_bnd0}, for any $\delta >0$ and sufficiently large $T$ we have
\beq
\begin{split}
& | \E\{ \tU_{T+s+1-j} \tU_{T+t+1-k} \} \, - \,  \E\{ U_{s+1-j} U_{t+1-k} \} | < \frac{\delta}{\alpha} \xi^{-(k-t)}, \\
& |E\{ \tV_{T+s-j} \tV_{T+t-k}  \} \, - \,  E\{ V_{s-j} V_{t-k}  \}| < \frac{\gamma \delta}{\alpha} \xi^{-(k-t)},  \quad j \in [0, s-1], \ k \in [t, T+t].
\end{split}
\eeq
Using this in \eqref{eq:Del3UV_def}, following steps similar to \eqref{eq:Delta2_def} and \eqref{eq:bnd_Csj}, and  noting the convergence of the power series defining $R(\gamma/\xi\alpha^2)$,  we have $\lim_{T \to \infty} \Delta_{3U,j} = \lim_{T \to \infty} \Delta_{3V,j} =0$ for $j \in [0, s-1]$. Using this in \eqref{eq:O2_del_split} along with \eqref{eq:uivi_SEconv}, we have 
\begin{equation}
    \begin{split}
       &  \lim_{T \to \infty} O_2  = \gamma \Big( \prod_{i=2}^{t+1}    \E\{ \su_i'(F_{i-1})\} \E\{ \sv_{i-1}'(G_{i-1})\} \Big)  \sum_{j=0}^{s-1} \,  \prod_{i=s-j+2}^{s+1}  \hspace{-5pt} \E\{ \su_i'(F_{i-1})\} \E\{ \sv_{i-1}'(G_{i-1})\}   \\
     & \  \sum_{k=t}^{\infty} 
    \left(\frac{\gamma}{\alpha^2} \right)^{k-t}    \Big[ \kappa_{2(j+ k+1)}^\infty \E\{ U_{s+1-j} U_{t+1-k} \}  +   \kappa_{2(j+k+2)}^\infty \frac{1}{\alpha} \, \E\{ \su'_{s+1-j}(F_{s-j}) \}    \E\{ V_{s-j} V_{t-k} \} \Big].
    \end{split}
    \label{eq:O2lim}
\end{equation}

Using a similar sequence of steps, we also have 
\begin{equation}
    \begin{split}
       &  \lim_{T \to \infty} O_3  = \gamma \Big( \prod_{i=2}^{s+1}    \E\{ \su_i'(F_{i-1})\} \E\{ \sv_{i-1}'(G_{i-1})\} \Big)  \sum_{k=0}^{t-1} \,  \prod_{i=t-k+2}^{t+1}  \hspace{-5pt} \E\{ \su_i'(F_{i-1})\} \E\{ \sv_{i-1}'(G_{i-1})\}   \\
     & \  \sum_{j=s}^{\infty} 
    \left(\frac{\gamma}{\alpha^2} \right)^{j-s}    \Big[ \kappa_{2(j+ k+1)}^\infty \E\{  U_{s+1-j}U_{t+1-k} \}   +   \kappa_{2(j+k+2)}^\infty \frac{1}{\alpha} \, \E\{ \su'_{t+1-k}(F_{t-k}) \}    \E\{  V_{s-j}V_{t-k} \} \Big].
    \end{split}
    \label{eq:O3lim}
\end{equation}
Noting that the sums of the limits in \eqref{eq:O1_lim_rect}, \eqref{eq:O4lim}, \eqref{eq:O2lim} and \eqref{eq:O3lim} equals $\omega_{s+1, t+1}$ (defined in \eqref{eq:omega_rect}), we have shown that $\lim_{T \to \infty} \tomega_{T+s+1, T+t+1} = \omega_{s+1, t+1}$. The sequence of steps to show that $\lim_{T \to \infty} \tsigma_{T+s+1, T+t+1} = \sigma_{s+1, t+1}$ is very similar, and is omitted to avoid repetition.

\textbf{Proof of \eqref{eq:fake_modified_conv_uf}-\eqref{eq:fake_modified_conv_vg}}.
Since $\psi, \varphi \in \PL(2)$,
using the Cauchy-Schwarz inequality (as in \eqref{eq:PL2_psi_bnd}), for a universal constant $C >0$ we have
\begin{align}
   &  \bigg\vert  \frac{1}{m}  \sum_{i=1}^m \psi (u^*_{i}, \tu^{T+1}_i, \ldots, \tu^{T+t+1}_i, \tf^{T+1}_i, \ldots \tf^{T+t}_i) - 
\frac{1}{m}  \sum_{i=1}^m \psi (u^*_{i}, \hu^1_i, \ldots, \hu^{t+1}_i, \hf^1_i, \ldots \hf^{t}_i)  \bigg\vert \nonumber  \\
& \le 2C(t+2) \left[ 1 + \frac{\| \bu^* \|^2}{m}  +  
 \sum_{\ell=1}^{t+1} \Big( \frac{\| \tbu^{T+\ell} \|^2}{m} +
 \frac{\| \hbu^{\ell} \|^2}{m} \Big) + 
 \sum_{\ell=1}^{t} \Big( \frac{\| \tbf^{T+\ell} \|^2}{m}
 + \frac{\| \hbf^{\ell} \|^2}{m} \Big) \right]^{\frac{1}{2}} \nonumber  \\ 
 &  \cdot \left( \frac{\| \tbu^{T+1} - \hbu^1 \|^2}{m}+ \ldots + \frac{\| \tbu^{T+t+1} - \hbu^{t+1} \|^2}{m}  +\frac{\| \tbf^{T+1} - \hbf^1  \|^2}{m} + \ldots  + \frac{\| \tbf^{T+t} - \hbf^t  \|^2}{m}\right)^{\frac{1}{2}},
 \label{eq:PL2psi_bnd_rect}
\end{align}
\begin{align}
 & \bigg\vert  \frac{1}{n}  \sum_{i=1}^n \varphi (v^*_{i}, \tv^{T+1}_i, \ldots, \tv^{T+t}_i, \tg^{T+1}_i, \ldots \tg^{T+t}_i) -  
\frac{1}{n}  \sum_{i=1}^n \varphi (v^*_{i}, \hv^1_i, \ldots, \hv^{t}_i, \hg^1_i, \ldots \hg^{t}_i)  \bigg\vert \nonumber \\
& \le 2C(t+2)\left[ 1 + \frac{\| \bv^* \|^2}{n}  +  
 \sum_{\ell=1}^{t} \Big( \frac{\| \tbv^{T+\ell} \|^2}{n} +
 \frac{\| \hbv^{\ell} \|^2}{n} \Big) + 
 \sum_{\ell=1}^{t} \Big( \frac{\| \tbg^{T+\ell} \|^2}{n}
 + \frac{\| \hbg^{\ell} \|^2}{n} \Big) \right]^{\frac{1}{2}} \nonumber \\
 &  \cdot \left( \frac{\| \tbv^{T+1} - \hbv^1 \|^2}{n}+ \ldots + \frac{\| \tbv^{T+t} - \hbv^{t} \|^2}{n}  +\frac{\| \tbg^{T+1} - \hbg^1  \|^2}{n} + \ldots  + \frac{\| \tbg^{T+t} - \hbg^t  \|^2}{n}\right)^{\frac{1}{2}}.
  \label{eq:PL2phi_bnd_rect}
\end{align}
The proof strategy is similar to the square case. We inductively show that in the limit $T, n \to \infty$ (with the limit in $n$ taken first): i)  the terms in the last line of \eqref{eq:PL2psi_bnd_rect} and \eqref{eq:PL2phi_bnd_rect} all converge to 0 almost surely, and ii) each of the terms within the square brackets in \eqref{eq:PL2psi_bnd_rect} and \eqref{eq:PL2phi_bnd_rect} converges to a finite deterministic value.

\underline{Base case $t=1$}: Recalling that $\hbu^1 = \sqrt{m} \bu_{\PCA}$, from Lemma \ref{lemma:convrect}, we have 
\beq
\lim_{T \to \infty} \lim_{m \to \infty} \frac{\| \tbu^{T+1} - \hbu^1 \|^2}{m}=0.
\label{eq:tuhu_conv_rect}
\eeq
 Writing $x = \gamma/\alpha^2$ for brevity,
recall that $\sum_{i=1}^\infty \kappa_{2i}^\infty x^i = R(x)$.  
From the definitions of $\tbg^{T+1}$ and $\hbg^1$ in \eqref{eq:AMPfake1rect} and \eqref{eq:AMP_rect_init_mod} and  we have 
\begin{equation*}
    \begin{split}
     \tbg^{T+1} - \hbg^1 & = \frac{1}{1+ \gamma R(x)}  \bX^\sT(\tbu^{T+1} - \hbu^1) \, + \, \Bigg[ \frac{\gamma R(x)}{ 1+ \gamma R(x)}  \bX^\sT\tbu^{T+1}  - \alpha \sum_{j=1}^{T} \kappa_{2j} x^j \tbv^{T+1-j}\Bigg], 
    \end{split}
\end{equation*}
where we have used $\tsb_{T+1, T+1-j} = \alpha \kappa_{2j} x^j$ for $j \in [1,T]$. Therefore
\begin{equation}
    \begin{split}
      &   \frac{\|\tbg^{T+1} - \hbg^1 \|^2}{n} \le
        \frac{2}{(1+ \gamma R(x))^2} \| \bX \|^2_{\rm op} \frac{\| \tbu^{T+1} - \hbu^1 \|^2}{n} \\
        & \qquad + \, \frac{2}{n}\Bigg\| \frac{\gamma R(x)}{ 1+ \gamma R(x)}  \bX^\sT\tbu^{T+1}  - \alpha \sum_{j=1}^{T} \kappa_{2j} x^j \tbv^{T+1-j} \Bigg\|^2 \, =: \, 2( S_1 +  S_2).
    \end{split}
    \label{eq:tgt1_hbg1_diff}
\end{equation}
Since $\| \bX \|_{\rm op} \stackrel{n \to \infty}{\longrightarrow} D^{-1}(x)$,  from \eqref{eq:tuhu_conv_rect} we have $\lim_{T, n \to \infty} S_1=0$. (Here and in the remainder of the proof,  $\lim_{T, n \to \infty}$  denotes the limit  $n \to \infty$ taken first and then $T \to \infty$.) Next, using the definition of $\tbg^{T+1}$ in \eqref{eq:AMPfake1rect}, we write the second term $S_2$ as
\begin{align}
  S_2 & = \frac{1}{n}\Bigg\| \frac{\gamma R(x)}{ 1+ \gamma R(x)}  \tbg^{T+1}  - \frac{\alpha}{1+ \gamma R(x)} \sum_{j=1}^{T} \kappa_{2j} x^j \tbv^{T+1-j} \Bigg\|^2   \nonumber  \\
   & \le \frac{2}{n}\Bigg\| \frac{\gamma R(x)}{ 1+ \gamma R(x)}  \tbg^{T+1}  - \frac{\alpha}{1+ \gamma R(x)} \sum_{j=1}^{T} \kappa_{2j}^\infty \, x^j \tbv^{T+1-j} \Bigg\|^2  + \frac{2\alpha^2}{(1+ \gamma R(x))^2} \Delta_{S_2}, \label{eq:S2_delS2}
   \end{align}
where 
\beq
\Delta_{S_2} := \frac{1}{n} \Big\|  \sum_{j=1}^{T} (\kappa_{2j}^\infty - \kappa_{2j}) x^j \tbv^{T+1-j} \Big \|^2 = \frac{1}{n} \hspace{-1pt} \sum_{i,j=1}^T (\kappa_{2i}^\infty - \kappa_{2i}) (\kappa_{2j}^\infty - \kappa_{2j}) x^{i+j} 
\frac{\< \tbv^{T+1-i}, \, \tbv^{T+1-j} \>}{n}.
\label{eq:DelS2_def}
\eeq
Using the state evolution result of Proposition \ref{prop:tilSErect}, we almost surely have 
\beq
\begin{split}
\lim_{n \to \infty} \, \frac{\< \tbv^{T+1-i}, \, \tbv^{T+1-j} \>}{n} & =  \E\{ \tV_{T+1-i} \tV_{T+1-j} \} \\ 
& =\frac{\gamma^2}{\alpha^2}(\alpha^2 \Delta_{\PCA} \,  + \, \tomega_{T+1-i, T+1-j}) < C,
\end{split}
\label{eq:tvitvj}
\eeq
for some universal constant $C>0$.
Here, $\tomega_{T+1-i, T+1-j}$ is defined in \eqref{eq:tomega_update}, and we recall from \eqref{eq:G0GKtilde_rect}-\eqref{eq:Vtilde_rect} that
\beq
\tV_{T+1-j}  = \frac{\gamma}{\alpha} \tG_{T+1-j} \ \text{ with }  \ \tG_{T+1-j}
= \alpha\sqrt{\Delta_{\PCA}} V_* + \tZ_{T+1-j}, \quad \text{ for } j \in [0, T].
\label{eq:tVtG}
\eeq
Since $\kappa_{2i} \to \kappa^\infty_{2i}$  as $n \to \infty$, for $i \in [1, T]$ (by the model assumptions), using \eqref{eq:tvitvj} in \eqref{eq:DelS2_def},
\beq
\lim_{T \to \infty} \lim_{n \to \infty} \, \Delta_{S_2} =0 \ \text{ almost surely}.
\label{eq:DelS2_lim}
\eeq
Next, using Proposition \ref{prop:tilSErect}, for any $T >0$, the first term in \eqref{eq:S2_delS2} has the following almost sure limit as $n \to \infty$:
\begin{align}
    & \lim_{n \to \infty}  \frac{1}{n}\Bigg\| \frac{\gamma R(x)}{ 1+ \gamma R(x)}  \tbg^{T+1}  - \frac{\alpha}{1+ \gamma R(x)} \sum_{j=1}^{T} \kappa_{2j}^\infty x^j \tbv^{T+1-j} \Bigg\|^2  \nonumber \\
    & = \E \Bigg\{ \Bigg( \frac{\gamma R(x)}{ 1+ \gamma R(x)} \tG_{T+1} - \frac{\alpha}{1+ \gamma R(x)} \sum_{j=1}^{T} \kappa_{2j}^\infty x^j \tV_{T+1-j} \Bigg)^2 \Bigg\}  \nonumber \\
    &  \stackrel{\mathclap{\mbox{\footnotesize (a)}}}{=} \frac{\gamma^2 }{(1+ \gamma R(x))^2} \, \E \Bigg\{  \Bigg(\Big( R(x) - \sum_{j=1}^T \kappa_{2j}^\infty x^j \Big) \tG_{T+1} 
    + \sum_{j=1}^T \kappa_{2j}^\infty x^j (\tG_{T+1} - \tG_{T+1-j}) \Bigg)^2 \Bigg\},
    \label{eq:S2_term1_rect}
\end{align}
where (a) is obtained using \eqref{eq:tVtG}. From \eqref{eq:Rrect}, we  have $\lim_{T \to \infty} \sum_{j=1}^T \kappa_{2j}^\infty x^j =R(x)$. Furthermore,  using \eqref{eq:tVtG} we have
\begin{align}
    & \E \Bigg\{  \Bigg(\sum_{j=1}^T \kappa_{2j}^\infty x^j (\tG_{T+1} - \tG_{T+1-j}) \Bigg)^2 \Bigg\} \nonumber \\ 
   & = 
   \sum_{i,j=1}^T \kappa_{2i}^\infty \kappa_{2j}^\infty \, x^{i+j} \, (\tomega_{T+1, T+1} - \tomega_{T+1, T+1-i} -  \tomega_{T+1, T+1-j} + \tomega_{T+1-i, T+1-j}) \nonumber  \\
    & \quad \longrightarrow 0 \ \text{ as } \ T \to \infty,
   \label{eq:tsigij}
\end{align}
where the $T \to \infty$  limit is obtained using Lemma \ref{lem:SE_FP_phase1rect} and steps similar to \eqref{eq:S2_sumexp}-\eqref{eq:S2lim}. Using \eqref{eq:DelS2_lim}-\eqref{eq:tsigij} in \eqref{eq:S2_delS2}, we have
\beq
\lim_{T \to \infty} \lim_{n \to \infty} S_2 =0 \ \  \text{ almost surely}.
\eeq
Hence using \eqref{eq:tgt1_hbg1_diff}, we have shown that
$\lim_{T,n  \to \infty} \frac{1}{n} \| \tbg^{T+1} - \hbg^1 \|^2 =0$ almost surely.

The proof that $\lim_{T, n \to \infty} \frac{1}{n} \| \tbf^{T+1} - \hbf^1 \|^2=0$ uses similar steps: from the definitions of $\tbf^{T+1}$ and $\hbf^1$ in \eqref{eq:AMPfake1rect} and \eqref{eq:AMP_rectg_mod}, we have
\begin{equation}
    \begin{split}
        \tbf^{T+1} - \hbf^1 = \frac{\gamma}{\alpha} \bX (\tbg^{T+1} - \hbg^1) \, + \, \bar{\sa}_{1,1} \hbu^1 - \sum_{j=0}^T \tsa_{T+1, T+1-j} \tbu^{T+1-j},
    \end{split}
\end{equation}
where $\bar{\sa}_{1,1} = \alpha \sum_{j=0}^\infty \kappa^\infty_{2(j+1)} x^{j+1}$ and $\tsa_{T+1, T+1-j} = \alpha \kappa_{2(j+1)} x^{j+1}$ for $j \in [0,T]$. Therefore, 
\begin{equation}
    \begin{split}
       &  \frac{\| \tbf^{T+1} - \hbf^1 \|^2}{n}   \leq \frac{5\gamma^2}{\alpha^2} \| \bX \|^2_{\rm op} 
        \frac{\| \tbg^{T+1} - \hbg^1\|^2}{n} 
       \,  + \, 5\bar{\sa}_{1,1}^2 \frac{ \| \hbu^1 - \tbu^{T+1} \|^2}{n} \\
        & \quad + \frac{5}{n} 
        \Bigg\| \alpha \sum_{j=0}^T \kappa^\infty_{2(j+1)} x^{j+1} (\tbu^{T+1}- \tbu^{T+1-j})\Bigg\|^2 \,  +  \, 5\alpha^2 \Bigg(  \sum_{j=T+1}^\infty \kappa^\infty_{2(j+1)} x^{j+1} \Bigg)^2 \frac{\| \hbu^1 \|^2}{n} \\ 
        & \quad + \frac{5}{n}  \Bigg\| \alpha \sum_{j=0}^T 
        (\kappa^\infty_{2(j+1)} - \kappa_{2(j+1)} ) x^{j+1} \tbu^{T+1-j}\Bigg\|^2.
    \end{split}
    \label{eq:tfT1hf1_split}
\end{equation} 
We have shown $\lim_{T,n  \to \infty} \frac{1}{n} \| \tbg^{T+1} - \hbg^1 \|^2 =0$ and $\lim_{T,n  \to \infty} \frac{1}{n} \| \tbu^{T+1} - \hbu^1 \|^2$, hence the first two terms in \eqref{eq:tfT1hf1_split} converge to 0.   For the third term in \eqref{eq:tfT1hf1_split}, we first apply Proposition \ref{prop:tilSErect} to express the $n \to \infty$ limit in terms of state evolution parameters of the artificial AMP,  which  can then be shown to converge to $0$ as $T \to \infty$ using Lemma \ref{lem:SE_FP_phase1rect} and steps similar to \eqref{eq:S2_sumexp}-\eqref{eq:S2lim}. Since the power series $\sum_{j=0}^\infty \kappa^\infty_{2(j+1)} x^{j+1} = R(x)$ converges, and $\| \hbu^1 \|^2/n = m/n = \gamma$, the fourth term  converges to $0$ as $T, n \to \infty$. As $\kappa_{2(j+1)} \to \kappa^\infty_{2(j+1)}$ as $n \to \infty$, by arguments similar to  \eqref{eq:kappj_kappj_inf}-\eqref{eq:uiuj_inner_prod}, the final term in \eqref{eq:tfT1hf1_split} also converges to $0$.

Recalling that $\tbv^{T+1} - \hbv^1 = \frac{\gamma}{\alpha}(\tbg^{T+1} - \hbg^1)$, it follows that $\lim_{T, n \to \infty} \frac{1}{n}\| \tbv^{T+1} - \hbv^1 \|^2 =0$ almost surely. Finally, a triangle inequality sandwiching argument like the one used in \eqref{eq:tbu_triangle}-\eqref{eq:hu1_lim} yields
\begin{equation}
\begin{split}
      & \lim_{T \to \infty}  \lim_{n \to \infty} \frac{\| \tbv^{T+1} \|^2}{n} =   
    \lim_{T \to \infty}  \lim_{n \to \infty} \frac{\| \hbv^{1} \|^2}{n} = \frac{\gamma^2}{\alpha^2}( \alpha^2\Delta_{\PCA} + \omega_{1,1}), \\
    &  \lim_{T \to \infty}  \lim_{n \to \infty} \frac{\| \tbu^{T+1} \|^2}{m} =   
    \lim_{T \to \infty}  \lim_{n \to \infty} \frac{\| \hbu^{1} \|^2}{m} =1.  
\end{split}
\label{eq:hatu1_hatv1_lim}
\end{equation}
This completes the proof of \eqref{eq:fake_modified_conv_uf}-\eqref{eq:fake_modified_conv_vg} for $t=1$.

\underline{Induction step}: For $t \ge 1$, assume that the following hold almost surely for $\ell \in [1, t]$:
\begin{equation}
    \begin{split}
         \lim_{T \to \infty} \lim_{n \to \infty} \frac{\| \hbu^\ell - \tbu^{T+\ell}\|^2}{m} \,   = \,  \lim_{T \to \infty} \lim_{n\to \infty} \frac{\| \hbg^\ell - \tbg^{T+\ell}\|^2}{n}
        &  \,   = \, 
        \lim_{T \to \infty} \lim_{n\to \infty}  \frac{\| \hbv^\ell - \tbv^{T+\ell}\|^2}{n} =0.
    \end{split}
\end{equation}
We now show that   $\lim_{T, n \to \infty} \frac{1}{n}\| \tbf^{T+t} - \hbf^{t} \|^2=0$.  We have already shown this for $t=1$ above. For $t\ge 2$,  using the definitions  $\tbf^{T+t}$ and $\hbf^t$ in \eqref{eq:AMPfake1rect} and \eqref{eq:AMP_rectg_mod}, and applying the Cauchy-Schwarz inequality, we have 
\begin{equation}
    \begin{split}
           &  \frac{1}{n}\| \tbf^{T+t} - \hbf^{t} \|^2 \leq 
   \frac{(t+1)}{n}\Bigg( \| \bX (\tbv^{T+t} - \hbv^t) \|^2 + 
    \sum_{\ell=2}^t \| \tsa_{T+t,T+\ell} \tbu^{T+\ell} - \bar{\sa}_{t,\ell} \hbu^\ell \|^2   \\
   & \hspace{2in}   +  \ \Big\|  \sum_{i=1}^{T+1} \tsa_{T+t,i} \tbu^i  - \bar{\sa}_{t,1} \hbu^1\Big\|^2 \Bigg).
    \end{split}
    \label{eq:tfT_hft_split_rect}
\end{equation}
The decomposition and the analysis of the three terms in \eqref{eq:tfT_hft_split_rect} is similar to that in \eqref{eq:fTt_ft_bnd} for the square case. Using arguments similar to \eqref{eq:XuTthut}-\eqref{eq:fTt_ft_lim}, we obtain  $\lim_{T, n \to \infty} \frac{1}{n}\| \tbf^{T+t} - \hbf^{t} \|^2 =0 $. Recalling that $\hbu^{t+1} = \su_{t+1}(\hbf^t)$ and $\tbu^{T+t+1}=\su_{t+1}(\tbf^{T+t})$ with $\su_{t+1}$ Lipschitz, we also  have  $\lim_{T, n \to \infty} \frac{1}{n}\| \tbu^{T+t+1} - \hbu^{t+1} \|^2 =0 $ almost surely.  The proof that 
$\lim_{T, n \to \infty} \frac{1}{n}\| \tbg^{T+t+1} - \hbg^{t+1} \|^2 =0 $ uses a decomposition similar to \eqref{eq:tfT_hft_split_rect} and  is along the same lines. Since $\hbv^{t+1} = \sv_{t+1}(\hbg^{t+1})$ and $\tbv^{T+t+1}=\sv_{t+1}(\tbg^{T+t+1})$ with $\sv_{t+1}$ Lipschitz, it follows that  $\lim_{T, n \to \infty} \frac{1}{n}\| \tbv^{T+t+1} - \hbv^{t+1} \|^2 =0 $ almost surely.

Using these results together with a triangle inequality sandwich argument similar to \eqref{eq:tbu_triangle}-\eqref{eq:hu1_lim}, we  have $$ \lim_{n \to \infty} \frac{1}{n} \| \hbu^{t+1} \|^2 = \lim_{T, n \to \infty}\| \tbu^{T+t+1} \|^2 = \E\{ \su_{t+1}(F_t)^2 \}.$$ Similarly, $$ \lim_{n \to \infty} \frac{1}{n} 
\| \hbv^{t+1} \|^2 = \lim_{T, n \to \infty} \frac{1}{n} 
\| \tbv^{T+t+1} \|^2 = \E\{ \sv_{t+1}(G_{t+1})^2 \}.$$ Using these results in \eqref{eq:PL2psi_bnd_rect} and \eqref{eq:PL2phi_bnd_rect} completes the inductive proof of
\eqref{eq:fake_modified_conv_uf}-\eqref{eq:fake_modified_conv_vg}.
\end{proof}

\subsection{Proof of Theorem \ref{thm:rect}}\label{subsec:app_thm_prof_rect}
The proof is along the same lines as that for the square case in Section \ref{subsec:app_thm_prof_sq}; to avoid repetition,  we only sketch the main steps.
The first step is to show using Lemma \ref{lem:sec_phase_rect}  that the state evolution result holds for the the modified AMP. That is, the following almost sure limits hold  for $t \ge 1$:
\begin{align}
\lim_{m \to \infty} \frac{1}{m} \sum_{i=1}^m 
 \psi (u^*_{i}, \hu^1_i, \ldots, \hu^{t+1}_i, \hf^1_i, \ldots \hf^{t}_i)
 = \E \left\{ \psi(U_*, U_1, \ldots, U_{t+1}, F_1, \ldots, F_t) \right\}, \label{eq:ustat_hat_rect}  \\
 \lim_{n \to \infty} \frac{1}{n} \sum_{i=1}^n \varphi (\hv^*_{i}, \hv^1_i, \ldots, \hv^t_i, \hg^1_i, \ldots g^{t}_i)
 = \E \left\{ \varphi(V_*,V_1, \ldots, V_t, G_1, \ldots, G_t) \right\}. \label{eq:vstat_hat_rect}
\end{align}
For each of \eqref{eq:ustat_hat_rect} and \eqref{eq:vstat_hat_rect}, we use a three-term decomposition as in \eqref{eq:S1S2S3_sq}. Using arguments similar to those used to analyze \eqref{eq:S1S2S3_sq}, we can show that each of the terms goes to 0 as $T, n \to \infty$.

The second part of the proof is to inductively show that the following statements hold almost surely for $t \ge 1$:
\begin{align}
    & \lim_{m \to \infty} \abs{\frac{1}{m} \sum_{i=1}^m \psi(u_i^*, u_i^1, \ldots, u^{t+1}_i, f^1_i, \ldots, f^t_i) - 
\frac{1}{m} \sum_{i=1}^m \psi(u_i^*, \hu_i^1, \ldots, \hu^{t+1}_i, \hf^1_i, \ldots, \hf^t_i)} =0, \label{eq:lim_bubf_hbuhbf_joint_rect} \\
 & \lim_{m \to \infty} \frac{\| \bdff^{t} - \hbf^{t} \|^2}{n}=0, \quad  \lim_{m \to \infty} \frac{\| \bu^{t+1} - \hbu^{t+1} \|^2}{m}=0,
 \label{eq:lim_bubf_hbuhbf_sep_rect} \\
  & \lim_{n \to \infty} \abs{\frac{1}{n} \sum_{i=1}^n \varphi(v_i^*, v_i^1, \ldots, v^{t}_i, g^1_i, \ldots, g^t_i) - 
\frac{1}{n} \sum_{i=1}^n \varphi(v_i^*, \hv_i^1, \ldots, \hv^{t}_i, \hg^1_i, \ldots, \hg^t_i)} =0, \label{eq:lim_bvbg_hbvhbg_joint_rect} \\
 & \lim_{n \to \infty} \frac{\| \bg^{t} - \hbg^{t} \|^2}{n}=0, \quad  \lim_{n \to \infty} \frac{\| \bv^{t} - \hbv^{t} \|^2}{n}=0.
 \label{eq:lim_bvbg_hbvhbg_sep_rect}
\end{align}
Since  $\psi \in \PL(2)$, by the same arguments as in \eqref{eq:PL2_uf_hbuf_bnd}, we have 
\begin{align}
& 
\left\vert \frac{1}{m}  \sum_{i=1}^m \psi (u^*_{i},  u^{1}_i, \ldots, u^{t+1}_i, f^1_i, \ldots f^{t}_i)
\, - \, 
\frac{1}{m}  \sum_{i=1}^m \psi (u^*_{i}, \hu^1_i, \ldots, \hu^{t+1}_i, \hf^1_i, \ldots \hf^{t}_i)  \right\vert  \nonumber \\
  & \le 2C(t+2) \left[ 1 + \frac{\| \bu^* \|^2}{m}  +  
 \sum_{\ell=1}^{t+1} \Big( \frac{\| \bu^{\ell} \|^2}{m} +
 \frac{\| \hbu^{\ell} \|^2}{m} \Big) + 
 \sum_{\ell=1}^{t} \Big( \frac{\| \bdff^{\ell} \|^2}{m}
 + \frac{\| \hbf^{\ell} \|^2}{m} \Big) \right]^{\frac{1}{2}} \nonumber  \\ 
 &  \quad \cdot \left( \frac{\| \hbu^{1} - \bu^1 \|^2}{m}+ \ldots + \frac{\| \bu^{t+1} - \hbu^{t+1} \|^2}{m}  +\frac{\| \bdff^{1} - \hbf^1  \|^2}{m} + \ldots  + \frac{\| \bdff^{t} - \hbf^t  \|^2}{m}\right)^{\frac{1}{2}}. \label{eq:PL2_uf_hbuf_bnd_rect}
\end{align}
Using $\varphi \in \PL(2)$, an analogous bound holds for  the term in \eqref{eq:lim_bvbg_hbvhbg_joint_rect}.

 We then argue that $\lim_{n \to \infty} \frac{1}{m} \| \bdff^t - \hbf^t  \|^2 =0$; this follows from a  bound similar to \eqref{eq:ft_hft_bnd} 
 and the induction hypothesis. (In the argument,  $\hbu^t, \bu^t$, $\{ \sb_{t, \ell}, \bar{\sb}_{t, \ell} \}_{\ell \in [1,t]}$ in \eqref{eq:ft_hft_bnd}  are replaced by $\hbv^t, \bv^t$, $\{ \sa_{t, \ell}, \bar{\sa}_{t, \ell} \}_{\ell \in [1,t]}$, respectively.)  Then, recalling $\hbu^{t+1} = \su_{t+1}(\hbf^t)$ and $\bu^{t+1} = \su_{t+1}(\bdff^t)$, since $\su_{t+1}$ Lipschitz, it follows that $\lim_{n \to \infty} \frac{1}{m} \| \bu^{t+1} - \hbu^{t+1}  \|^2 =0$. Using the triangle inequality sandwiching argument in \eqref{eq:tbu_triangle}, the terms $\frac{1}{m}\|  \bdff^t\|^2$, $\frac{1}{m}\| \hbf^t\|^2$, $\frac{1}{m}\|  \bu^t\|^2$, and $\frac{1}{m}\|  \hbu^t\|^2$ converge to deterministic limits (analogous to 
\eqref{eq:fl_hfl_ul_hul_sq}). This leads to \eqref{eq:lim_bubf_hbuhbf_joint_rect} via  \eqref{eq:PL2_uf_hbuf_bnd_rect}. The results  \eqref{eq:lim_bvbg_hbvhbg_joint_rect}-\eqref{eq:lim_bvbg_hbvhbg_sep_rect} are obtained using a similar sequence of steps.

Combining \eqref{eq:lim_bubf_hbuhbf_joint_rect} with \eqref{eq:ustat_hat_rect} and \eqref{eq:lim_bvbg_hbvhbg_joint_rect} with \eqref{eq:vstat_hat_rect} yields the result of Theorem \ref{thm:rect}.  \qed

\section{An auxiliary lemma}
The following result is proved in \cite[Lemma 6]{BM-MPCS-2011}.
\begin{lemma}
\label{lem:lipderiv}
Let $F \colon \reals \to\reals$ be a Lipschitz function, with derivative $F'$ that is continuous almost everywhere in the first argument. Let $U_m$  be a sequence of random variables in $\reals$ converging in distribution to the random variable $U$ as $m \to \infty$. Furthermore, assume that the distribution of $U$ is absolutely continuous with respect to the Lebesgue measure. Then,
\[ \lim_{m \to \infty}  \E\{ F'(U_m) \} = \E\{ F'(U) \}. \]
\end{lemma}

\end{document}